\documentclass{article}

\usepackage{template}

\AppendixTOC     % Enable appendix table of contents
\usepackage[utf8]{inputenc}
\usepackage[T1]{fontenc}

\usepackage{amsfonts,mathrsfs} 
\usepackage{cancel}
\usepackage{float}
\usepackage{overpic}
\usepackage{adjustbox}
\usepackage{wrapfig}
\usepackage{colortbl}

\usepackage{pifont}
\usepackage{soul}
\usepackage{comment}
\usepackage{cprotect}

\usepackage{tikz}
\usepackage{multicol}

\usepackage[most]{tcolorbox}

\newtcolorbox{mybox}[2][]{%
    attach boxed title to top center = {yshift=-8pt},
    colback      = black,
    colframe     = black,
    fonttitle    = \bfseries,
    colbacktitle = white,
    title        = #2,#1,
    enhanced,
}

\title{Elucidating Representation Degradation Problem in Diffusion Model Training}

\author{
    \PaperAuthor[1]{Zhipeng Yao\eqcontrib}{yiucp@outlook.com}
    \quad
    \PaperAuthor[3]{Dazhou Li\eqcontrib}{lidazhou@syuct.edu.cn}
    \quad
    \PaperAuthor[1]{Fan Zhu}{fan.zhu@louisville.edu}
    \\
    \PaperAuthor[1]{Zitong Zhang}{zitong.zhang@louisville.edu}
    \quad
    \PaperAuthor[4]{Wenbin Zhang}{wenbin.zhang@fiu.edu}
    \\
    \PaperAuthor[5]{Xinwei He}{xwhe@mail.hzau.edu.cn}
    \quad
    \PaperAuthor[2]{Yeying Jin\corrauthor}{jinyeying@u.nus.edu}
    \quad
    \PaperAuthor[1]{Rui Yu\corrauthor}{rui.yu@louisville.edu}
    \\
    \PaperAffiliation[1]{University of Louisville}
    \quad
    \PaperAffiliation[2]{National University of Singapore}
    \quad
    \PaperAffiliation[4]{Florida International University}
    \\
    \PaperAffiliation[3]{Shenyang University of Chemical Technology}
    \quad
    \PaperAffiliation[5]{Huazhong Agricultural University}
    \\
    \PaperProject{https://github.com/ZPyAo0591/Representation-Preservation-Diffusion-Models}
}

\begin{document}

\maketitle

% ============================================================
% Main paper
% ============================================================

\begin{abstract}
Diffusion models have achieved remarkable success, yet their training remains inefficient due to a severe optimization bottleneck, which we term Representation Degradation. As noise levels increase, the outputs of the trained model exhibit progressive structural distortion, which can destabilize training and impair generation quality. Our analysis shows that irreducible residuals induce gradient interference across noise levels through shared parameters, perturbing learned representations, while weak Neural Tangent Kernel (NTK) modes slow error correction. Based on a criterion for minimizing stationary parameter variance, we propose Representation Preservation Diffusion Models (RPDM), which reallocate optimization across noise levels without external representation guidance. Across diffusion backbones, RPDM improves representation and generation quality, accelerating training convergence.
\end{abstract}
% \begin{keywords}
%   Diffusion Models, Representation Learning, Generative Modeling, Neural Tangent Kernel, Optimization
% \end{keywords}
\section{Introduction}
\label{sec:intro}

Diffusion models~\cite{sohl2015deep, ho2020denoising, song2020score} have become a dominant paradigm for generative modeling~\cite{li2022diffusion, ho2022video, kong2020diffwave, zhou20213d}. As the field increasingly shifts toward scalable Transformer-based architectures~\cite{peebles2023scalable, bao2023all}, improving training efficiency has become a central challenge. Existing methods often accelerate training by heuristically reallocating the optimization effort across noise levels~\cite{hang2023efficient, karras2022elucidating, wang2024closer, zheng2024non, hang2024improved, zheng2024beta}, but they pay limited attention to the quality of the learned representations. More recently, Representation Alignment (REPA)~\cite{yu2024representation} has shown that aligning diffusion features with pre-trained priors can substantially accelerate training. This result raises a question about standard diffusion training: 

\begin{tcolorbox}[colback=gray!4!white, colframe=gray!60!black, boxrule=0.8pt, arc=1.5mm, left=10pt, right=10pt, top=8pt, bottom=8pt]
\begin{center}
    \textit{``what intrinsic bottleneck makes standard diffusion training inefficient in the first place?''}
\end{center}
\end{tcolorbox}

To investigate this question, we examine whether standard diffusion models can preserve meaningful structure in their predictions across noise levels. We find that although the predicted geometry aligns well with the target distribution at low noise levels, it becomes increasingly distorted and eventually collapses under high noise. As illustrated in \cref{fig:prediction_distribution}, the predicted distribution, when projected onto a 2D plane, preserves meaningful structure at low noise but loses it severely at high noise. We refer to this phenomenon as \emph{Representation Degradation}~\cite{gao2019representation}. These observations reveal an intrinsic instability in diffusion training, suggesting that increasing noise makes it harder for the model to preserve meaningful structure in its predictions.

To explain this phenomenon, we analyze diffusion training as noise-conditioned regression with continuous-time optimization using the Neural Tangent Kernel (NTK)~\cite{jacot2018neural}. A single network must fit prediction targets across noise levels with varying uncertainty. The Bayes decomposition shows that irreducible residuals can induce stochastic gradient fluctuations even at the Bayes optimum. Through shared parameters, these fluctuations propagate across noise levels, with the optimization allocation~\cite{choi2022perception} governing stationary parameter covariance at fixed sampling density. Our NTK analysis links smaller eigenvalues to tighter bounds on mode responses to parameter updates and slower error correction under a local frozen kernel approximation~\cite{rahaman2019spectral}. Persistent parameter fluctuations disturb hidden representations and can increase linear probe risk, explaining how optimization coupling across noise levels contributes to representation degradation.

\begin{figure}[t]
\vspace{-1em}
\centering
\begin{tikzpicture}[scale=1, every node/.style={sloped, allow upside down}]
    \draw[->, thick] (-5.5,4) -- (5.5,4) node[midway, above] {Given \(\bm{x}_t\) input network, where \(t \in 0 \to T\)};
\end{tikzpicture}	

\vspace{1.0em}
% Reduce both inner and outer separation between tikz and image content
\begin{minipage}{0.9\textwidth}
    % Image Prediction Group
    \begin{subfigure}[b]{0.48\textwidth} 
        \centering
        \includegraphics[width=0.48\textwidth]{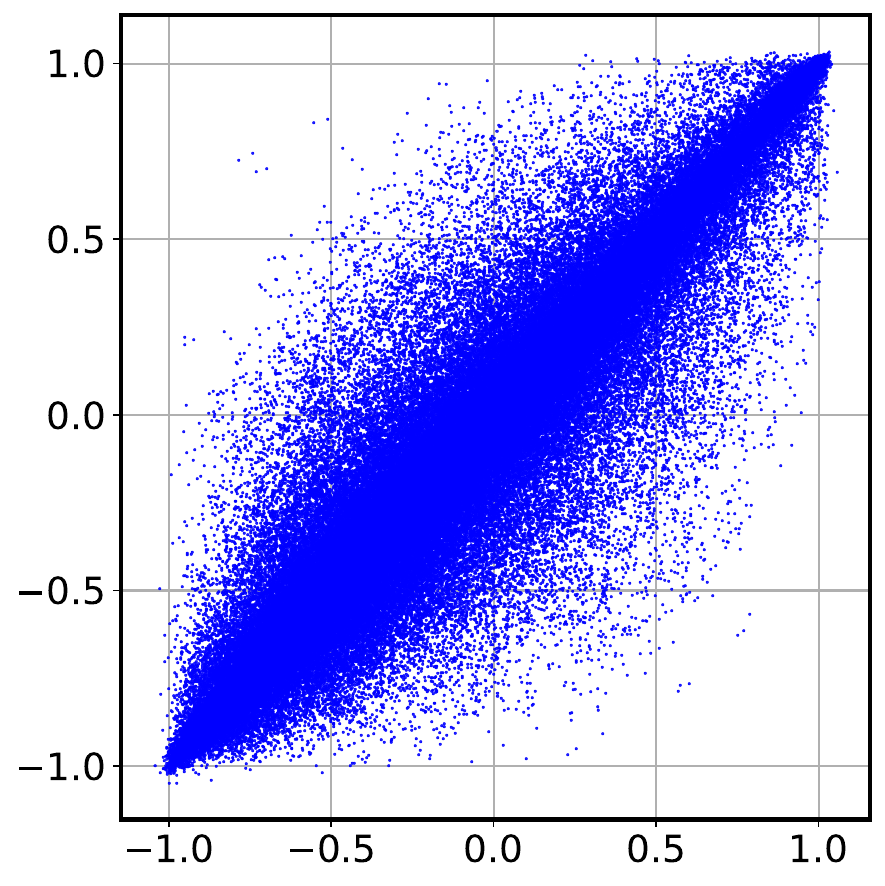}
        \includegraphics[width=0.48\textwidth]{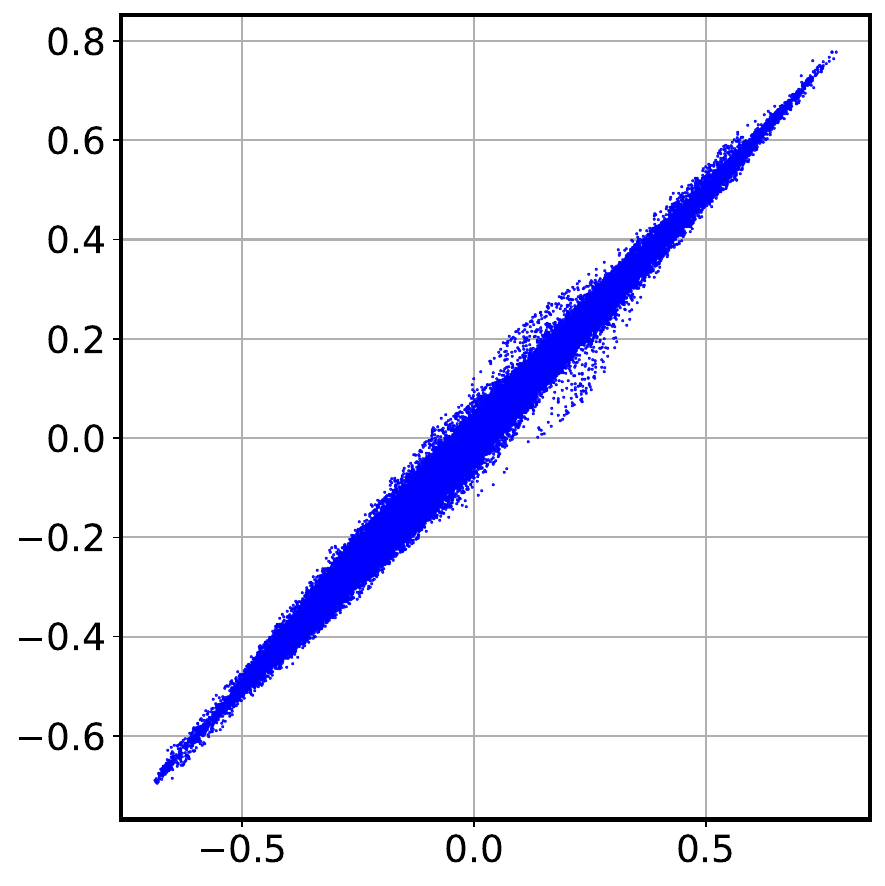}
        % \includegraphics[width=0.98\textwidth]{Figures/output/image.pdf}
        % \vspace{-0.5em}
        \caption{$x_\theta$-prediction.}
        \label{fig:prediction_distribution_img}
    \end{subfigure}
    \hfill
    % Noise Prediction Group
    \begin{subfigure}[b]{0.48\textwidth} 
        \centering
        \includegraphics[width=0.48\textwidth]{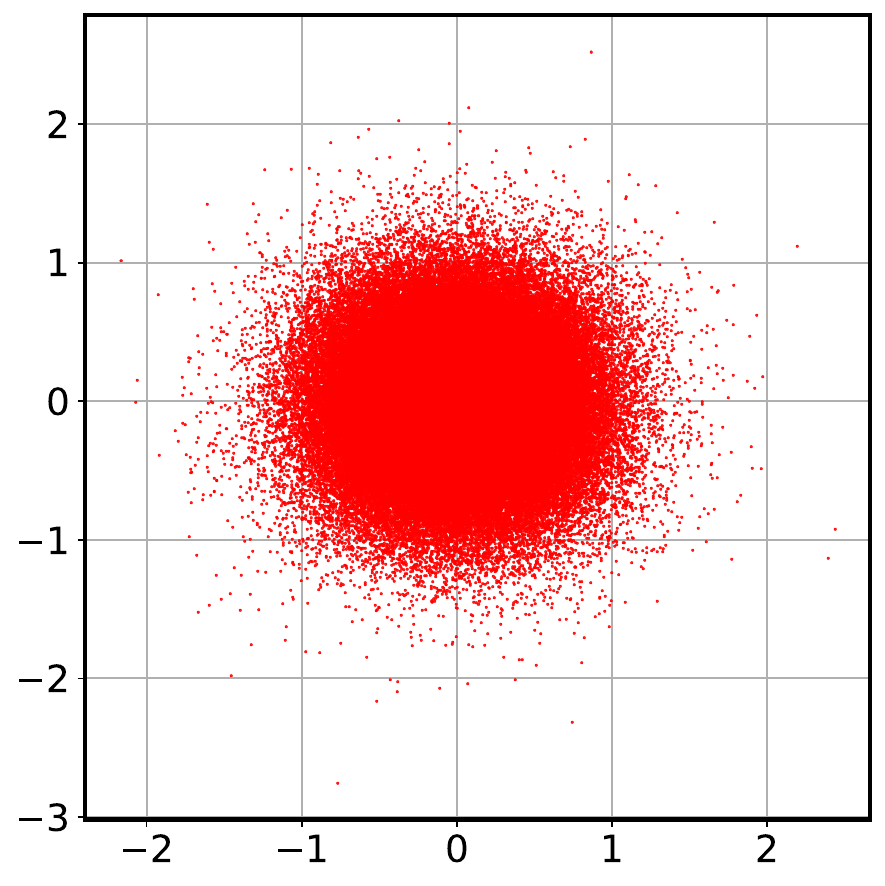}
        \includegraphics[width=0.48\textwidth]{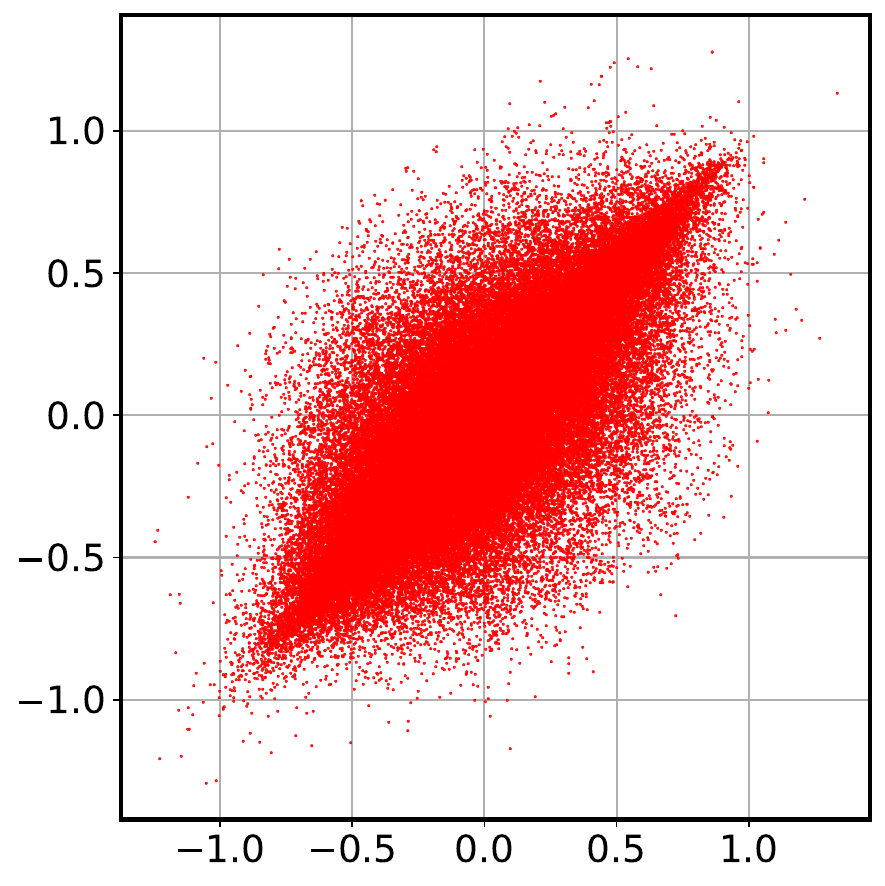}
        % \vspace{-0.5em}
        \caption{$\epsilon_\theta$-prediction.}
        \label{fig:prediction_distribution_eps}
    \end{subfigure}
\end{minipage}  
\vspace{-0.5em}
\caption{We visualized pre-trained model predictions as forward diffusion progressively corrupts the inputs. \cref{fig:prediction_distribution_img} shows that the image distribution becomes increasingly stretched and flattened, while \cref{fig:prediction_distribution_eps} shows the predicted noise shifting from an isotropic to an anisotropic distribution.}% Also, the predictions are numerically lower than the target.}
\vspace{-1.0em}
\label{fig:prediction_distribution}
\end{figure}

Together, these findings suggest that representation degradation arises partly from the optimization allocation across noise levels. For a scalar tangent model with fixed sampling density and allocation budget, we show that inverse conditional variance weighting of stochastic gradients minimizes stationary parameter variance. This reduces representation fluctuations and linear probe risk, thereby improving representation quality. Motivated by this result, we propose \textbf{Representation Preservation Diffusion Models (RPDM)}, which construct optimization allocation based on the contributions of data and noise components to the prediction target without explicitly estimating conditional variance. Empirically, RPDM mitigates representation degradation, thereby accelerating convergence and improving the quality of generation and representation without external representation supervision. Our main contributions are summarized as follows:

% \vspace{-0.5em}
\begin{itemize}[leftmargin=0.2in]%[label=\textbullet, leftmargin=1em, itemsep=0pt]
    \item We identify \emph{Representation Degradation}, showing that model outputs undergo substantial structural distortion across noise levels, accompanied by reduced representation quality.
    \item We develop an NTK analysis showing how irreducible residuals across noise levels in standard diffusion training can induce gradient interference through shared parameters, perturbing and thereby degrading learned representations and increasing linear probe risk.
    \item We derive an optimal allocation condition for preserving learned representations and propose RPDM, an architecture-agnostic framework that adaptively reallocates optimization across noise levels, which in turn improves representation and generation quality, accelerating training without external supervision.
\end{itemize}

\section{Background \& Preliminaries}
\label{sec:preliminaries}

In this section, we establish the mathematical foundations for analyzing representation degradation in diffusion training. We first formulate diffusion training as a unified regression problem and then introduce the neural tangent kernel (NTK) to characterize how parameter sharing couples optimization across noise levels.

\subsection{Diffusion Formulation}
\label{subsec:diffusion_formulation}

Let $x_0\sim q(x_0)$ be a clean data sample in $\mathbb R^d$ and let $\epsilon\sim\mathcal N(0,I)$ be independent Gaussian noise. The forward diffusion process is
\begin{equation}
q(x_t\mid x_0)=\mathcal N\big(x_t;\alpha_t x_0,\sigma_t^2I\big).
\label{eq:forward_process_conditional}
\end{equation}
Here $\alpha_t$ and $\sigma_t$ denote the signal and noise amplitudes, respectively. We adopt the log signal-to-noise ratio (log-SNR) $\lambda_t=\log(\alpha_t^2/\sigma_t^2)$ to unify the analysis across noise schedules. Assuming that $t\mapsto\lambda_t$ is strictly monotone with inverse $t(\lambda)$, we write $\alpha_\lambda\triangleq\alpha_{t(\lambda)}$ and $\sigma_\lambda\triangleq\sigma_{t(\lambda)}$. Diffusion training is then formulated as regression with the perturbed input and generalized prediction target
\begin{equation}
x_\lambda=\alpha_\lambda x_0+\sigma_\lambda\epsilon,\qquad y_\lambda=c_{x,\lambda}x_0+c_{\epsilon,\lambda}\epsilon.
\label{eq:noisy_sample_lambda}
\end{equation}
Here $c_{x,\lambda}$ and $c_{\epsilon,\lambda}$ are deterministic coefficients that specify the prediction target at each noise level.We denote the marginal distribution of $x_\lambda$ induced by the forward process by $q_\lambda$.

% \subsection{Diffusion Formulation}
% \label{subsec:diffusion_formulation}

% Let $x_0\sim q(x_0)$ be a clean data sample in $\mathbb R^d$. The forward diffusion process is
% \begin{equation}
% q(x_t\mid x_0)=\mathcal N\big(x_t;\alpha_t x_0,\sigma_t^2I\big).
% \label{eq:forward_process_conditional}
% \end{equation}
% Here $\alpha_t$ and $\sigma_t$ denote the signal and noise amplitudes, respectively. We adopt the log signal-to-noise ratio (log-SNR) $\lambda_t=\log(\alpha_t^2/\sigma_t^2)$ to unify the analysis across noise schedules. Assuming that $t\mapsto\lambda_t$ is strictly monotone with inverse $t(\lambda)$, we write $\alpha_\lambda\triangleq\alpha_{t(\lambda)}$ and $\sigma_\lambda\triangleq\sigma_{t(\lambda)}$. Diffusion training is then formulated as regression with the perturbed input and generalized prediction target
% \begin{equation}
% x_\lambda=\alpha_\lambda x_0+\sigma_\lambda\epsilon,\quad y_\lambda=c_{x,\lambda}x_0+c_{\epsilon,\lambda}\epsilon,\quad \epsilon\sim\mathcal N(0,I).
% \label{eq:noisy_sample_lambda}
% \end{equation}
% Here $\epsilon$ is independent of $x_0$. The deterministic coefficients $c_{x,\lambda}$ and $c_{\epsilon,\lambda}$ specify the prediction target at each noise level. We denote the marginal distribution of $x_\lambda$ by $q_\lambda$.

\subsection{Neural Tangent Kernel Framework}
\label{subsec:ntk_framework}

Parameter sharing in diffusion models allows optimization at one noise level to affect predictions at others. We adopt the NTK framework~\cite{jacot2018neural} to characterize this coupling through network Jacobians.

Let $f_{\theta,m}:\mathbb R^d\times[\lambda_{\min},\lambda_{\max}]\to\mathbb R^d$ be a network with parameters $\theta\in\mathbb R^P$ and width parameter $m$.

\begin{definition}[Joint neural tangent kernel]
\label{def:joint_ntk}
For the predictor $f_{\theta,m}$, we define the joint NTK over pairs of inputs and noise levels as
\begin{equation}
\Theta_{\theta,m}\big((x,\lambda),(x',\lambda')\big)
=\big(\nabla_\theta f_{\theta,m}(x,\lambda)\big)\big(\nabla_\theta f_{\theta,m}(x',\lambda')\big)^\top\in\mathbb R^{d\times d},
\label{eq:joint_ntk}
\end{equation}
where $\nabla_\theta f_{\theta,m}(x,\lambda)\in\mathbb R^{d\times P}$ denotes the Jacobian of the network predictions with respect to its parameters.
\end{definition}

The joint kernel characterizes how gradient updates at $(x',\lambda')$ affect the instantaneous prediction dynamics at $(x,\lambda)$. This reveals how estimation errors and stochastic gradient fluctuations propagate across noise levels.

\begin{assumption}[Optimization setup]
\label{assum:unified_assumption}
The following conditions govern our theoretical analysis:
\begin{itemize}[leftmargin=0.2in]%[leftmargin=1.5em, labelsep=0.5em]
    \item \textbf{Parameter sharing:} The same parameter vector $\theta$ is used across all noise levels.
    \item \textbf{Initialization and activation:} The network uses random initialization scaled with width and twice differentiable activation functions that are Lipschitz continuous.
    \item \textbf{Continuous optimization dynamics:} Optimization is modeled by gradient flow in continuous training time $\tau$. Training samples are drawn independently in the stochastic setting.
    \item \textbf{Integrability and regularity:} All relevant random quantities have finite second moments. Differentiation commutes with expectation and the training trajectories are well defined.
\end{itemize}
\end{assumption}

\section{Elucidating Representation Degradation}
\label{sec:elucidating_degradation}

\Cref{sec:preliminaries} established the theoretical framework for analyzing optimization coupling across noise levels in diffusion models. In the following, we use this framework to investigate how such coupling contributes to representation degradation during diffusion training. Detailed proofs are provided in \cref{apdx:sec:ntk_dynamics,apdx:sec:elbo_dynamics,apdx:sec:spectral_degradation}.

\subsection{Neural Tangent Kernel in Diffusion Models}
\label{subsec:ntk_dynamics}

To characterize the optimization dynamics across noise levels, we first express the diffusion training objective in the log-SNR domain under a unified regression framework.

Standard diffusion training~\cite{kingma2021variational,kingma2023understanding} minimizes a weighted regression loss over sampled timesteps. Up to terms independent of $\theta$, the diffusion ELBO takes the quadratic form
\begin{equation}
	\mathcal L_{\mathrm{ELBO},m}(\theta)=\mathbb E_t\!\left[\frac{w(\lambda_t)|\dot{\lambda}_t|}{2}\mathbb E_{x_0,\epsilon}\!\left[\|f_{\theta,m}(x_{\lambda_t},\lambda_t)-y_{\lambda_t}\|_2^2\right]\right].
	\label{eq:time_domain_elbo_objective}
\end{equation}
Here $\mathbb E_t$ denotes expectation under the timestep sampling law, $w(\lambda)\ge0$ is the loss weight at noise level $\lambda$, and $\dot{\lambda}_t$ denotes the rate of log-SNR change. These factors can be unified through a common optimization allocation in the log-SNR domain.

\begin{lemma}[Optimization equivalence]
	\label{lem:optimization_equivalence}
	Let $t\mapsto\lambda_t$ be a $C^1$ strictly monotone map onto $[\lambda_{\min},\lambda_{\max}]$ with $\dot{\lambda}_t\neq0$, and let $\rho(\lambda)$ be the induced log-SNR sampling density. Define the optimization allocation $M(\lambda)\triangleq\rho(\lambda)w(\lambda)|\dot{\lambda}_{t(\lambda)}|$ and assume that the objective is finite. Then \Eqref{eq:time_domain_elbo_objective} equals both the integral over $[\lambda_{\min},\lambda_{\max}]$ weighted by $\frac12M(\lambda)$ and the expectation under $\rho$ weighted by $\frac{M(\lambda)}{2\rho(\lambda)}$. The ratio satisfies $M(\lambda)/\rho(\lambda)=w(\lambda)|\dot{\lambda}_{t(\lambda)}|$ wherever $\rho(\lambda)>0$, with zero contribution elsewhere.
\end{lemma}

\Cref{lem:optimization_equivalence} shows that $\rho(\lambda)$, $w(\lambda)$ and $|\dot{\lambda}_{t(\lambda)}|$ are equivalent in their effects on the population objective whenever they yield the same optimization allocation $M(\lambda)$ for a fixed regression formulation. However, different sampling densities $\rho(\lambda)$ may induce different stochastic gradient fluctuations.

We now examine the composition of the diffusion loss from a Bayesian estimation perspective. We define the Bayes predictor $f_\lambda^\star(x)\triangleq\mathbb E[y_\lambda\mid x_\lambda=x]$ and the residual $\xi_\lambda\triangleq f_\lambda^\star(x_\lambda)-y_\lambda$, which satisfies $\mathbb E[\xi_\lambda\mid x_\lambda]=0$. The estimation error is defined as $r_{\theta,m}(x,\lambda)\triangleq f_{\theta,m}(x,\lambda)-f_\lambda^\star(x)$. Along the population gradient flow $\theta_m(\tau)$, we denote the predictor, estimation error and joint NTK by $f_{\tau,m}$, $r_{\tau,m}$ and $\Theta_{\tau,m}$, respectively.

\begin{lemma}[Bayes-risk decomposition]
\label{lem:elbo_bayes_decomposition}
For each fixed noise level $\lambda$, the local weighted ELBO density admits an exact decomposition into irreducible Bayes risk and excess risk:
\begin{equation}
\ell_{\tau,m}(\lambda)=\frac{M(\lambda)}{2}\mathbb E\|\xi_\lambda\|_2^2+\frac{M(\lambda)}{2}\mathbb E_{x\sim q_\lambda}\|r_{\tau,m}(x,\lambda)\|_2^2,
\label{eq:elbo_bayes_decomposition}
\end{equation}
where the two terms are denoted by $\ell_\lambda^\star$ and $\bar{\ell}_{\tau,m}(\lambda)$, respectively.
\end{lemma}
Since the Bayes risk is independent of the network parameters, population optimization acts only on the excess risk. Its underlying estimation error evolves according to the joint NTK dynamics.

\begin{proposition}[Finite-width training dynamics]
\label{prop:joint_ntk_dynamics}
Under the weighted population gradient flow and the assumptions in \cref{assum:unified_assumption}, the estimation error satisfies
\begin{equation}
\frac{\partial}{\partial\tau}r_{\tau,m}(x,\lambda)=-\int_{\lambda_{\min}}^{\lambda_{\max}}M(\lambda')\mathbb E_{x'\sim q_{\lambda'}}\!\left[\Theta_{\tau,m}\big((x,\lambda),(x',\lambda')\big)r_{\tau,m}(x',\lambda')\right]\mathrm d\lambda'.
\label{eq:joint_residual_ode}
\end{equation}
\end{proposition}

\Cref{prop:joint_ntk_dynamics} characterizes the coupling of estimation errors across noise levels without a frozen kernel approximation. Although $\xi_\lambda$ vanishes from the population drift, it can still induce stochastic gradient fluctuations. These dynamics provide the basis for our subsequent analysis of gradient fluctuations during training.

\subsection{Evidence Lower Bound Analysis}
\label{subsec:elbo_dynamics}

We now investigate the role of the irreducible residual in stochastic gradient fluctuations.
\begin{lemma}[Stochastic gradient decomposition]
\label{lem:gradient_forcing}
For fixed $\theta$, the stochastic gradient decomposes as
\begin{equation}
\nabla_\theta\widehat{\ell}_m=\frac{M(\lambda)}{\rho(\lambda)}\big(\nabla_\theta f_{\theta,m}(x_\lambda,\lambda)\big)^\top r_{\theta,m}(x_\lambda,\lambda)+\frac{M(\lambda)}{\rho(\lambda)}\big(\nabla_\theta f_{\theta,m}(x_\lambda,\lambda)\big)^\top\xi_\lambda.
\label{eq:gradient_decomposition}
\end{equation}
Here $\widehat{\ell}_m=\frac{M(\lambda)}{2\rho(\lambda)}\|f_{\theta,m}(x_\lambda,\lambda)-y_\lambda\|_2^2$ denotes the loss for $\lambda\sim\rho$. The second term has zero conditional mean given $(x_\lambda,\lambda)$ but may have nonzero covariance.
\end{lemma}
\begin{figure}[t]
    \centering
    \vspace{-1.0em}
    \includegraphics[width=\linewidth]{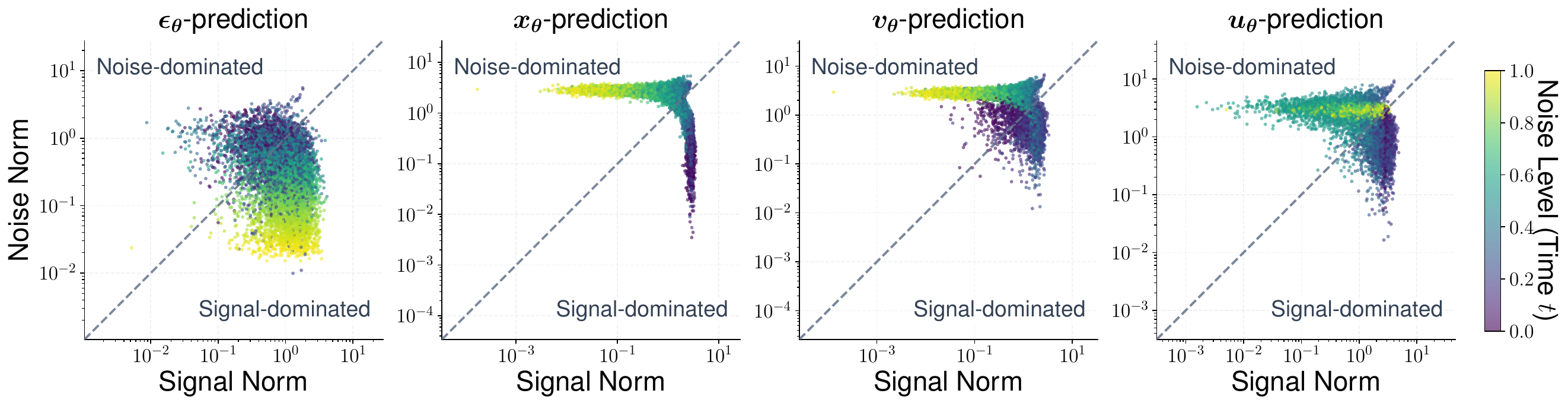}
    \vspace{-1.5em}
    \cprotect\caption{Signal--noise decomposition for different prediction targets~\footnotemark. The horizontal and vertical axes show $\|f_\lambda^\star(x_\lambda)\|_2$ and $\|\xi_\lambda\|_2$, respectively. Colors indicate diffusion time $t$, and the dashed diagonal separates noise-dominated (above) from signal-dominated (below) regimes.}
    \label{fig:gradient-decomposition}
    \vspace{-1.0em}
\end{figure}

\footnotetext{$\epsilon$ and $x$ denote the noise and data, $v$ denotes the velocity~\cite{salimans2022progressive}, and $u$ denotes the flow matching vector field~\cite{lipman2022flow}.}

As shown in \Cref{fig:gradient-decomposition}, irreducible residual norms are larger at high noise levels for $x_\theta$-prediction and at low noise levels for $\epsilon_\theta$-prediction. These residuals can induce stochastic gradient fluctuations even when the estimation error vanishes, suggesting that reaching the Bayes predictor does not necessarily eliminate parameter fluctuations during training. However, the magnitude of these fluctuations depends not only on the irreducible residual but also on the network Jacobian and the loss scaling factor $M(\lambda)/\rho(\lambda)$.

To characterize how these fluctuations propagate, we consider a continuous-time training model with a sampling interval $\Delta>0$. At each $j\Delta$, $j=0,1,\ldots$, independently draw $\Lambda_j\sim\rho$ and a fresh pair $(x_0,\epsilon)$, and define $(X_j,Y_j)=(x_{\Lambda_j},y_{\Lambda_j})$ and $\xi_j\triangleq f_{\Lambda_j}^\star(X_j)-Y_j$. The triplet $(\Lambda_j,X_j,Y_j)$ is held fixed for $\tau\in[j\Delta,(j+1)\Delta)$, with derivatives at sampling times evaluated immediately after the new sample is drawn.

Consider the parameter trajectory $\vartheta_m(\tau)$ satisfying, for $\tau\in[j\Delta,(j+1)\Delta)$,
\begin{equation}
    \frac{\mathrm d\vartheta_m(\tau)}{\mathrm d\tau}=-\frac{M(\Lambda_j)}{\rho(\Lambda_j)}\big(\nabla_\theta f_{\vartheta_m(\tau),m}(X_j,\Lambda_j)\big)^\top\big(r_{\vartheta_m(\tau),m}(X_j,\Lambda_j)+\xi_j\big).
    \label{eq:training_ode_bayes_forcing_quantified}
\end{equation}

Given these parameter dynamics, we can derive the evolution of the estimation error through the joint NTK.

\begin{proposition}[Continuous-time training dynamics]
\label{prop:sampled_continuous_time_dynamics}
Assume that the right-hand side of \Eqref{eq:training_ode_bayes_forcing_quantified} is locally Lipschitz in the parameters and that the trajectory exists on the interval considered. Let $f_{\tau,m}$, $r_{\tau,m}$, and $\Theta_{\tau,m}$ denote the predictor, estimation error, and joint NTK evaluated along $\vartheta_m(\tau)$ rather than the population trajectory. For almost every $\tau\in[j\Delta,(j+1)\Delta)$,
\begin{equation}
\frac{\partial}{\partial\tau}r_{\tau,m}(x,\lambda)=-\frac{M(\Lambda_j)}{\rho(\Lambda_j)}\Theta_{\tau,m}\big((x,\lambda),(X_j,\Lambda_j)\big)r_{\tau,m}(X_j,\Lambda_j)-\frac{M(\Lambda_j)}{\rho(\Lambda_j)}\Theta_{\tau,m}\big((x,\lambda),(X_j,\Lambda_j)\big)\xi_j.
\label{eq:residual_dynamics_training_quantified}
\end{equation}
At each sampling time, the conditional expectation of the right-hand side given the current parameters agrees with the population dynamics in \cref{prop:joint_ntk_dynamics} evaluated at the same parameters.
\end{proposition}

The first term couples estimation errors across noise levels, while the second captures the influence of the irreducible residual through the joint NTK. However, the evolving parameters may become dependent on the current sample, potentially breaking the agreement with population dynamics within each sampling interval. We next examine how this discrepancy affects the excess risk.

\begin{lemma}[Expected excess-risk evolution]
\label{lem:expected_local_excess_risk_evolution}
Under the assumptions of \cref{prop:sampled_continuous_time_dynamics}, the parameter dynamics in \Eqref{eq:training_ode_bayes_forcing_quantified} yield the following conditional expected loss derivative at each sampling time:
\begin{equation}
\mathbb E\!\left[\left.\frac{\mathrm d}{\mathrm d\tau}\mathcal L_{\mathrm{ELBO},m}(\vartheta_m(\tau))\right|_{\tau=j\Delta}\Bigm|\vartheta_m(j\Delta)\right]=-\left\|\nabla_\theta\mathcal L_{\mathrm{ELBO},m}(\vartheta_m(j\Delta))\right\|_2^2.
\label{eq:expected_population_loss_dissipation}
\end{equation}
\end{lemma}

\Eqref{eq:expected_population_loss_dissipation} establishes agreement with population gradient-flow dissipation at sampling times, while parameter fluctuations are quantified below under the frozen tangent approximation~\cite{lee2019wide}.

Let $J_m(x,\lambda)=\nabla_\theta f_{\theta_0,m}(x,\lambda)$ and assume that the linearized model realizes the Bayes predictor at $\theta_m^\star$. To characterize the contribution of each noise level to the covariance of the irreducible gradient component, define $\Gamma_{\xi,m}(\lambda)=\frac{M(\lambda)^2}{\rho(\lambda)}\mathbb E_{x\sim q_\lambda}[J_m(x,\lambda)^\top\operatorname{Cov}(y_\lambda\mid x_\lambda=x)J_m(x,\lambda)]$. The factor $M(\lambda)^2/\rho(\lambda)$ reveals that even with the same optimization allocation, different sampling densities can induce different gradient fluctuations.

Define the Hessian $H_{x,\lambda}=\frac{M(\lambda)}{\rho(\lambda)}J_m(x,\lambda)^\top J_m(x,\lambda)$ and its expectation $H=\mathbb E_{\lambda\sim\rho,\,x\sim q_\lambda}[H_{x,\lambda}]$. To characterize parameter fluctuations within each sampling interval, we define $B_s(x,\lambda)=\int_0^s e^{-uH_{x,\lambda}}\,\mathrm du$.

\begin{theorem}[Covariance propagation]
\label{thm:covariance_propagation}
Under the frozen tangent model and the training dynamics in \cref{prop:sampled_continuous_time_dynamics}, let $\delta_j=\vartheta_m(j\Delta)-\theta_m^\star$, $H_j=H_{X_j,\Lambda_j}$, $B_{j,s}=B_s(X_j,\Lambda_j)$, and $\Sigma_{\xi,j}=\frac{M(\Lambda_j)^2}{\rho(\Lambda_j)^2}J_m(X_j,\Lambda_j)^\top\operatorname{Cov}(Y_j\mid X_j,\Lambda_j)J_m(X_j,\Lambda_j)$. For $0\le s<\Delta$, the parameter covariance satisfies
\begin{equation}
\operatorname{Cov}(\vartheta_m(j\Delta+s))=\operatorname{Cov}(e^{-sH_j}\delta_j)+\mathbb E[B_{j,s}\Sigma_{\xi,j}B_{j,s}].
\label{eq:covariance_propagation}
\end{equation}
Furthermore, writing $\bar{\mathcal L}_{\tau,m}=\int_{\lambda_{\min}}^{\lambda_{\max}}\bar{\ell}_{\tau,m}(\lambda)\,\mathrm d\lambda$, its expected value decomposes exactly into a mean error term $\frac12\|\mathbb E[\vartheta_m(\tau)]-\theta_m^\star\|_H^2$ and a covariance term $\frac12\operatorname{Tr}(H\operatorname{Cov}(\vartheta_m(\tau)))$, where $\|v\|_H^2=v^\top Hv$.
\end{theorem}

The covariance decomposition separates fluctuations arising from the parameter error and the irreducible residual. We now characterize the resulting excess risk.

\begin{theorem}[Expected excess-risk bound]
\label{thm:elbo_asymptotics}
Under the assumptions of \cref{thm:covariance_propagation} and the initialization and subspace conditions in \cref{apdx:sec:elbo_dynamics}, suppose $0\preceq H_{x,\lambda}\preceq LI$ and $H\succeq\mu I$ on the active subspace, with $0<\mu\le L<\infty$. Let $H_\lambda=M(\lambda)\mathbb E_{x\sim q_\lambda}[J_m(x,\lambda)^\top J_m(x,\lambda)]$ and $V_{\xi,m}=\int_{\lambda_{\min}}^{\lambda_{\max}}\operatorname{Tr}(\Gamma_{\xi,m}(\lambda))\,\mathrm d\lambda<\infty$. For $\tau=j\Delta+s$, $0\le s<\Delta$, we have $\mathbb E[\bar{\mathcal L}_{\tau,m}]=O(q_\Delta^j+\Delta)$ as $\Delta\downarrow0$, where $q_\Delta=1-\frac{\mu}{L}(1-e^{-2L\Delta})$.

The stationary parameter covariance satisfies $T_\Delta=\Delta S_{\xi,m}+O(\Delta^2)$, where the leading coefficient $S_{\xi,m}$ solves the Lyapunov equation $HS_{\xi,m}+S_{\xi,m}H=\int_{\lambda_{\min}}^{\lambda_{\max}}\Gamma_{\xi,m}(\nu)\,\mathrm d\nu$. For each fixed noise level $\lambda$ with finite $H_\lambda$, the stationary local excess risk satisfies
\begin{equation}
\lim_{j\to\infty}\mathbb E[\bar{\ell}_{j\Delta+s,m}(\lambda)]=\frac{\Delta}{2}\operatorname{Tr}(H_\lambda S_{\xi,m})+O_\lambda(\Delta^2),
\label{eq:local_elbo_asymptotic}
\end{equation}
as $\Delta\downarrow0$, uniformly over $0\le s<\Delta$. The corresponding stationary global excess risk satisfies $\lim_{j\to\infty}\mathbb E[\bar{\mathcal L}_{j\Delta+s,m}]=\frac{V_{\xi,m}}{4}\Delta+O(\Delta^2)$, which is $\Theta(\Delta)$ whenever $V_{\xi,m}>0$.
\end{theorem}

The expected excess risk is bounded by a geometrically decaying term and an $O(\Delta)$ stationary floor. For $M(\lambda)>0$, the stationary MSE differs from the Bayes risk by $\Delta\operatorname{Tr}(H_\lambda S_{\xi,m})/M(\lambda)+O_\lambda(\Delta^2)$. Even when the prediction target at a given noise level is readily recoverable, irreducible residuals at other noise levels can increase its prediction error through shared network parameters. This explains why a low Bayes risk does not necessarily imply a low prediction error, consistent with the behavior of $\epsilon$-prediction at high noise in \cref{fig:bayes-risk}.

\begin{figure}[t]
    \vspace{-1.0em}
    \centering
    \includegraphics[width=\linewidth]{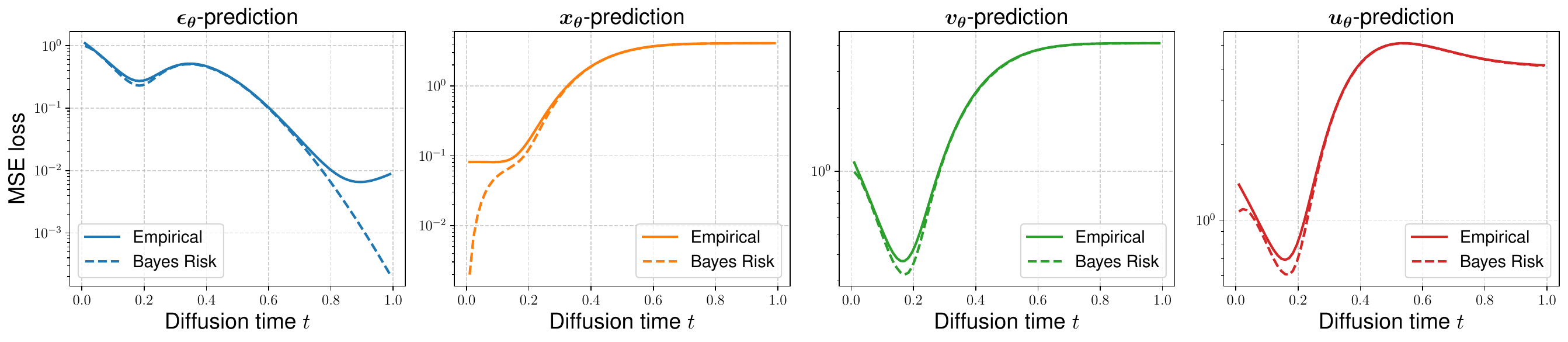}
    \vspace{-1.5em}
    \cprotect\caption{Empirical MSE and Bayes risk across diffusion time $t$ for different prediction targets. The Bayes optimal predictor is $\mathbb{E}[y \mid x_t]$, whose expected squared error defines the Bayes risk. The gap between empirical MSE and Bayes risk reflects the excess error for each parameterization.}
    \label{fig:bayes-risk}
    \vspace{-1.0em}
\end{figure}

\subsection{How Representation Degradation Emerges}
\label{subsec:representation_degradation}

The parameter fluctuations characterized above can also affect hidden representations through shared network parameters. We first examine how the irreducible residual contributes to gradient fluctuations in hidden layers.

\begin{proposition}[Gradient covariance in hidden layers]
\label{prop:quantitative_representation_degradation}
Suppose the network admits the factorization $f_{\theta,m}(x,\lambda)=W_{\theta,m}h_{\theta,m}(x,\lambda)$, where $h_{\theta,m}(x,\lambda)\in\mathbb R^m$. Let $\theta_h$ denote the parameters affecting only $h_{\theta,m}$. For fixed $\theta$ and with the Jacobian evaluated at $(x_\lambda,\lambda)$, the gradient component induced by the irreducible residual in \Eqref{eq:gradient_decomposition} has zero conditional mean given $(x_\lambda,\lambda)$ and covariance
\begin{equation}
\left(\frac{M(\lambda)}{\rho(\lambda)}\right)^2(\nabla_{\theta_h}h_{\theta,m})^\top W_{\theta,m}^\top\operatorname{Cov}(y_\lambda\mid x_\lambda)W_{\theta,m}(\nabla_{\theta_h}h_{\theta,m}).
\label{eq:representation_gradient_conditional_covariance}
\end{equation}
\end{proposition}

\Cref{prop:quantitative_representation_degradation} shows how irreducible residuals induce gradient fluctuations in hidden layers. To understand their effects on hidden representations, we first examine the relationship between the NTK and feature spectra.

Fix a noise level $\lambda$ and let $J_i=J_m(x_i,\lambda)$ and $h_i=h_{\theta_0,m}(x_i,\lambda)$ for evaluation inputs independent of training. Define the scalar NTK and feature Gram matrices by $K_\lambda(i,j)=\langle J_i,J_j\rangle_{\mathrm F}$ and $G_\lambda(i,j)=\langle h_i,h_j\rangle$, respectively. For the ordered eigenpairs $(\kappa_a^\lambda,u_a^\lambda)$ of $K_\lambda$ with orthonormal eigenvectors, define $B_a^\lambda=\sum_i(u_a^\lambda)_iJ_i$.

\begin{lemma}[NTK modes and hidden feature spectra]
\label{lem:ntk_mode_accessibility}
For each NTK mode, $\|B_a^\lambda\|_{\mathrm F}^2=\kappa_a^\lambda$ and $\|B_a^\lambda v\|_2^2\le\kappa_a^\lambda\|v\|_2^2$ for any parameter displacement $v$. If $W_{\theta,m}\in\mathbb R^{d\times m}$ is an independently parameterized, unconstrained trainable linear readout, then $K_\lambda\succeq dG_\lambda$, implying $d\lambda_a(G_\lambda)\le\kappa_a^\lambda$ and $\operatorname{rank}(G_\lambda)\le\operatorname{rank}(K_\lambda)$.
\end{lemma}

Under the stated readout conditions, the NTK spectrum provides an upper bound on the hidden feature spectrum. We next quantify the effects of parameter fluctuations on NTK modes and hidden representations using the stationary covariance from \cref{thm:elbo_asymptotics} and the first-order approximation $\widetilde h_{\theta,m}(x,\lambda)=h_{\theta_0,m}(x,\lambda)+A_m(x,\lambda)(\theta-\theta_0)$, where $A_m(x,\lambda)=\nabla_\theta h_{\theta_0,m}(x,\lambda)$.

\begin{figure}[t]
    \vspace{-1.0em}
    \centering
    \begin{subfigure}[t]{0.56\linewidth}
        \centering
        \includegraphics[width=\linewidth]{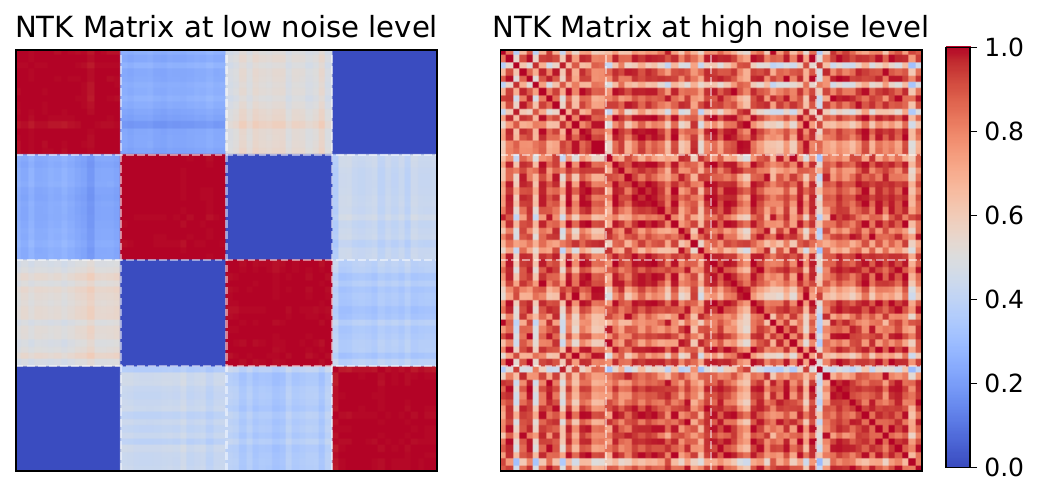}
        \vspace{-1.2em}
        \caption{NTK heatmaps}
        \label{fig:ntk-heatmaps}
    \end{subfigure}
    \hfill
    \begin{subfigure}[t]{0.42\linewidth}
        \centering
        \raisebox{-0.5em}{\includegraphics[width=\linewidth]{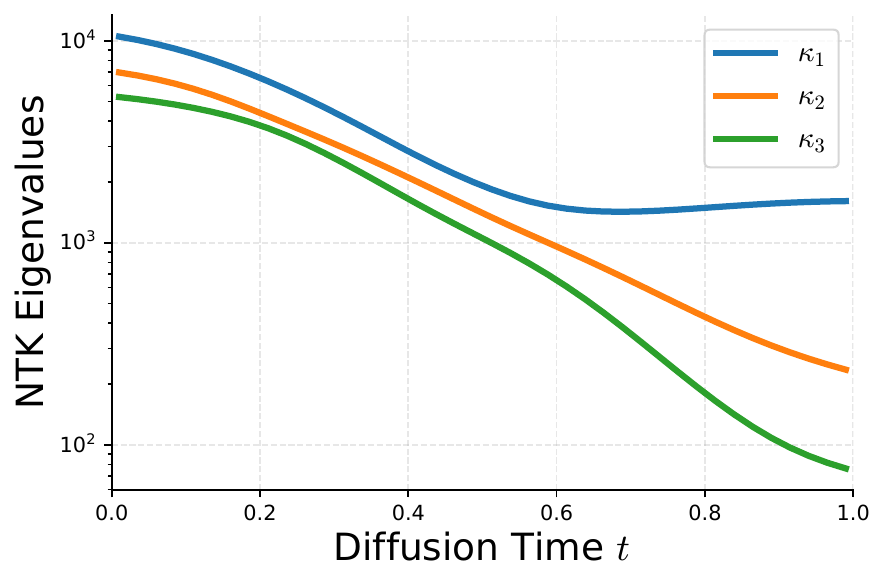}}
        \vspace{-1.2em}
        \caption{NTK spectrum}
        \label{fig:ntk-spectrum}
    \end{subfigure}
    \vspace{-0.5em}
    \caption{NTK analysis across diffusion noise levels. \cref{fig:ntk-heatmaps} NTK heatmaps show a clear block structure at low noise that becomes less distinct at high noise. \cref{fig:ntk-spectrum} NTK eigenvalues decrease for several modes as noise increases, indicating weaker mode-wise error correction.}
    \vspace{-1.0em}
    \label{fig:ntk-analysis}
\end{figure}

\begin{theorem}[Long-time NTK responses and feature fluctuations]
\label{thm:spectral_local_elbo_bound}
Under the frozen tangent assumptions of \cref{thm:elbo_asymptotics}, let $\vartheta_k=\vartheta_m(k\Delta+s)$ and $\delta_k=\vartheta_k-\theta_m^\star$, with $0\le s<\Delta$. The stationary cross-covariance between $B_a^\lambda\delta_k$ and $B_b^\lambda\delta_k$ has leading term $\Delta B_a^\lambda S_{\xi,m}(B_b^\lambda)^\top$, with an $O_\lambda(\Delta^2)$ remainder. In particular, the mean-square perturbation along each NTK mode satisfies $\lim_{k\to\infty}\mathbb E\|B_a^\lambda\delta_k\|_2^2\le\Delta V_{\xi,m}\kappa_a^\lambda/(2\mu)+O_\lambda(\Delta^2)$.

For fixed evaluation pairs $(x,\lambda)$ and $(x',\lambda')$, let $A=A_m(x,\lambda)$, $A'=A_m(x',\lambda')$, $h_k=\widetilde h_{\vartheta_k,m}(x,\lambda)$, and $h_k'=\widetilde h_{\vartheta_k,m}(x',\lambda')$. Assuming finite second moments of the evaluation feature Jacobians, their stationary covariance satisfies
\begin{equation}
\lim_{k\to\infty}\operatorname{Cov}(h_k,h_k')=\Delta AS_{\xi,m}(A')^\top+O(\Delta^2).
\label{eq:feature_stationary_covariance}
\end{equation}
The expansions hold uniformly over $0\le s<\Delta$ for fixed evaluation pairs. When $\operatorname{Tr}(AS_{\xi,m}A^\top)>0$, the mean-square deviation from the features at $\theta_m^\star$ is $\Theta(\Delta)$, with a root-mean-square deviation of $\Theta(\sqrt{\Delta})$. Thus, feature fluctuations persist in the long-time limit for a fixed sampling interval.
\end{theorem}

\Eqref{eq:feature_stationary_covariance} shows that gradient fluctuations at one noise level can induce feature fluctuations at other noise levels through shared network parameters. As shown in \Cref{fig:ntk-analysis}, the NTK block structure becomes less distinct at high noise, while several eigenvalues decrease. Under the corresponding local dynamics with a frozen scalar kernel, smaller eigenvalues imply slower error correction along the associated modes.

We now examine how feature fluctuations affect the performance of linear probes. For a deterministic linear probe $b\in\mathbb R^m$ and a target $z$, define the population squared-error risk $\mathcal R_\lambda(b;\theta)=\mathbb E_{x_\lambda,z}[(b^\top\widetilde h_{\theta,m}(x_\lambda,\lambda)-z)^2]$.

\begin{corollary}[Long-time risk of linear probes]
\label{cor:representation_linear_probe_risk}
Under \cref{thm:spectral_local_elbo_bound}, fix $\lambda$, a square-integrable target $z$, and a deterministic probe $b$. Suppose $(x_\lambda,z)$ is independent of training and $\mathbb E\|A_m(x_\lambda,\lambda)\|_{\mathrm F}^2<\infty$. Define $Q=\mathbb E_{x\sim q_\lambda}[A_m(x,\lambda)S_{\xi,m}A_m(x,\lambda)^\top]$. Then
\begin{equation}
\lim_{k\to\infty}\left(\mathbb E[\mathcal R_\lambda(b;\vartheta_k)]-\mathcal R_\lambda(b;\theta_m^\star)\right)=\Delta b^\top Qb+O_\lambda(\Delta^2).
\label{eq:representation_probe_excess}
\end{equation}
If the feature second-moment matrix at $\theta_m^\star$ is positive definite, the same expansion holds for the optimal stationary expected risk with $b$ replaced by the reference optimal probe $b_0$.
\end{corollary}

The quality of representations at a given noise level depends on optimization across the entire noise range. Irreducible residuals at other noise levels can impair their linear predictive performance through shared network parameters, regardless of how easily the local prediction target can be recovered. When $b_0^\top Qb_0>0$, the excess risk remains of order $\Delta$ even after optimizing the linear probe.

\begin{table}[htbp]
\centering
% \vspace{-0.5em}
\caption{Summary of quantitative results and their implications in \cref{sec:elucidating_degradation}. Stationary limits precede $\Delta\to0$.}
\vspace{-0.5em}
\label{tab:asymptotic_summary}
\begingroup
\small
\setlength{\tabcolsep}{3pt}
\renewcommand{\arraystretch}{1.25}
\begin{tabular*}{\linewidth}{@{\extracolsep{\fill}}llm{0.26\linewidth}c@{}}
\toprule
\textbf{Theoretical Quantity} & \textbf{Rate / Leading-order Result} & \textbf{Implication} & \textbf{Reference} \\
\midrule

Excess-risk convergence
& $O(q_\Delta^j)+O(\Delta)$
& Geometrically decaying bound with an $O(\Delta)$ floor.
& \multirow[c]{4}{*}[-2.3em]{\cref{thm:elbo_asymptotics}} \\

\addlinespace[3pt]
Stationary covariance
& $T_\Delta=\Delta S_{\xi,m}+O(\Delta^2)$
& Irreducible residuals can sustain parameter fluctuations.
& \\

\addlinespace[3pt]
Local excess risk
& $\frac{\Delta}{2}\operatorname{Tr}(H_\lambda S_{\xi,m})+O_\lambda(\Delta^2)$
& Prediction error depends on fluctuations across noise levels.
& \\

\addlinespace[3pt]
Global excess risk
& $\frac{V_{\xi,m}}{4}\Delta+O(\Delta^2)$
& A positive risk floor persists when $V_{\xi,m}>0$.
& \\

\midrule

NTK mode perturbation
& $\leq\frac{V_{\xi,m}\kappa_a^\lambda}{2\mu}\Delta+O_\lambda(\Delta^2)$
& Smaller eigenvalues give tighter perturbation bounds.
& \multirow[c]{2}{*}[-0.6em]{\cref{thm:spectral_local_elbo_bound}} \\

\addlinespace[3pt]
Feature covariance
& $\Delta AS_{\xi,m}(A')^\top+O(\Delta^2)$
& Parameter fluctuations can affect features across noise levels.
& \\

\midrule

Linear-probe excess risk
& $\Delta b^\top Qb+O_\lambda(\Delta^2)$
& Excess risk can persist with optimal linear probes.
& \cref{cor:representation_linear_probe_risk} \\

\bottomrule
\end{tabular*}
\endgroup
% \vspace{-1em}
\end{table}

The results in \cref{tab:asymptotic_summary} reveal that representation degradation can arise from optimization coupling across noise levels rather than local prediction difficulty alone.

% \Cref{tab:asymptotic_summary} summarizes the main theoretical results. Representation degradation can arise from optimization coupling across noise levels rather than solely from the difficulty of individual prediction tasks. Even at noise levels where prediction targets are readily recoverable, gradient fluctuations induced by irreducible residuals at other noise levels can impair the predictive quality of hidden representations.
\section{Representation Preservation Diffusion Models}
\label{sec:elucidated_representation_diffusion}

The preceding analysis shows that the optimization allocation determines how the conditional covariance of the prediction target contributes to parameter fluctuations and representation degradation. This motivates Representation Preservation Diffusion Models (RPDM), which aim to reduce stationary parameter covariance by reallocating optimization across noise levels. Proofs are provided in \cref{apdx:sec:unified_weight}.

\begin{proposition}[Optimal variance allocation]
\label{prop:unified_target_adaptive_weight}
Under the assumptions of \cref{thm:elbo_asymptotics}, consider a scalar tangent model ($d=1$, $J_m=1$) with fixed sampling density $\rho$. Among nonnegative allocations supported on $\{\lambda:\rho(\lambda)>0\}$ with $\int_{\lambda_{\min}}^{\lambda_{\max}}M(\lambda)\,\mathrm d\lambda=1$, an allocation $M_\star$ minimizes the leading coefficient of the stationary parameter variance if the following holds $\rho$-almost everywhere:
\begin{equation}
\frac{M_\star(\lambda)}{\rho(\lambda)}\mathbb E_{x\sim q_\lambda}[\operatorname{Var}(y_\lambda\mid x)]=\mathrm{const}\geq0.
\label{eq:optimal_variance_allocation}
\end{equation}
\end{proposition}
We construct the RPDM allocation from the second moments of the scaled source components. For the generalized prediction target in \Eqref{eq:noisy_sample_lambda}, with $q_x\footnote{In our latent diffusion experiments, we set $q_x=1$ after latent normalization, as empirically validated in \cref{par:latent_statistics}.}=d^{-1}\mathbb E\|x_0\|_2^2$, we define
\begin{equation}
M_1(\lambda)\propto\rho(\lambda)|\dot{\lambda}_{t(\lambda)}|\sqrt{c_{x,\lambda}^2\alpha_\lambda^2q_x+c_{\epsilon,\lambda}^2\sigma_\lambda^2}.
\label{eq:effective_amplitude_score}
\end{equation}
Under uniform timestep sampling, the allocation is proportional to the loss weight. We implement RPDM solely through loss weighting, normalizing its allocation to match the baseline integral.

\begin{remark}[Equivalent population objectives]
By \cref{lem:optimization_equivalence}, for a fixed noise schedule, sampling density $\rho_0$ and loss weight $w(\lambda)$ with positive finite mean yield the same allocation as density $\rho_1(\lambda)=\rho_0(\lambda)w(\lambda)/\mathbb E_{\rho_0}[w]$ and constant weight $\mathbb E_{\rho_0}[w]$. The population objectives coincide, but gradient covariance depends on $\rho$ through $M(\lambda)^2/\rho(\lambda)$.
\end{remark}

\begin{theorem}[Asymptotic linear probe risk]
\label{thm:erd_representation_improvement}
Consider baseline and RPDM allocations $M_0$ and $M_1$ with equal integrals and common sampling density $\rho$. Assume a shared frozen tangent model satisfying \cref{thm:spectral_local_elbo_bound,cor:representation_linear_probe_risk}, where $\theta_m^\star$ realizes the Bayes predictor for both allocations and both Hessians are positive definite on the active subspace. Let $S_{\xi,m}^{(i)}$ denote the leading stationary covariance coefficient and $\mathcal R_\lambda^{(i)}(b)$ the stationary expected squared-error risk of a fixed probe $b$ at noise level $\lambda$. Then
\begin{equation}
\mathcal R_\lambda^{(1)}(b)-\mathcal R_\lambda^{(0)}(b)=\Delta\mathbb E_{x\sim q_\lambda}\left[b^\top A_m(x,\lambda)\big(S_{\xi,m}^{(1)}-S_{\xi,m}^{(0)}\big)A_m(x,\lambda)^\top b\right]+O_\lambda(\Delta^2).
\label{eq:erd_probe_risk_improvement}
\end{equation}
If the coefficient of $\Delta$ is negative, RPDM reduces stationary probe risk by $\Theta(\Delta)$. Under the scalar model of \cref{prop:unified_target_adaptive_weight}, the normalized RPDM allocation minimizes the leading variance coefficient if \Eqref{eq:optimal_variance_allocation} holds.
\end{theorem}
In the scalar tangent model, if \Eqref{eq:optimal_variance_allocation} holds with a positive constant, the optimal ratio $M_\star(\lambda)/\rho(\lambda)$ varies inversely with the mean-square irreducible residual. Reallocating optimization can thus reduce feature fluctuations along predictive directions and improve representation quality without increasing the total allocation.
\section{Experiments}
In this section, we empirically demonstrate that RPDM improves representation quality and generation performance. Detailed experimental configurations and additional results are provided in \cref{app:sec:exp_sup}. In addition, we use a 2D Gaussian mixture model (GMM) to mimic semantically separable latent spaces~\cite{zheng2025diffusion} and investigate representation degradation during diffusion training in \cref{sec:toy_experiments}.

% \begin{figure}[htbp]
\begin{figure}[t]
\vspace{-1.0em}
\centering
\resizebox{\textwidth}{!}{%
\begin{minipage}{0.415\textwidth}
\centering\small
\captionof{table}{\textbf{FID comparisons with DiTs, SiTs and U-ViT} on ImageNet 256$\times$256 without classifier-free guidance (CFG). $\downarrow$ denotes lower values are better; Iter. indicates the number of training iterations.}
\label{tab:wo_cfg}
\vspace{-0.05in}

\resizebox{\linewidth}{!}{%
\begin{tabular}{lccc}
\toprule
Model & \#Params & Iter. & FID$\downarrow$ \\
\midrule
% DiT-S/2 & 33M & 400K & 68.4 \\
% \textbf{RPDM (ours) ($\bm{v}_{\theta}$)} & 33M & \textbf{400K} & \textbf{57.19} \\
% \midrule
DiT-B/2 & 130M & 400K & 43.4 \\
SiT-B/2 & 130M & 400K & 33.0 \\
+ \textbf{RPDM (ours)} & \textbf{130M} & \textbf{400K} & \textbf{32.3} \\
\midrule
DiT-L/2 & 458M & 400K & 23.3 \\
SiT-L/2 & 458M & 400K & 18.8 \\
+ \textbf{RPDM (ours)} & \textbf{458M} & \textbf{400K} & \textbf{17.2} \\
\midrule
DiT-XL/2 & 675M & 400K & 19.5 \\
SiT-XL & 675M & 400K & 17.2 \\
+ \textbf{RPDM (ours)} & \textbf{675M} & \textbf{400K} & \textbf{15.0} \\
\midrule
U-ViT-M/2 & 131M & 400K & 27.7 \\
% REPA & 130M & 400K & 24.4 \\
+ \textbf{RPDM (ours)} & \textbf{131M} & \textbf{400K} & \textbf{19.9} \\
\midrule
U-ViT-L/2 & 287M & 400K & 20.4 \\
% REPA & 458M & 400K & 9.7 \\
+ \textbf{RPDM (ours)} & \textbf{287M} & \textbf{400K} & \textbf{15.0} \\
\midrule
U-ViT-H/2 & 501M & 400K & 13.7 \\
Min-SNR & 501M & 400K & 11.7 \\
% REPA & 675M & 400K & 7.9 \\
+ \textbf{RPDM (ours)} & \textbf{501M} & \textbf{400K} & \textbf{8.9} \\
\bottomrule
\end{tabular}%
}
\end{minipage}%
\hspace{0.015\textwidth}%
\begin{minipage}{0.52\textwidth}
\centering\large
\captionof{table}{\textbf{System-level comparison} on ImageNet 256$\times$256 with CFG. $\downarrow$ and $\uparrow$ indicate whether lower or higher values are better, respectively. Results that include additional CFG scheduling are marked with an asterisk (*), where the guidance interval from~\cite{kynkaanniemi2024applying} is applied for.}
\label{tab:main}
\vspace{-0.06in}

\resizebox{\linewidth}{!}{%
\begin{tabular}{lcccccc}
\toprule
Model & Epochs & FID$\downarrow$ & sFID$\downarrow$ & IS$\uparrow$ & Pre.$\uparrow$ & Rec.$\uparrow$ \\
\midrule

\multicolumn{7}{l}{\emph{Pixel diffusion}} \\
ADM-U~\cite{dhariwal2021diffusion} & 400 & 3.94 & 6.14 & 186.7 & 0.82 & 0.52 \\
VDM$++$~\cite{kingma2021variational} & 560 & 2.40 & - & 225.3 & - & - \\
Simple diffusion~\cite{hoogeboom2023simple} & 800 & 2.77 & - & 211.8 & - & - \\
CDM~\cite{ho2022cascaded} & 2160 & 4.88 & - & 158.7 & - & - \\
\arrayrulecolor{black!40}\midrule

\multicolumn{7}{l}{\emph{Latent diffusion, U-Net}\vspace{0.02in}} \\
LDM-4~\cite{rombach2022high} & 200 & 3.60 & - & 247.7 & 0.87 & 0.48 \\
\arrayrulecolor{black!40}\midrule

\multicolumn{7}{l}{\emph{Latent diffusion, Transformer}} \\
MaskDiT~\cite{zheng2023fast} & 1600 & 2.28 & 5.67 & 276.6 & 0.80 & 0.61 \\
SD-DiT~\cite{zhu2024sd} & 480 & 3.23 & - & - & - & - \\
Min-SNR*~\cite{hang2023efficient} & 1400 & 1.57 & - & - & - & - \\
\arrayrulecolor{black!30}\cmidrule(lr){1-7}
DiT-XL/2~\cite{peebles2023scalable} & 1400 & 2.27 & 4.60 & 278.2 & \textbf{0.83} & 0.57 \\
SiT-XL/2~\cite{ma2024sit} & 1400 & 2.06 & 4.50 & 270.3 & 0.82 & 0.59 \\
% REPA~\cite{yu2024representation} & 200 & 1.96 & 4.49 & 264.0 & 0.82 & 0.60 \\
% REPA~\cite{yu2024representation} & 800 & 1.80 & 4.50 & 284.0 & 0.81 & 0.61 \\
REPA*~\cite{yu2024representation} & 800 & \textbf{1.42} & 4.70 & 305.7 & 0.80 & \textbf{0.65} \\
\arrayrulecolor{black}\midrule

\multicolumn{7}{l}{\emph{Latent diffusion, Transformer + U-Net hybrid}} \\
DiffiT*~\cite{hatamizadeh2024diffit} & - & 1.73 & - & 276.5 & 0.80 & 0.62 \\
MDTv2-XL/2*~\cite{gao2023mdtv2} & 1080 & 1.58 & 4.52 & 314.7 & 0.79 & 0.65 \\
\arrayrulecolor{black!30}\cmidrule(lr){1-7}
U-ViT-H/2~\cite{bao2023all} & 400 & 2.29 & 5.68 & 263.9 & 0.82 & 0.57 \\
% Ours & 200 & 1.91 & \textbf{4.46} & 272.3 & 0.79 & 0.57 \\
\textbf{+ RPDM (ours)} & 800 & 1.89 & 4.48 & 282.5 & 0.82 & 0.60 \\
\textbf{+ RPDM (ours)*} & 800 & 1.45 & 4.51 & \textbf{315.4} & 0.81 & 0.62 \\
\arrayrulecolor{black}\bottomrule
\end{tabular}%
}
\end{minipage}%
}
\vspace{-0.5em}
\end{figure}

\subsection{System-Level Comparison}
We thoroughly evaluate RPDM on both DiT and U-ViT architectures using the ImageNet $256{\times}256$ dataset. In \cref{tab:wo_cfg}, we report FID scores without classifier-free guidance (CFG). Under the same training schedule (400K iterations), RPDM consistently outperforms the original DiT and U-ViT, as well as the SiT~\cite{ma2024sit}, across all model scales. We also evaluate system-level performance under CFG, as shown in \cref{tab:main}. Compared with strong baselines including LDM~\cite{rombach2022high}, U-ViT-H/2~\cite{bao2023all}, and SD-DiT~\cite{zhu2024sd}, RPDM achieves highly competitive performance, significantly narrowing the gap with the recent REPA~\cite{yu2024representation}.

\subsection{Representation Quality Evaluation}

\textbf{Effective Feature Rank.}
We assess representation dimensionality using effective rank. We train the baseline and RPDM on CIFAR-10~\cite{krizhevsky2009learning} for 400K iterations with the same ViT-S~\cite{dosovitskiy2020image} backbone for unconditional diffusion. We then compute effective ranks of features from Layers $\{1,4,7,10,13\}$ at noise levels $\{0.0,0.2,0.4,0.6,0.8,1.0\}$. In \cref{fig:effective_rank}, colors distinguish layers, and solid and dashed curves denote RPDM and the baseline, respectively. RPDM achieves substantially higher effective rank across noise levels in deeper layers, particularly the final layer. This suggests better preservation of deep representation diversity~\cite{garrido2023rankme,jing2022understanding} during training, alongside improved generative performance, improving FID from 26.54 to \textbf{19.31}.

\begin{figure}[htbp]
    \vspace{-0.5em}
    \centering
    \begin{subfigure}[t]{0.49\linewidth}
        \centering
        \includegraphics[width=\linewidth]{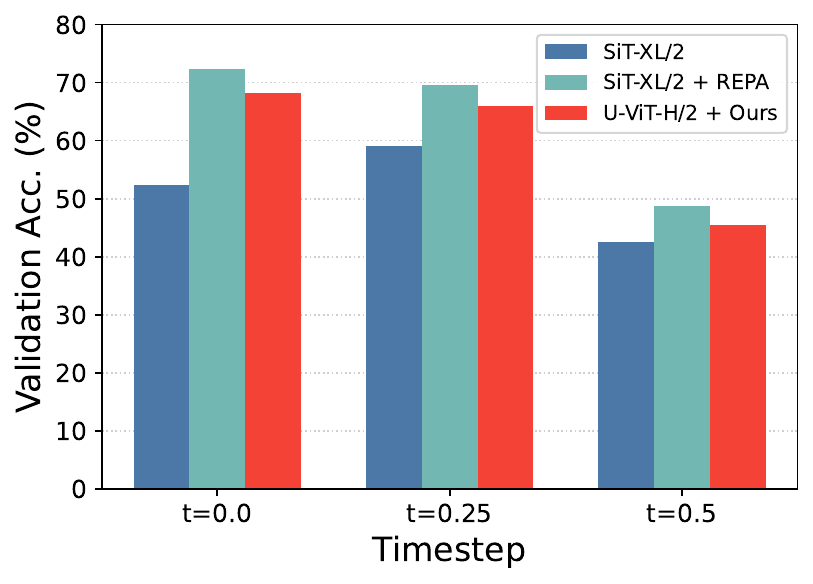}
        \vspace{-1.5em}
        \caption{Linear probing accuracy across diffusion timesteps.}
        \label{fig:linear_probe}
    \end{subfigure}
    \hfill
    \begin{subfigure}[t]{0.49\linewidth}
        \centering
        \includegraphics[width=\linewidth]{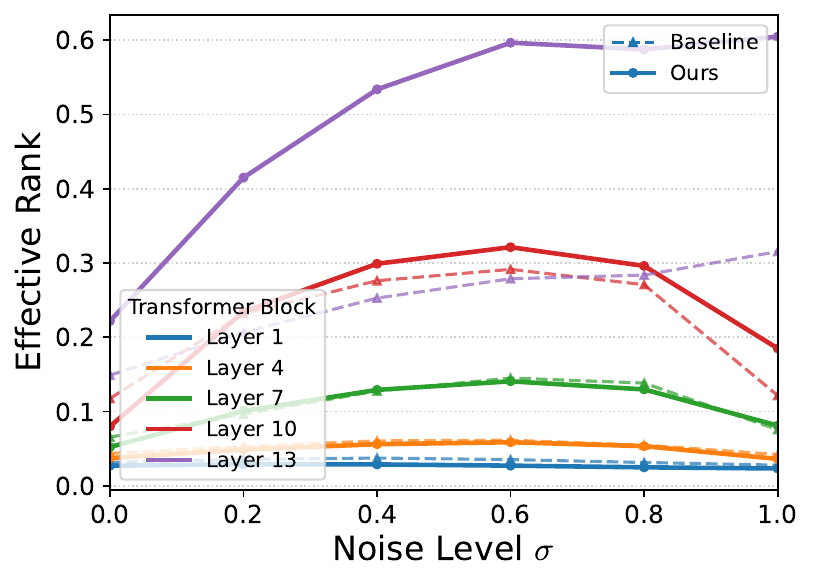}
        \vspace{-1.5em}
        \caption{Layer-wise effective rank across noise levels.}
        \label{fig:effective_rank}
    \end{subfigure}
    \vspace{-0.5em}
    \caption{Evaluation of the quality of learned representations across diffusion noise levels.}
    \vspace{-0.5em}
    \label{fig:representation_quality}
\end{figure}

\textbf{Linear Probe Accuracy.}
We evaluate representation quality by training linear classifiers on frozen diffusion features extracted at timesteps $t\in\{0.0,0.25,0.5\}$. As shown in \cref{fig:linear_probe}, RPDM outperforms SiT across all timesteps despite training for 800 rather than 1400 epochs, indicating better semantic accessibility throughout the diffusion trajectory. REPA achieves the highest accuracy with 800 epochs, leaving a moderate gap to methods with external representation supervision. These results demonstrate that RPDM improves representation quality and mitigates representation degradation during diffusion training.

\textbf{NTK Eigenspectrum.}
Following the setup in the previous paragraph, we analyze the fixed noise NTK $\Theta_\lambda = J_\lambda J_\lambda^\top$, where $J_\lambda$ denotes the Jacobian of predictions with respect to model parameters. We compute its top 50 eigenvalues using matrix-free Lanczos~\cite{lanczos1950iteration} on $J_\lambda^\top J_\lambda$, with 10,000 fixed samples across five uniformly spaced diffusion timesteps. As shown in \Cref{fig:ntk_eigenspectrum}, RPDM consistently raises the lower spectral envelope across training stages, particularly at higher eigenvalue ranks. 

\begin{figure}[t]
    \centering
    \vspace{-1.5em}
    \includegraphics[width=1.0\linewidth]{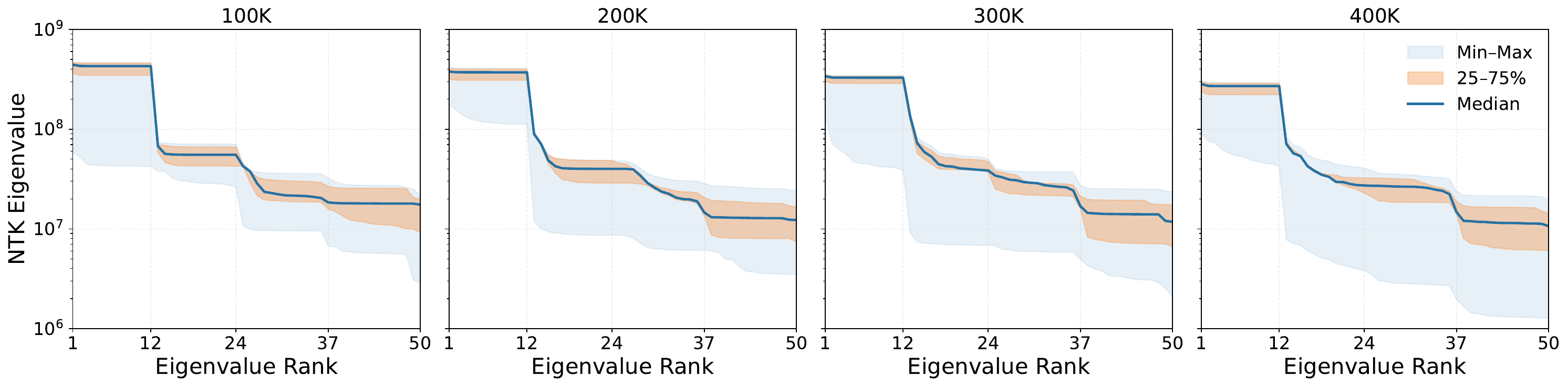}
    %\vspace{-0.5em}
    \includegraphics[width=1.0\linewidth]{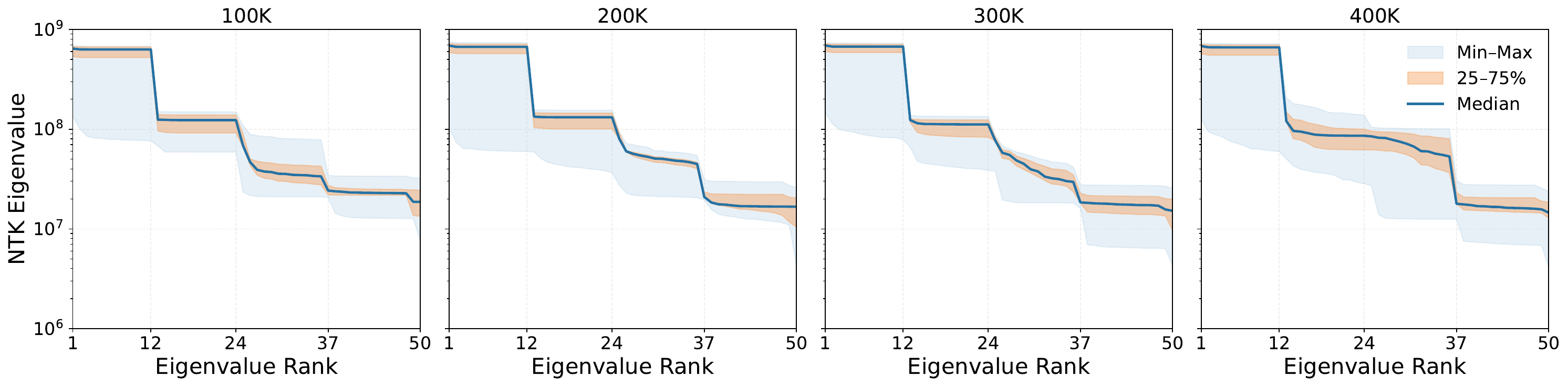}
    \vspace{-1.5em}
    \caption{Comparison of NTK eigenvalue spectra between the baseline (top) and our method (bottom).}
    \label{fig:ntk_eigenspectrum}
    \vspace{-1.0em}
\end{figure}
\section{Related Work}

\textbf{Diffusion Model Training Acceleration.}
Prior work improves diffusion training efficiency mainly through resource allocation across noise levels~\cite{choi2022perception,hang2023efficient,yu2023debias,yao2024fasterdit,xu2024towards,zheng2024beta,zheng2024non,wang2024closer,nichol2021improved,hang2024improved}, while other approaches modify the training dynamics or diffusion process~\cite{li2023not,wu2023fastdiffusionmodel}. Despite their effectiveness, these methods are largely heuristic and do not explain the intrinsic optimization difficulty of diffusion training.

\textbf{Representation Alignment in Diffusion Models.}
Recent works show that improving representation quality can accelerate diffusion training through representation alignment, pretraining, or latent modeling~\cite{yu2024representation,leng2025repa,yao2025reconstruction,wu2025representation,jiang2025no,zheng2025diffusion,shi2025latent}. However, these methods mainly rely on external priors, auxiliary objectives, or modified latent frameworks and do not explain the intrinsic role of learned representations in standard diffusion training. In contrast, we investigate representation degradation arising naturally within standard diffusion training.
%chu2025usp

\textbf{Theoretical Analyses of Diffusion and NTK Dynamics.}
Prior theory has studied diffusion models from variational, dynamical, and score-based perspectives~\cite{kingma2021variational,luo2022understanding,kingma2023understanding,karras2024analyzing,han2024neural}. In parallel, NTK, spectral bias, and collapse-related work~\cite{jacot2018neural,lee2019wide,rahaman2019spectral,jing2022understanding} provides tools for analyzing optimization geometry and representation collapse. These directions remain unconnected in explaining extreme noise representation degradation. We bridge this gap by connecting diffusion optimization, NTK dynamics, and empirical representation degradation.

\section{Discussion \& Conclusion}

We emphasize two aspects. (i) Our analysis suggests that overemphasizing uncertain targets can amplify gradient fluctuations from irreducible residuals, especially where weak NTK modes limit error correction. RPDM addresses this by reallocating optimization effort using a tractable effective amplitude score rather than estimating conditional target variance. (ii) While noise-level allocation has been explored, our core contribution is identifying representation degradation as an intrinsic bottleneck in diffusion training. We show that mitigating degradation improves representations and generation quality without external representation guidance. This perspective could extend to layer-wise allocation~\cite{jiang2025tracing} in future work (see \cref{sec:layerwise_allocation}).

% We emphasize two aspects. (i) Our analysis suggests that overemphasizing uncertain targets can amplify gradient fluctuations from irreducible conditional residuals, especially where weak NTK modes limit error correction. To address this mismatch, RPDM reallocates optimization effort using a tractable effective amplitude score rather than directly estimating conditional target variance. (ii) Although noise-level allocation has been explored, our core contribution is identifying representation degradation as an intrinsic bottleneck in diffusion training. We show that mitigating this degradation enables diffusion models to learn stronger representations and achieve better generation quality without relying on external representation guidance. This perspective could also extend to layerwise allocation~\cite{jiang2025tracing} in future work.

% \section{Limitations \& Conclusion} 

% \textbf{Limitations.}
% Our analysis relies on continuous-time optimization, NTK-regime, and sample-path surrogate assumptions, which may not fully capture finite-step stochastic and highly nonlinear diffusion training dynamics. In addition, ERD uses an effective-amplitude score as a tractable proxy for target recoverability rather than directly estimating exact posterior recoverability.

% \textbf{Conclusion.}
% We work identify representation degradation as an intrinsic bottleneck in diffusion training. By mitigating this degradation, we show that diffusion models can learn stronger representations and achieve better generation quality without relying on external representation guidance.  

\subsubsection*{Acknowledgment}
% \vspace{-0.5em}
This research was supported by \textsc{Nvidia} grants and utilized \textsc{Nvidia} GPUs and software, including \textsc{Nvidia} A100 GPUs on Brev. The authors gratefully acknowledge the \textsc{Nvidia} Academic Grant Program for providing computational resources that enabled this work.

% ============================================================
% References
% ============================================================

\bibliographystyle{plainnat}
\bibliography{main}

@inproceedings{sohl2015deep,
  title={Deep unsupervised learning using nonequilibrium thermodynamics},
  author={Sohl-Dickstein, Jascha and Weiss, Eric and Maheswaranathan, Niru and Ganguli, Surya},
  booktitle={International conference on machine learning},
  pages={2256--2265},
  year={2015},
  organization={PMLR}
}

@article{ho2020denoising,
  title={Denoising diffusion probabilistic models},
  author={Ho, Jonathan and Jain, Ajay and Abbeel, Pieter},
  journal={Advances in neural information processing systems},
  volume={33},
  pages={6840--6851},
  year={2020}
}

@article{song2020denoising,
  title={Denoising diffusion implicit models},
  author={Song, Jiaming and Meng, Chenlin and Ermon, Stefano},
  journal={arXiv preprint arXiv:2010.02502},
  year={2020}
}

@article{salimans2022progressive,
  title={Progressive distillation for fast sampling of diffusion models},
  author={Salimans, Tim and Ho, Jonathan},
  journal={arXiv preprint arXiv:2202.00512},
  year={2022}
}

@article{jacot2018neural,
  title={Neural tangent kernel: Convergence and generalization in neural networks},
  author={Jacot, Arthur and Gabriel, Franck and Hongler, Cl{\'e}ment},
  journal={Advances in neural information processing systems},
  volume={31},
  year={2018}
}

@article{kingma2021variational,
  title={Variational diffusion models},
  author={Kingma, Diederik and Salimans, Tim and Poole, Ben and Ho, Jonathan},
  journal={Advances in neural information processing systems},
  volume={34},
  pages={21696--21707},
  year={2021}
}

@article{li2022diffusion,
  title={Diffusion-lm improves controllable text generation},
  author={Li, Xiang and Thickstun, John and Gulrajani, Ishaan and Liang, Percy S and Hashimoto, Tatsunori B},
  journal={Advances in Neural Information Processing Systems},
  volume={35},
  pages={4328--4343},
  year={2022}
}

@article{ho2022video,
  title={Video diffusion models},
  author={Ho, Jonathan and Salimans, Tim and Gritsenko, Alexey and Chan, William and Norouzi, Mohammad and Fleet, David J},
  journal={Advances in Neural Information Processing Systems},
  volume={35},
  pages={8633--8646},
  year={2022}
}

@article{kong2020diffwave,
  title={Diffwave: A versatile diffusion model for audio synthesis},
  author={Kong, Zhifeng and Ping, Wei and Huang, Jiaji and Zhao, Kexin and Catanzaro, Bryan},
  journal={arXiv preprint arXiv:2009.09761},
  year={2020}
}

@inproceedings{zhou20213d,
  title={3d shape generation and completion through point-voxel diffusion},
  author={Zhou, Linqi and Du, Yilun and Wu, Jiajun},
  booktitle={Proceedings of the IEEE/CVF international conference on computer vision},
  pages={5826--5835},
  year={2021}
}

@article{krizhevsky2009learning,
  title={Learning multiple layers of features from tiny images},
  author={ Krizhevsky, A.  and  Hinton, G. },
  journal={Handbook of Systemic Autoimmune Diseases},
  volume={1},
  number={4},
  year={2009},
}

@inproceedings{liu2015deep,
  title={Deep learning face attributes in the wild},
  author={Liu, Ziwei and Luo, Ping and Wang, Xiaogang and Tang, Xiaoou},
  booktitle={Proceedings of the IEEE international conference on computer vision},
  pages={3730--3738},
  year={2015}
}

@inproceedings{Deng2009,
  author = {Jia Deng and Wei Dong and Richard Socher and Li-Jia Li and Kai Li and Li Fei-Fei},
  title = {ImageNet: A large-scale hierarchical image database},
  booktitle = {IEEE Conference on Computer Vision and Pattern Recognition},
  year = {2009}
}

@inproceedings{rombach2022high,
  title={High-resolution image synthesis with latent diffusion models},
  author={Rombach, Robin and Blattmann, Andreas and Lorenz, Dominik and Esser, Patrick and Ommer, Bj{\"o}rn},
  booktitle={Proceedings of the IEEE/CVF conference on computer vision and pattern recognition},
  pages={10684--10695},
  year={2022}
}

@article{karras2022elucidating,
  title={Elucidating the design space of diffusion-based generative models},
  author={Karras, Tero and Aittala, Miika and Aila, Timo and Laine, Samuli},
  journal={Advances in neural information processing systems},
  volume={35},
  pages={26565--26577},
  year={2022}
}

@article{ho2022classifier,
  title={Classifier-free diffusion guidance},
  author={Ho, Jonathan and Salimans, Tim},
  journal={arXiv preprint arXiv:2207.12598},
  year={2022}
}

@article{kynkaanniemi2024applying,
  title={Applying guidance in a limited interval improves sample and distribution quality in diffusion models},
  author={Kynk{\"a}{\"a}nniemi, Tuomas and Aittala, Miika and Karras, Tero and Laine, Samuli and Aila, Timo and Lehtinen, Jaakko},
  journal={arXiv preprint arXiv:2404.07724},
  year={2024}
}

@article{dhariwal2021diffusion,
  title={Diffusion models beat gans on image synthesis},
  author={Dhariwal, Prafulla and Nichol, Alexander},
  journal={Advances in neural information processing systems},
  volume={34},
  pages={8780--8794},
  year={2021}
}

@article{kingma2014adam,
  title={Adam: A method for stochastic optimization},
  author={Kingma, Diederik P},
  journal={arXiv preprint arXiv:1412.6980},
  year={2014}
}

@article{loshchilov2017decoupled,
  title={Decoupled weight decay regularization},
  author={Loshchilov, Ilya and Hutter, Frank},
  journal={arXiv preprint arXiv:1711.05101},
  year={2017}
}

@inproceedings{hang2023efficient,
  title={Efficient diffusion training via min-snr weighting strategy},
  author={Hang, Tiankai and Gu, Shuyang and Li, Chen and Bao, Jianmin and Chen, Dong and Hu, Han and Geng, Xin and Guo, Baining},
  booktitle={Proceedings of the IEEE/CVF International Conference on Computer Vision},
  pages={7441--7451},
  year={2023}
}

@inproceedings{ronneberger2015u,
  title={U-net: Convolutional networks for biomedical image segmentation},
  author={Ronneberger, Olaf and Fischer, Philipp and Brox, Thomas},
  booktitle={Medical image computing and computer-assisted intervention--MICCAI 2015: 18th international conference, Munich, Germany, October 5-9, 2015, proceedings, part III 18},
  pages={234--241},
  year={2015},
  organization={Springer}
}

@inproceedings{nichol2021improved,
  title={Improved denoising diffusion probabilistic models},
  author={Nichol, Alexander Quinn and Dhariwal, Prafulla},
  booktitle={International conference on machine learning},
  pages={8162--8171},
  year={2021},
  organization={PMLR}
}

@article{hang2024improved,
  title={Improved noise schedule for diffusion training},
  author={Hang, Tiankai and Gu, Shuyang},
  journal={arXiv preprint arXiv:2407.03297},
  year={2024}
}

@inproceedings{bao2023all,
  title={All are worth words: A vit backbone for diffusion models},
  author={Bao, Fan and Nie, Shen and Xue, Kaiwen and Cao, Yue and Li, Chongxuan and Su, Hang and Zhu, Jun},
  booktitle={Proceedings of the IEEE/CVF conference on computer vision and pattern recognition},
  pages={22669--22679},
  year={2023}
}

@inproceedings{peebles2023scalable,
  title={Scalable diffusion models with transformers},
  author={Peebles, William and Xie, Saining},
  booktitle={Proceedings of the IEEE/CVF International Conference on Computer Vision},
  pages={4195--4205},
  year={2023}
}

@inproceedings{ma2024sit,
  title={Sit: Exploring flow and diffusion-based generative models with scalable interpolant transformers},
  author={Ma, Nanye and Goldstein, Mark and Albergo, Michael S and Boffi, Nicholas M and Vanden-Eijnden, Eric and Xie, Saining},
  booktitle={European Conference on Computer Vision},
  pages={23--40},
  year={2024},
  organization={Springer}
}

@inproceedings{choi2022perception,
  title={Perception prioritized training of diffusion models},
  author={Choi, Jooyoung and Lee, Jungbeom and Shin, Chaehun and Kim, Sungwon and Kim, Hyunwoo and Yoon, Sungroh},
  booktitle={Proceedings of the IEEE/CVF Conference on Computer Vision and Pattern Recognition},
  pages={11472--11481},
  year={2022}
}

@article{yu2023debias,
  title={Debias the training of diffusion models},
  author={Yu, Hu and Shen, Li and Huang, Jie and Zhou, Man and Li, Hongsheng and Zhao, Feng},
  journal={arXiv preprint arXiv:2310.08442},
  year={2023}
}

@article{xu2024towards,
  title={Towards Faster Training of Diffusion Models: An Inspiration of A Consistency Phenomenon},
  author={Xu, Tianshuo and Mi, Peng and Wang, Ruilin and Chen, Yingcong},
  journal={arXiv preprint arXiv:2404.07946},
  year={2024}
}

@inproceedings{zheng2024beta,
  title={Beta-Tuned Timestep Diffusion Model},
  author={Zheng, Tianyi and Jiang, Peng-Tao and Wan, Ben and Zhang, Hao and Chen, Jinwei and Wang, Jia and Li, Bo},
  booktitle={European Conference on Computer Vision},
  year={2024}
}

@inproceedings{zheng2024non,
  title={Non-uniform Timestep Sampling: Towards Faster Diffusion Model Training},
  author={Zheng, Tianyi and Geng, Cong and Jiang, Peng-Tao and Wan, Ben and Zhang, Hao and Chen, Jinwei and Wang, Jia and Li, Bo},
  booktitle={ACM Multimedia 2024},
  year={2024}
}

@article{wang2024closer,
  title={A Closer Look at Time Steps is Worthy of Triple Speed-Up for Diffusion Model Training},
  author={Wang, Kai and Zhou, Yukun and Shi, Mingjia and Yuan, Zhihang and Shang, Yuzhang and Peng, Xiaojiang and Zhang, Hanwang and You, Yang},
  journal={arXiv preprint arXiv:2405.17403},
  year={2024}
}

@article{yao2024fasterdit,
  title={Fasterdit: Towards faster diffusion transformers training without architecture modification},
  author={Yao, Jingfeng and Cheng, Wang and Liu, Wenyu and Wang, Xinggang},
  journal={arXiv preprint arXiv:2410.10356},
  year={2024}
}

@misc{wu2023fastdiffusionmodel,
      title={Fast Diffusion Model}, 
      author={Zike Wu and Pan Zhou and Kenji Kawaguchi and Hanwang Zhang},
      year={2023},
      eprint={2306.06991},
      archivePrefix={arXiv},
      primaryClass={cs.CV},
      url={https://arxiv.org/abs/2306.06991}, 
}

@article{li2023not,
  title={Not All Steps are Equal: Efficient Generation with Progressive Diffusion Models},
  author={Li, Wenhao and Su, Xiu and You, Shan and Huang, Tao and Wang, Fei and Qian, Chen and Xu, Chang},
  journal={arXiv preprint arXiv:2312.13307},
  year={2023}
}

@article{song2020score,
  title={Score-based generative modeling through stochastic differential equations},
  author={Song, Yang and Sohl-Dickstein, Jascha and Kingma, Diederik P and Kumar, Abhishek and Ermon, Stefano and Poole, Ben},
  journal={arXiv preprint arXiv:2011.13456},
  year={2020}
}

@article{ho2022cascaded,
  title={Cascaded diffusion models for high fidelity image generation},
  author={Ho, Jonathan and Saharia, Chitwan and Chan, William and Fleet, David J and Norouzi, Mohammad and Salimans, Tim},
  journal={Journal of Machine Learning Research},
  volume={23},
  number={47},
  pages={1--33},
  year={2022}
}

@article{luo2022understanding,
  title={Understanding diffusion models: A unified perspective},
  author={Luo, Calvin},
  journal={arXiv preprint arXiv:2208.11970},
  year={2022}
}

@article{yu2024representation,
  title={Representation alignment for generation: Training diffusion transformers is easier than you think},
  author={Yu, Sihyun and Kwak, Sangkyung and Jang, Huiwon and Jeong, Jongheon and Huang, Jonathan and Shin, Jinwoo and Xie, Saining},
  journal={arXiv preprint arXiv:2410.06940},
  year={2024}
}

@article{leng2025repa,
  title={Repa-e: Unlocking vae for end-to-end tuning with latent diffusion transformers},
  author={Leng, Xingjian and Singh, Jaskirat and Hou, Yunzhong and Xing, Zhenchang and Xie, Saining and Zheng, Liang},
  journal={arXiv preprint arXiv:2504.10483},
  year={2025}
}

@inproceedings{yao2025reconstruction,
  title={Reconstruction vs. generation: Taming optimization dilemma in latent diffusion models},
  author={Yao, Jingfeng and Yang, Bin and Wang, Xinggang},
  booktitle={Proceedings of the Computer Vision and Pattern Recognition Conference},
  pages={15703--15712},
  year={2025}
}

@article{wu2025representation,
  title={Representation Entanglement for Generation: Training Diffusion Transformers Is Much Easier Than You Think},
  author={Wu, Ge and Zhang, Shen and Shi, Ruijing and Gao, Shanghua and Chen, Zhenyuan and Wang, Lei and Chen, Zhaowei and Gao, Hongcheng and Tang, Yao and Yang, Jian and others},
  journal={arXiv preprint arXiv:2507.01467},
  year={2025}
}

@article{jiang2025no,
  title={No Other Representation Component Is Needed: Diffusion Transformers Can Provide Representation Guidance by Themselves},
  author={Jiang, Dengyang and Wang, Mengmeng and Li, Liuzhuozheng and Zhang, Lei and Wang, Haoyu and Wei, Wei and Dai, Guang and Zhang, Yanning and Wang, Jingdong},
  journal={arXiv preprint arXiv:2505.02831},
  year={2025}
}

@article{zheng2025diffusion,
  title={Diffusion transformers with representation autoencoders},
  author={Zheng, Boyang and Ma, Nanye and Tong, Shengbang and Xie, Saining},
  journal={arXiv preprint arXiv:2510.11690},
  year={2025}
}

@article{shi2025latent,
  title={Latent diffusion model without variational autoencoder},
  author={Shi, Minglei and Wang, Haolin and Zheng, Wenzhao and Yuan, Ziyang and Wu, Xiaoshi and Wang, Xintao and Wan, Pengfei and Zhou, Jie and Lu, Jiwen},
  journal={arXiv preprint arXiv:2510.15301},
  year={2025}
}

@article{ahmadian2023intriguing,
  title={Intriguing properties of quantization at scale},
  author={Ahmadian, Arash and Dash, Saurabh and Chen, Hongyu and Venkitesh, Bharat and Gou, Zhen Stephen and Blunsom, Phil and {\"U}st{\"u}n, Ahmet and Hooker, Sara},
  journal={Advances in Neural Information Processing Systems},
  volume={36},
  pages={34278--34294},
  year={2023}
}

@inproceedings{hoogeboom2023simple,
  title={simple diffusion: End-to-end diffusion for high resolution images},
  author={Hoogeboom, Emiel and Heek, Jonathan and Salimans, Tim},
  booktitle={International Conference on Machine Learning},
  pages={13213--13232},
  year={2023},
  organization={PMLR}
}

@inproceedings{hatamizadeh2024diffit,
  title={Diffit: Diffusion vision transformers for image generation},
  author={Hatamizadeh, Ali and Song, Jiaming and Liu, Guilin and Kautz, Jan and Vahdat, Arash},
  booktitle={European Conference on Computer Vision},
  pages={37--55},
  year={2024},
  organization={Springer}
}

@article{gao2023mdtv2,
  title={Mdtv2: Masked diffusion transformer is a strong image synthesizer},
  author={Gao, Shanghua and Zhou, Pan and Cheng, Ming-Ming and Yan, Shuicheng},
  journal={arXiv preprint arXiv:2303.14389},
  year={2023}
}

@article{zheng2023fast,
  title={Fast training of diffusion models with masked transformers},
  author={Zheng, Hongkai and Nie, Weili and Vahdat, Arash and Anandkumar, Anima},
  journal={arXiv preprint arXiv:2306.09305},
  year={2023}
}

@inproceedings{zhu2024sd,
  title={Sd-dit: Unleashing the power of self-supervised discrimination in diffusion transformer},
  author={Zhu, Rui and Pan, Yingwei and Li, Yehao and Yao, Ting and Sun, Zhenglong and Mei, Tao and Chen, Chang Wen},
  booktitle={Proceedings of the IEEE/CVF Conference on Computer Vision and Pattern Recognition},
  pages={8435--8445},
  year={2024}
}

@article{gao2019representation,
  title={Representation degeneration problem in training natural language generation models},
  author={Gao, Jun and He, Di and Tan, Xu and Qin, Tao and Wang, Liwei and Liu, Tie-Yan},
  journal={arXiv preprint arXiv:1907.12009},
  year={2019}
}

@article{kingma2023understanding,
  title={Understanding diffusion objectives as the elbo with simple data augmentation},
  author={Kingma, Diederik and Gao, Ruiqi},
  journal={Advances in Neural Information Processing Systems},
  volume={36},
  pages={65484--65516},
  year={2023}
}

@article{han2024neural,
  title={Neural network-based score estimation in diffusion models: Optimization and generalization},
  author={Han, Yinbin and Razaviyayn, Meisam and Xu, Renyuan},
  journal={arXiv preprint arXiv:2401.15604},
  year={2024}
}

@inproceedings{rahaman2019spectral,
  title={On the spectral bias of neural networks},
  author={Rahaman, Nasim and Baratin, Aristide and Arpit, Devansh and Draxler, Felix and Lin, Min and Hamprecht, Fred and Bengio, Yoshua and Courville, Aaron},
  booktitle={International conference on machine learning},
  pages={5301--5310},
  year={2019},
  organization={PMLR}
}

@article{lee2019wide,
  title={Wide neural networks of any depth evolve as linear models under gradient descent},
  author={Lee, Jaehoon and Xiao, Lechao and Schoenholz, Samuel and Bahri, Yasaman and Novak, Roman and Sohl-Dickstein, Jascha and Pennington, Jeffrey},
  journal={Advances in neural information processing systems},
  volume={32},
  year={2019}
}

@inproceedings{jing2022understanding,
  title={Understanding Dimensional Collapse in Contrastive Self-supervised Learning},
  author={Li Jing and Pascal Vincent and Yann LeCun and Yuandong Tian},
  booktitle={International Conference on Learning Representations},
  year={2022},
  url={https://openreview.net/forum?id=YevsQ05DEN7}
}

@inproceedings{karras2024analyzing,
  title={Analyzing and improving the training dynamics of diffusion models},
  author={Karras, Tero and Aittala, Miika and Lehtinen, Jaakko and Hellsten, Janne and Aila, Timo and Laine, Samuli},
  booktitle={Proceedings of the IEEE/CVF Conference on Computer Vision and Pattern Recognition},
  pages={24174--24184},
  year={2024}
}

@article{heusel2017gans,
  title={Gans trained by a two time-scale update rule converge to a local nash equilibrium},
  author={Heusel, Martin and Ramsauer, Hubert and Unterthiner, Thomas and Nessler, Bernhard and Hochreiter, Sepp},
  journal={Advances in neural information processing systems},
  volume={30},
  year={2017}
}

@article{nash2021generating,
  title={Generating images with sparse representations},
  author={Nash, Charlie and Menick, Jacob and Dieleman, Sander and Battaglia, Peter W},
  journal={arXiv preprint arXiv:2103.03841},
  year={2021}
}

@article{salimans2016improved,
  title={Improved techniques for training gans},
  author={Salimans, Tim and Goodfellow, Ian and Zaremba, Wojciech and Cheung, Vicki and Radford, Alec and Chen, Xi},
  journal={Advances in neural information processing systems},
  volume={29},
  year={2016}
}

@article{kynkaanniemi2019improved,
  title={Improved precision and recall metric for assessing generative models},
  author={Kynk{\"a}{\"a}nniemi, Tuomas and Karras, Tero and Laine, Samuli and Lehtinen, Jaakko and Aila, Timo},
  journal={Advances in neural information processing systems},
  volume={32},
  year={2019}
}

@inproceedings{szegedy2016rethinking,
  title={Rethinking the inception architecture for computer vision},
  author={Szegedy, Christian and Vanhoucke, Vincent and Ioffe, Sergey and Shlens, Jon and Wojna, Zbigniew},
  booktitle={Proceedings of the IEEE conference on computer vision and pattern recognition},
  pages={2818--2826},
  year={2016}
}

@article{lipman2022flow,
  title={Flow matching for generative modeling},
  author={Lipman, Yaron and Chen, Ricky TQ and Ben-Hamu, Heli and Nickel, Maximilian and Le, Matt},
  journal={arXiv preprint arXiv:2210.02747},
  year={2022}
}

@article{dosovitskiy2020image,
  title={An image is worth 16x16 words: Transformers for image recognition at scale},
  author={Dosovitskiy, Alexey},
  journal={arXiv preprint arXiv:2010.11929},
  year={2020}
}

@inproceedings{garrido2023rankme,
  title={Rankme: Assessing the downstream performance of pretrained self-supervised representations by their rank},
  author={Garrido, Quentin and Balestriero, Randall and Najman, Laurent and Lecun, Yann},
  booktitle={International conference on machine learning},
  pages={10929--10974},
  year={2023},
  organization={PMLR}
}

@article{lanczos1950iteration,
  title={An iteration method for the solution of the eigenvalue problem of linear differential and integral operators},
  author={Lanczos, Cornelius},
  journal={Journal of research of the National Bureau of Standards},
  volume={45},
  number={4},
  pages={255--282},
  year={1950}
}

@inproceedings{jiang2025tracing,
  title={Tracing representation progression: Analyzing and enhancing layer-wise similarity},
  author={Jiang, Jiachen and Zhou, Jinxin and Zhu, Zhihui},
  booktitle={International Conference on Learning Representations},
  volume={2025},
  pages={1118--1143},
  year={2025}
}

@article{oquab2023dinov2,
  title={Dinov2: Learning robust visual features without supervision},
  author={Oquab, Maxime and Darcet, Timoth{\'e}e and Moutakanni, Th{\'e}o and Vo, Huy and Szafraniec, Marc and Khalidov, Vasil and Fernandez, Pierre and Haziza, Daniel and Massa, Francisco and El-Nouby, Alaaeldin and others},
  journal={arXiv preprint arXiv:2304.07193},
  year={2023}
}

@inproceedings{roy2007effective,
  title={The effective rank: A measure of effective dimensionality},
  author={Roy, Olivier and Vetterli, Martin},
  booktitle={2007 15th European signal processing conference},
  pages={606--610},
  year={2007},
  organization={IEEE}
}

@article{mandt2017stochastic,
  title={Stochastic gradient descent as approximate bayesian inference},
  author={Mandt, Stephan and Hoffman, Matthew D and Blei, David M},
  journal={Journal of Machine Learning Research},
  volume={18},
  number={134},
  pages={1--35},
  year={2017}
}

% ============================================================
% Appendix
% ============================================================

\clearpage
\appendix
% \etocdepthtag.toc{mtappendix}
% \etocsettagdepth{mtchapter}{none}
% \etocsettagdepth{mtappendix}{subsection}
% {\small \tableofcontents}

\section*{Notation}
\begingroup
\small
\setlength{\tabcolsep}{2pt}
\renewcommand{\arraystretch}{1.12}
\noindent
\begin{tabular}{@{}p{0.14\linewidth}@{\hspace{0.8em}}p{0.32\linewidth}@{\hspace{3.0em}}p{0.14\linewidth}@{\hspace{0.8em}}p{0.32\linewidth}@{}}
\toprule
\textbf{Symbol} & \textbf{Description} & \textbf{Symbol} & \textbf{Description} \\
\midrule
$x_0,\epsilon$ & Clean data and Gaussian noise
& $\Theta_{\theta,m}$ & Joint neural tangent kernel \\

$x_\lambda$ & Noisy diffusion sample
& $J_m(x,\lambda)$ & Frozen prediction Jacobian \\

$\lambda,t,\tau$ & Log-SNR, diffusion and training times
& $K_\lambda,G_\lambda$ & NTK and feature Gram matrices \\

$\dot{\lambda}_t$ & Log-SNR change rate
& $u_a^\lambda$ & NTK eigenvector \\

$\alpha_\lambda,\sigma_\lambda$ & Signal and noise scales
& $\kappa_a^\lambda,B_a^\lambda$ & NTK eigenvalue and mode operator \\

$y_\lambda$ & Prediction target
& $h_{\theta,m}$ & Hidden representation \\

$q_\lambda$ & Marginal distribution of $x_\lambda$
& $\theta,\theta_0$ & Network and initial parameters \\

$d,m,P$ & Data dimension, width and \#params
& $A_m(x,\lambda)$ & Feature Jacobian \\

$\rho(\lambda),w(\lambda)$ & Sampling density and loss weight
& $H_{x,\lambda}$ & Sampled Hessian \\

$M(\lambda)$ & Optimization allocation
& $H_\lambda,H$ & Local and population Hessians \\

$f_{\theta,m},f_\lambda^\star$ & Model and Bayes predictors
& $\Gamma_{\xi,m}(\lambda)$ & Irreducible gradient covariance density \\

$r_{\tau,m},\xi_\lambda$ & Estimation error, irreducible residual
& $V_{\xi,m}$ & Integrated irreducible gradient variance \\

$\Delta$ & Sampling interval
& $\mu,L$ & Hessian bounds \\

$T_\Delta$ & Stationary parameter covariance
& $S_{\xi,m}$ & Leading covariance coefficient \\

$q_\Delta$ & Geometric contraction factor
& $M_0,M_1,M_\star$ & Baseline, RPDM, optimal allocations \\

$b$ & Linear probe coefficients
& $Q$ & Linear probe risk matrix \\

$\mathcal R_\lambda(b;\theta)$ & Linear probe squared-error risk
& $q_x$ & Per-dimension data second moment \\

$\ell_\lambda^\star$ & Irreducible Bayes-risk density
& $\bar{\ell}_{\tau,m}$ & Local excess-risk density \\

$\bar{\mathcal L}_{\tau,m}$ & Global excess risk
& $\theta_m^\star$ & Bayes-optimal tangent parameter \\

$\vartheta_m(\tau)$ & Stochastic parameter trajectory
& $\widetilde h_{\theta,m}$ & Linearized hidden representation \\

$W_{\theta,m}$ & Trainable readout matrix
& $S_{\xi,m}^{(i)}$ & Allocation covariance coefficient \\

$c_{x,\lambda}$ & Data coefficient in prediction target
& $c_{\epsilon,\lambda}$ & Noise coefficient in prediction target \\

\bottomrule
\end{tabular}
\endgroup
\clearpage
\section{Neural Tangent Kernel in Diffusion Models}
\label{apdx:sec:ntk_dynamics}

\begin{proof}[Proof of \Cref{lem:optimization_equivalence}]
	By definition of the pushforward density $\rho$, changing the sampling variable from $t$ to $\lambda=\lambda_t$ in \Eqref{eq:time_domain_elbo_objective} yields
	\begin{align}
		\mathcal L_{\mathrm{ELBO},m}(\theta)
		&=\frac12\mathbb E_t\!\left[w(\lambda_t)|\dot{\lambda}_t|\,\mathbb E_{x_0,\epsilon}\!\left[\|f_{\theta,m}(x_{\lambda_t},\lambda_t)-y_{\lambda_t}\|_2^2\right]\right]\nonumber\\
        &=\frac12\int_{\lambda_{\min}}^{\lambda_{\max}}\rho(\lambda)w(\lambda)|\dot{\lambda}_{t(\lambda)}|\mathbb E_{x_0,\epsilon}\!\left[\|f_{\theta,m}(x_\lambda,\lambda)-y_\lambda\|_2^2\right]\mathrm d\lambda\nonumber\\
		&=\frac12\int_{\lambda_{\min}}^{\lambda_{\max}}M(\lambda)\mathbb E_{x_0,\epsilon}\!\left[\|f_{\theta,m}(x_\lambda,\lambda)-y_\lambda\|_2^2\right]\mathrm d\lambda
		\label{eq:elbo_objective}\\
		&=\mathbb E_{\lambda\sim\rho}\!\left[\frac{M(\lambda)}{2\rho(\lambda)}\mathbb E_{x_0,\epsilon}\!\left[\|f_{\theta,m}(x_\lambda,\lambda)-y_\lambda\|_2^2\right]\right],
		\label{eq:sampling_representation_elbo}
	\end{align}
	\Eqref{eq:elbo_objective} follows from the definition of $M(\lambda)$, and \Eqref{eq:sampling_representation_elbo} rewrites the integral as an expectation under $\rho$.
\end{proof}

\begin{proof}[Proof of \Cref{lem:elbo_bayes_decomposition}]
	For fixed $\lambda$, the prediction residual decomposes as
	\begin{align}
		f_{\tau,m}(x_\lambda,\lambda)-y_\lambda&=\underbrace{f_{\tau,m}(x_\lambda,\lambda)-f_\lambda^\star(x_\lambda)}_{r_{\tau,m}(x_\lambda,\lambda)}+\underbrace{f_\lambda^\star(x_\lambda)-y_\lambda}_{\xi_\lambda}\nonumber\\
		&=r_{\tau,m}(x_\lambda,\lambda)+\xi_\lambda.
		\label{eq:bayes_residual_split_quantified}
	\end{align}
	By definition of the Bayes predictor,
	\begin{align}
		\mathbb E[\xi_\lambda\mid x_\lambda]&=f_\lambda^\star(x_\lambda)-\mathbb E[y_\lambda\mid x_\lambda]\nonumber\\
		&=0.
		\label{eq:conditional_bayes_component_zero_quantified}
	\end{align}
	Since $r_{\tau,m}(x_\lambda,\lambda)$ is measurable with respect to $x_\lambda$, the tower property gives
	\begin{align}
		\mathbb E\langle r_{\tau,m}(x_\lambda,\lambda),\xi_\lambda\rangle&=\mathbb E\!\left[\left\langle r_{\tau,m}(x_\lambda,\lambda),\underbrace{\mathbb E[\xi_\lambda\mid x_\lambda]}_{=0}\right\rangle\right]\nonumber\\
		&=0.
		\label{eq:cross_term_zero_bayes_quantified}
	\end{align}
	Consequently,
	\begin{align}
		&\frac12M(\lambda)\mathbb E\!\left[\|f_{\tau,m}(x_\lambda,\lambda)-y_\lambda\|_2^2-\|f_\lambda^\star(x_\lambda)-y_\lambda\|_2^2\right]\nonumber\\
		&\quad=\frac12M(\lambda)\mathbb E\!\left[\|r_{\tau,m}(x_\lambda,\lambda)+\xi_\lambda\|_2^2-\|\xi_\lambda\|_2^2\right]\nonumber\\
		&\quad=\frac12M(\lambda)\mathbb E\|r_{\tau,m}(x_\lambda,\lambda)\|_2^2+M(\lambda)\underbrace{\mathbb E\langle r_{\tau,m}(x_\lambda,\lambda),\xi_\lambda\rangle}_{=0}\nonumber\\
		&\quad=\bar{\ell}_{\tau,m}(\lambda).
		\label{eq:local_density_expansion_quantified}
	\end{align}
	Integrating over $\lambda$ and applying \cref{lem:optimization_equivalence} gives the global identity. The lower bounds follow from $\bar{\ell}_{\tau,m}(\lambda)\ge0$.
\end{proof}

\begin{proof}[Proof of \Cref{prop:joint_ntk_dynamics}]
	Under \cref{assum:unified_assumption}, differentiation commutes with integration and expectation, giving
	\begin{align}
		\nabla_\theta\mathcal L_{\mathrm{ELBO},m}(\theta)&=\int_{\lambda_{\min}}^{\lambda_{\max}}M(\lambda')\mathbb E_{x_0,\epsilon}\!\left[\big(\nabla_\theta f_{\theta,m}(x_{\lambda'},\lambda')\big)^\top\big(f_{\theta,m}(x_{\lambda'},\lambda')-y_{\lambda'}\big)\right]\mathrm d\lambda'.
		\label{eq:loss_gradient_inside_proof}
	\end{align}
	\begin{align}
		\frac{\mathrm d\theta_m(\tau)}{\mathrm d\tau}&=-\nabla_\theta\mathcal L_{\mathrm{ELBO},m}(\theta_m(\tau))\nonumber\\
		&=-\int_{\lambda_{\min}}^{\lambda_{\max}}M(\lambda')\mathbb E_{x_0,\epsilon}\!\left[\big(\nabla_\theta f_{\tau,m}(x_{\lambda'},\lambda')\big)^\top\big(f_{\tau,m}(x_{\lambda'},\lambda')-y_{\lambda'}\big)\right]\mathrm d\lambda'.
		\label{eq:parameter_flow_inside_proof}
	\end{align}
	For fixed $(x,\lambda)$, the chain rule and joint NTK definition yield
	\begin{align}
		&\frac{\partial}{\partial\tau}f_{\tau,m}(x,\lambda)\nonumber\\
		&\quad=\nabla_\theta f_{\tau,m}(x,\lambda)\frac{\mathrm d\theta_m(\tau)}{\mathrm d\tau}\nonumber\\
		&\quad=-\int_{\lambda_{\min}}^{\lambda_{\max}}M(\lambda')\mathbb E_{x_0,\epsilon}\!\left[\nabla_\theta f_{\tau,m}(x,\lambda)\big(\nabla_\theta f_{\tau,m}(x_{\lambda'},\lambda')\big)^\top\big(f_{\tau,m}(x_{\lambda'},\lambda')-y_{\lambda'}\big)\right]\mathrm d\lambda'\nonumber\\
		&\quad=-\int_{\lambda_{\min}}^{\lambda_{\max}}M(\lambda')\mathbb E_{x_0,\epsilon}\!\left[\Theta_{\tau,m}\big((x,\lambda),(x_{\lambda'},\lambda')\big)\big(f_{\tau,m}(x_{\lambda'},\lambda')-y_{\lambda'}\big)\right]\mathrm d\lambda'.
		\label{eq:function_flow_inside_proof}
	\end{align}
	Since $y$ is independent of $\tau$, $\partial_\tau(f_{\tau,m}(x,\lambda)-y)=\partial_\tau f_{\tau,m}(x,\lambda)$.

	Applying \Eqref{eq:bayes_residual_split_quantified} and conditioning on $(x',\lambda')$, the tower property and \Eqref{eq:conditional_bayes_component_zero_quantified} give
	\begin{align}
		&\mathbb E_{x_0,\epsilon}\!\left[\Theta_{\tau,m}\big((x,\lambda),(x_{\lambda'},\lambda')\big)\big(f_{\tau,m}(x_{\lambda'},\lambda')-y_{\lambda'}\big)\right]\nonumber\\
		&\quad=\mathbb E_{x'\sim q_{\lambda'}}\!\left[\Theta_{\tau,m}\big((x,\lambda),(x',\lambda')\big)\big(r_{\tau,m}(x',\lambda')+\underbrace{\mathbb E[\xi_{\lambda'}\mid x',\lambda']}_{=0}\big)\right]\nonumber\\
		&\quad=\mathbb E_{x'\sim q_{\lambda'}}\!\left[\Theta_{\tau,m}\big((x,\lambda),(x',\lambda')\big)r_{\tau,m}(x',\lambda')\right].
		\label{eq:conditional_residual_kernel_expectation}
	\end{align}
	Since $f_\lambda^\star(x)$ is independent of $\tau$, $\partial_\tau r_{\tau,m}(x,\lambda)=\partial_\tau f_{\tau,m}(x,\lambda)$. Substituting into \Eqref{eq:function_flow_inside_proof} gives \Eqref{eq:joint_residual_ode}. The unweighted case follows by setting $M=\rho$.
\end{proof}

\section{Evidence Lower Bound Analysis}
\label{apdx:sec:elbo_dynamics}

For a sample $(x_0,\epsilon,\lambda)$ with $\lambda\sim\rho$, define
\begin{align}
	\widehat{\ell}_m(\theta;x_0,\epsilon,\lambda)&\triangleq\frac{M(\lambda)}{2\rho(\lambda)}\|f_{\theta,m}(x_\lambda,\lambda)-y_\lambda\|_2^2.
	\label{eq:stochastic_local_elbo_loss_quantified}
\end{align}
By \cref{lem:optimization_equivalence}, its expectation equals $\mathcal L_{\mathrm{ELBO},m}(\theta)$.
\begin{align}
	\nabla_\theta\widehat{\ell}_m(\theta;x_0,\epsilon,\lambda)&=\underbrace{\frac{M(\lambda)}{\rho(\lambda)}\big(\nabla_\theta f_{\theta,m}(x_\lambda,\lambda)\big)^\top r_{\theta,m}(x_\lambda,\lambda)}_{\text{estimation-error component}}+\underbrace{\frac{M(\lambda)}{\rho(\lambda)}\big(\nabla_\theta f_{\theta,m}(x_\lambda,\lambda)\big)^\top\xi_\lambda}_{\text{irreducible component}},
	\label{eq:stochastic_gradient_bayes_decomposition_quantified}
\end{align}
\begin{align}
	\mathbb E[\nabla_\theta\widehat{\ell}_m(\theta;x_0,\epsilon,\lambda)]&=\nabla_\theta\mathcal L_{\mathrm{ELBO},m}(\theta).
	\label{eq:sampled_population_gradient_consistency}
\end{align}

\begin{proof}[Proof of \Cref{lem:gradient_forcing}]
	Differentiating \Eqref{eq:stochastic_local_elbo_loss_quantified} and applying \Eqref{eq:bayes_residual_split_quantified} gives
	\begin{align}
		\nabla_\theta\widehat{\ell}_m(\theta;x_0,\epsilon,\lambda)&=\frac{M(\lambda)}{\rho(\lambda)}\big(\nabla_\theta f_{\theta,m}(x_\lambda,\lambda)\big)^\top\big(f_{\theta,m}(x_\lambda,\lambda)-y_\lambda\big)\nonumber\\
		&=\frac{M(\lambda)}{\rho(\lambda)}\big(\nabla_\theta f_{\theta,m}(x_\lambda,\lambda)\big)^\top\big(r_{\theta,m}(x_\lambda,\lambda)+\xi_\lambda\big)\nonumber\\
		&=\frac{M(\lambda)}{\rho(\lambda)}\big(\nabla_\theta f_{\theta,m}(x_\lambda,\lambda)\big)^\top r_{\theta,m}(x_\lambda,\lambda)+\frac{M(\lambda)}{\rho(\lambda)}\big(\nabla_\theta f_{\theta,m}(x_\lambda,\lambda)\big)^\top\xi_\lambda.
		\label{eq:sampled_gradient_decomposition_proof}
	\end{align}
	By \Eqref{eq:conditional_bayes_component_zero_quantified},
	\begin{align}
		\mathbb E\!\left[\left.\frac{M(\lambda)}{\rho(\lambda)}\big(\nabla_\theta f_{\theta,m}(x_\lambda,\lambda)\big)^\top\xi_\lambda\right|x_\lambda,\lambda\right]&=\frac{M(\lambda)}{\rho(\lambda)}\big(\nabla_\theta f_{\theta,m}(x_\lambda,\lambda)\big)^\top\underbrace{\mathbb E[\xi_\lambda\mid x_\lambda,\lambda]}_{=0}\nonumber\\
		&\quad=0.
		\label{eq:irreducible_gradient_conditional_mean}
	\end{align}
	Since the estimation-error component is measurable with respect to $(x_\lambda,\lambda)$, the tower property gives
	\begin{align}
		&\mathbb E\!\left[\left(\frac{M(\lambda)}{\rho(\lambda)}\right)^2\big(\nabla_\theta f_{\theta,m}(x_\lambda,\lambda)\big)^\top r_{\theta,m}(x_\lambda,\lambda)\xi_\lambda^\top\big(\nabla_\theta f_{\theta,m}(x_\lambda,\lambda)\big)\right]\nonumber\\
		&\quad=\mathbb E\!\left[\left(\frac{M(\lambda)}{\rho(\lambda)}\right)^2\big(\nabla_\theta f_{\theta,m}(x_\lambda,\lambda)\big)^\top r_{\theta,m}(x_\lambda,\lambda)\underbrace{\mathbb E[\xi_\lambda^\top\mid x_\lambda,\lambda]}_{=0}\big(\nabla_\theta f_{\theta,m}(x_\lambda,\lambda)\big)\right]\nonumber\\
		&\quad=0.
		\label{eq:gradient_component_cross_moment}
	\end{align}
	The covariance decomposition follows from the vanishing mean and cross-moment identities.

	Since $\mathbb E[\xi_\lambda\xi_\lambda^\top\mid x_\lambda,\lambda]=\operatorname{Cov}(y_\lambda\mid x_\lambda,\lambda)$,
	\begin{align}
		&\mathbb E\!\left[\left.\left(\frac{M(\lambda)}{\rho(\lambda)}\right)^2\big(\nabla_\theta f_{\theta,m}(x_\lambda,\lambda)\big)^\top\xi_\lambda\xi_\lambda^\top\big(\nabla_\theta f_{\theta,m}(x_\lambda,\lambda)\big)\right|x_\lambda,\lambda\right]\nonumber\\
		&\quad=\left(\frac{M(\lambda)}{\rho(\lambda)}\right)^2\big(\nabla_\theta f_{\theta,m}(x_\lambda,\lambda)\big)^\top\operatorname{Cov}(y_\lambda\mid x_\lambda,\lambda)\big(\nabla_\theta f_{\theta,m}(x_\lambda,\lambda)\big).
		\label{eq:conditional_residual_gradient_second_moment}
	\end{align}
	Taking expectations over $(x_\lambda,\lambda)$ gives
	\begin{align}
		&\operatorname{Cov}\!\left(\frac{M(\lambda)}{\rho(\lambda)}\big(\nabla_\theta f_{\theta,m}(x_\lambda,\lambda)\big)^\top\xi_\lambda\right)\nonumber\\
		&\quad=\int_{\lambda_{\min}}^{\lambda_{\max}}\frac{M(\lambda)^2}{\rho(\lambda)}\mathbb E_{x\sim q_\lambda}\!\left[\big(\nabla_\theta f_{\theta,m}(x,\lambda)\big)^\top\operatorname{Cov}(y_\lambda\mid x_\lambda=x)\big(\nabla_\theta f_{\theta,m}(x,\lambda)\big)\right]\mathrm d\lambda.
		\label{eq:irreducible_gradient_covariance}
	\end{align}

	By conditional centering and the sampling representation,
	\begin{align}
		\mathbb E[\nabla_\theta\widehat{\ell}_m]&=\mathbb E\!\left[\frac{M(\lambda)}{\rho(\lambda)}\big(\nabla_\theta f_{\theta,m}(x_\lambda,\lambda)\big)^\top r_{\theta,m}(x_\lambda,\lambda)\right]\nonumber\\
		&=\int_{\lambda_{\min}}^{\lambda_{\max}}M(\lambda)\mathbb E_{x\sim q_\lambda}\!\left[\big(\nabla_\theta f_{\theta,m}(x,\lambda)\big)^\top r_{\theta,m}(x,\lambda)\right]\mathrm d\lambda\nonumber\\
		&=\nabla_\theta\mathcal L_{\mathrm{ELBO},m}(\theta).
		\label{eq:sampled_population_gradient_consistency_proof}
	\end{align}
	The last equality follows from \Eqref{eq:loss_gradient_inside_proof} and conditional centering.
\end{proof}

\begin{proof}[Proof of \Cref{prop:sampled_continuous_time_dynamics}]
	Since the Bayes predictor is independent of training time, the chain rule and \Eqref{eq:training_ode_bayes_forcing_quantified} yield
	\begin{align}
		&\frac{\partial}{\partial\tau}r_{\tau,m}(x,\lambda)=\frac{\partial}{\partial\tau}f_{\tau,m}(x,\lambda)=\nabla_\theta f_{\tau,m}(x,\lambda)\frac{\mathrm d\vartheta_m(\tau)}{\mathrm d\tau}\nonumber\\
		&=-\frac{M(\Lambda_j)}{\rho(\Lambda_j)}\nabla_\theta f_{\tau,m}(x,\lambda)\big(\nabla_\theta f_{\tau,m}(X_j,\Lambda_j)\big)^\top\big(r_{\tau,m}(X_j,\Lambda_j)+\xi_j\big)\nonumber\\
		&=-\frac{M(\Lambda_j)}{\rho(\Lambda_j)}\Theta_{\tau,m}\big((x,\lambda),(X_j,\Lambda_j)\big)r_{\tau,m}(X_j,\Lambda_j)-\frac{M(\Lambda_j)}{\rho(\Lambda_j)}\Theta_{\tau,m}\big((x,\lambda),(X_j,\Lambda_j)\big)\xi_j.
		\label{eq:sampled_function_dynamics_proof}
	\end{align}

	At a sampling time $j\Delta$, the current parameters depend only on the preceding training history and are independent of the newly drawn sample. By \Eqref{eq:conditional_bayes_component_zero_quantified},
	\begin{align}
		&\mathbb E\!\left[\left.\frac{\partial}{\partial\tau}r_{\tau,m}(x,\lambda)\right|_{\tau=j\Delta}\Bigm|\vartheta_m(j\Delta)\right]\nonumber\\
		&\quad=-\mathbb E\!\left[\left.\frac{M(\Lambda_j)}{\rho(\Lambda_j)}\Theta_{\vartheta_m(j\Delta),m}\big((x,\lambda),(X_j,\Lambda_j)\big)\big(r_{\vartheta_m(j\Delta),m}(X_j,\Lambda_j)+\xi_j\big)\right|\vartheta_m(j\Delta)\right]\nonumber\\
		&\quad=-\int_{\lambda_{\min}}^{\lambda_{\max}}M(\lambda')\mathbb E_{x'\sim q_{\lambda'}}\!\left[\Theta_{\vartheta_m(j\Delta),m}\big((x,\lambda),(x',\lambda')\big)r_{\vartheta_m(j\Delta),m}(x',\lambda')\right]\mathrm d\lambda'.
		\label{eq:sampled_population_drift_consistency}
	\end{align}
\end{proof}

Under \cref{prop:sampled_continuous_time_dynamics}, suppose the local excess-risk density is differentiable along almost every sample path and differentiation commutes with expectation. For fixed $\lambda$ and $\tau\in[j\Delta,(j+1)\Delta)$,
\begin{align}
	\frac{\mathrm d}{\mathrm d\tau}\mathbb E[\bar{\ell}_{\tau,m}(\lambda)]&=-M(\lambda)\mathbb E\!\left[\frac{M(\Lambda_j)}{\rho(\Lambda_j)}\Big\langle r_{\tau,m}(\cdot,\lambda),\Theta_{\tau,m}\big((\cdot,\lambda),(X_j,\Lambda_j)\big)r_{\tau,m}(X_j,\Lambda_j)\Big\rangle_{L^2(q_\lambda)}\right]\nonumber\\
	&\quad-M(\lambda)\mathbb E\!\left[\frac{M(\Lambda_j)}{\rho(\Lambda_j)}\Big\langle r_{\tau,m}(\cdot,\lambda),\Theta_{\tau,m}\big((\cdot,\lambda),(X_j,\Lambda_j)\big)\xi_j\Big\rangle_{L^2(q_\lambda)}\right].
	\label{eq:exact_local_excess_derivative_quantified}
\end{align}
\begin{align}
	&\mathbb E\!\left[\frac{M(\Lambda_j)}{\rho(\Lambda_j)}\Big\langle r_{\tau,m}(\cdot,\lambda),\Theta_{\tau,m}\big((\cdot,\lambda),(X_j,\Lambda_j)\big)\xi_j\Big\rangle_{L^2(q_\lambda)}\right]\nonumber\\
	&\quad=\left\langle\mathbb E[r_{\tau,m}(\cdot,\lambda)],\mathbb E\!\left[\frac{M(\Lambda_j)}{\rho(\Lambda_j)}\Theta_{\tau,m}\big((\cdot,\lambda),(X_j,\Lambda_j)\big)\xi_j\right]\right\rangle_{L^2(q_\lambda)}\nonumber\\
	&\qquad+\operatorname{Tr}\!\left(\operatorname{Cov}\!\left(r_{\tau,m}(\cdot,\lambda),\frac{M(\Lambda_j)}{\rho(\Lambda_j)}\Theta_{\tau,m}\big((\cdot,\lambda),(X_j,\Lambda_j)\big)\xi_j\right)\right).
	\label{eq:local_training_cross_covariance}
\end{align}

\begin{proof}[Proof of \Cref{lem:expected_local_excess_risk_evolution}]
	By \Eqref{eq:elbo_bayes_decomposition}, differentiation along the sampled trajectory gives
	\begin{align}
		\frac{\mathrm d}{\mathrm d\tau}\bar{\ell}_{\tau,m}(\lambda)&=\frac{\mathrm d}{\mathrm d\tau}\left(\frac{M(\lambda)}{2}\mathbb E_{x\sim q_\lambda}\|r_{\tau,m}(x,\lambda)\|_2^2\right)\nonumber\\
		&=M(\lambda)\mathbb E_{x\sim q_\lambda}\!\left[\left\langle r_{\tau,m}(x,\lambda),\frac{\partial}{\partial\tau}r_{\tau,m}(x,\lambda)\right\rangle\right]\nonumber\\
		&=-M(\lambda)\frac{M(\Lambda_j)}{\rho(\Lambda_j)}\Big\langle r_{\tau,m}(\cdot,\lambda),\Theta_{\tau,m}\big((\cdot,\lambda),(X_j,\Lambda_j)\big)r_{\tau,m}(X_j,\Lambda_j)\Big\rangle_{L^2(q_\lambda)}\nonumber\\
		&\quad-M(\lambda)\frac{M(\Lambda_j)}{\rho(\Lambda_j)}\Big\langle r_{\tau,m}(\cdot,\lambda),\Theta_{\tau,m}\big((\cdot,\lambda),(X_j,\Lambda_j)\big)\xi_j\Big\rangle_{L^2(q_\lambda)}.
		\label{eq:local_excess_derivative_chain_rule}
	\end{align}
	Taking expectations gives \Eqref{eq:exact_local_excess_derivative_quantified}.

	For square-integrable random elements in $L^2(q_\lambda;\mathbb R^d)$, expanding their inner product around their expectations gives \Eqref{eq:local_training_cross_covariance}. By \cref{lem:elbo_bayes_decomposition}, the integrated excess risk and the population ELBO have identical parameter derivatives.

	Combining \Eqref{eq:joint_residual_ode} with the joint NTK factorization yields
	\begin{align}
		&\left.\frac{\mathrm d}{\mathrm d\tau}\int_{\lambda_{\min}}^{\lambda_{\max}}\bar{\ell}_{\tau,m}(\lambda)\,\mathrm d\lambda\right|_{\mathrm{population}}\nonumber\\
		&\quad=-\int_{\lambda_{\min}}^{\lambda_{\max}}\int_{\lambda_{\min}}^{\lambda_{\max}}M(\lambda)M(\lambda')\mathbb E_{x\sim q_\lambda,x'\sim q_{\lambda'}}\!\left[r_{\tau,m}(x,\lambda)^\top\Theta_{\tau,m}\big((x,\lambda),(x',\lambda')\big)r_{\tau,m}(x',\lambda')\right]\mathrm d\lambda'\mathrm d\lambda\nonumber\\
		&\quad=-\left\|\int_{\lambda_{\min}}^{\lambda_{\max}}M(\lambda)\mathbb E_{x\sim q_\lambda}\!\left[\big(\nabla_\theta f_{\tau,m}(x,\lambda)\big)^\top r_{\tau,m}(x,\lambda)\right]\mathrm d\lambda\right\|_2^2\nonumber\\
		&\quad=-\|\nabla_\theta\mathcal L_{\mathrm{ELBO},m}(\theta_m(\tau))\|_2^2.
		\label{eq:population_excess_risk_dissipation}
	\end{align}
	For the sampled trajectory, adding and subtracting the population vector field gives
	\begin{align}
		&\frac{\mathrm d}{\mathrm d\tau}\mathbb E\!\left[\int_{\lambda_{\min}}^{\lambda_{\max}}\bar{\ell}_{\tau,m}(\lambda)\,\mathrm d\lambda\right]\nonumber\\
		&\quad=\mathbb E\!\left[\left\langle\nabla_\theta\mathcal L_{\mathrm{ELBO},m}(\vartheta_m(\tau)),\frac{\mathrm d\vartheta_m(\tau)}{\mathrm d\tau}\right\rangle\right]\nonumber\\
		&\quad=-\mathbb E\!\left[\|\nabla_\theta\mathcal L_{\mathrm{ELBO},m}(\vartheta_m(\tau))\|_2^2\right]+\mathbb E\!\left[\left\langle\nabla_\theta\mathcal L_{\mathrm{ELBO},m}(\vartheta_m(\tau)),\frac{\mathrm d\vartheta_m(\tau)}{\mathrm d\tau}+\nabla_\theta\mathcal L_{\mathrm{ELBO},m}(\vartheta_m(\tau))\right\rangle\right].
		\label{eq:sampled_excess_relative_to_population_dissipation_proof}
	\end{align}
	No conditional independence between the held sample and the current parameters is required.

	At a sampling time, the newly drawn sample is independent of the current parameters. By \cref{lem:optimization_equivalence,lem:gradient_forcing},
	\begin{align}
		&\mathbb E\!\left[\left.\frac{\mathrm d}{\mathrm d\tau}\mathcal L_{\mathrm{ELBO},m}(\vartheta_m(\tau))\right|_{\tau=j\Delta}\Bigm|\vartheta_m(j\Delta)\right]\nonumber\\
		&\quad=\mathbb E\!\left[\left.\left\langle\nabla_\theta\mathcal L_{\mathrm{ELBO},m}(\vartheta_m(j\Delta)),\frac{\mathrm d\vartheta_m(j\Delta)}{\mathrm d\tau}\right\rangle\right|\vartheta_m(j\Delta)\right]\nonumber\\
		&\quad=-\left\langle\nabla_\theta\mathcal L_{\mathrm{ELBO},m}(\vartheta_m(j\Delta)),\mathbb E\!\left[\nabla_\theta\widehat{\ell}_m(\vartheta_m(j\Delta);x_0,\epsilon,\lambda)\mid\vartheta_m(j\Delta)\right]\right\rangle\nonumber\\
		&\quad=-\left\|\nabla_\theta\mathcal L_{\mathrm{ELBO},m}(\vartheta_m(j\Delta))\right\|_2^2.
		\label{eq:expected_population_loss_dissipation_proof}
	\end{align}
\end{proof}

Fix a reference parameter $\theta_0$ and consider the frozen linearized model with Jacobian $J_m(x,\lambda)\triangleq\nabla_\theta f_{\theta_0,m}(x,\lambda)$. Suppose there exists a parameter $\theta_m^\star$ in the active tangent parameter subspace such that
\begin{align}
	f_{\theta_0,m}(x,\lambda)+J_m(x,\lambda)(\theta_m^\star-\theta_0)&=f_\lambda^\star(x)
	\label{eq:tangent_bayes_realizability}
\end{align}
for $q_\lambda$-almost every $x$ and almost every $\lambda$ with $M(\lambda)>0$. Then $r_{\tau,m}(x,\lambda)=J_m(x,\lambda)(\vartheta_m(\tau)-\theta_m^\star)$.

Define the Hessian and integrated propagator by
\begin{align}
	H_{x,\lambda}&\triangleq\frac{M(\lambda)}{\rho(\lambda)}J_m(x,\lambda)^\top J_m(x,\lambda),\qquad B_s(x,\lambda)\triangleq\int_0^s e^{-uH_{x,\lambda}}\mathrm du.
	\label{eq:sampled_tangent_curvature}
\end{align}
Suppose the sampling and initialization are independent and the required second moments are finite. For any holding interval $[j\Delta,(j+1)\Delta)$ and $0\le s<\Delta$,
\begin{align}
	\vartheta_m(j\Delta+s)-\theta_m^\star&=e^{-sH_{X_j,\Lambda_j}}\big(\vartheta_m(j\Delta)-\theta_m^\star\big)-B_s(X_j,\Lambda_j)\frac{M(\Lambda_j)}{\rho(\Lambda_j)}J_m(X_j,\Lambda_j)^\top\xi_j.
	\label{eq:frozen_sampled_training_solution}
\end{align}
\begin{align}
	&\operatorname{Cov}\!\left(\vartheta_m(j\Delta+s)\mid\vartheta_m(j\Delta),X_j,\Lambda_j\right)\nonumber\\
	&\quad=B_s(X_j,\Lambda_j)\left(\frac{M(\Lambda_j)}{\rho(\Lambda_j)}\right)^2J_m(X_j,\Lambda_j)^\top\operatorname{Cov}(Y_j\mid X_j,\Lambda_j)J_m(X_j,\Lambda_j)B_s(X_j,\Lambda_j).
	\label{eq:conditional_training_covariance}
\end{align}
For evaluation pairs $(x,\lambda)$ and $(x',\lambda')$,
\begin{align}
	\operatorname{Cov}\big(r_{\tau,m}(x,\lambda),r_{\tau,m}(x',\lambda')\big)&=J_m(x,\lambda)\operatorname{Cov}(\vartheta_m(\tau))J_m(x',\lambda')^\top.
	\label{eq:function_space_training_covariance}
\end{align}

For fixed $\lambda$, write $H_\lambda\triangleq M(\lambda)\mathbb E_{x\sim q_\lambda}[J_m(x,\lambda)^\top J_m(x,\lambda)]$.

\begin{proof}[Proof of \Cref{thm:covariance_propagation}]
	Substituting the frozen factorization $\Theta_{\theta_0,m}((x,\lambda),(x',\lambda'))=J_m(x,\lambda)J_m(x',\lambda')^\top$ and the estimation error $r_{\tau,m}(x,\lambda)=J_m(x,\lambda)(\vartheta_m(\tau)-\theta_m^\star)$ into \Eqref{eq:joint_residual_ode} yields
	\begin{align}
		&-\int_{\lambda_{\min}}^{\lambda_{\max}}M(\lambda')\mathbb E_{x'\sim q_{\lambda'}}\!\left[\Theta_{\theta_0,m}\big((x,\lambda),(x',\lambda')\big)r_{\tau,m}(x',\lambda')\right]\mathrm d\lambda'\nonumber\\
		&\quad=-J_m(x,\lambda)\left(\int_{\lambda_{\min}}^{\lambda_{\max}}M(\lambda')\mathbb E_{x'\sim q_{\lambda'}}\!\left[J_m(x',\lambda')^\top J_m(x',\lambda')\right]\mathrm d\lambda'\right)(\vartheta_m(\tau)-\theta_m^\star)\nonumber\\
		&\quad=-J_m(x,\lambda)\,\mathbb E_{\lambda'\sim\rho,\,x'\sim q_{\lambda'}}[H_{x',\lambda'}](\vartheta_m(\tau)-\theta_m^\star).
		\label{eq:population_curvature_from_joint_ntk}
	\end{align}

	Under \Eqref{eq:tangent_bayes_realizability}, the sampled parameter equation \Eqref{eq:training_ode_bayes_forcing_quantified} becomes, for $\tau\in[j\Delta,(j+1)\Delta)$,
	\begin{align}
		\frac{\mathrm d\vartheta_m(\tau)}{\mathrm d\tau}&=-\frac{M(\Lambda_j)}{\rho(\Lambda_j)}J_m(X_j,\Lambda_j)^\top\left[J_m(X_j,\Lambda_j)(\vartheta_m(\tau)-\theta_m^\star)+\xi_j\right]\nonumber\\
		&=-H_{X_j,\Lambda_j}\big(\vartheta_m(\tau)-\theta_m^\star\big)-\frac{M(\Lambda_j)}{\rho(\Lambda_j)}J_m(X_j,\Lambda_j)^\top\xi_j.
		\label{eq:frozen_tangent_parameter_ode}
	\end{align}
	Within a holding interval, the sampled input and residual remain fixed. Multiplying by $e^{sH_{X_j,\Lambda_j}}$ and integrating over $[0,s]$ gives
	\begin{equation}
		e^{sH_{X_j,\Lambda_j}}\big(\vartheta_m(j\Delta+s)-\theta_m^\star\big)-\big(\vartheta_m(j\Delta)-\theta_m^\star\big)=-\int_0^s e^{uH_{X_j,\Lambda_j}}\frac{M(\Lambda_j)}{\rho(\Lambda_j)}J_m(X_j,\Lambda_j)^\top\xi_j\,\mathrm du.
	\end{equation}
	Multiplying by $e^{-sH_{X_j,\Lambda_j}}$ yields
	\begin{align}
		\vartheta_m(j\Delta+s)-\theta_m^\star
		&=e^{-sH_{X_j,\Lambda_j}}\big(\vartheta_m(j\Delta)-\theta_m^\star\big)-\int_0^s e^{-(s-u)H_{X_j,\Lambda_j}}\frac{M(\Lambda_j)}{\rho(\Lambda_j)}J_m(X_j,\Lambda_j)^\top\xi_j\,\mathrm du\nonumber\\
		&=e^{-sH_{X_j,\Lambda_j}}\big(\vartheta_m(j\Delta)-\theta_m^\star\big)-B_s(X_j,\Lambda_j)\frac{M(\Lambda_j)}{\rho(\Lambda_j)}J_m(X_j,\Lambda_j)^\top\xi_j.
		\label{eq:frozen_sampled_training_solution_proof}
	\end{align}

	Since the fresh sample is independent of the parameters at the beginning of the interval,
	\begin{align}
		&\mathbb E[\vartheta_m(j\Delta+s)-\theta_m^\star\mid\vartheta_m(j\Delta),X_j,\Lambda_j]\nonumber\\
		&\quad=e^{-sH_{X_j,\Lambda_j}}\big(\vartheta_m(j\Delta)-\theta_m^\star\big)-B_s(X_j,\Lambda_j)\frac{M(\Lambda_j)}{\rho(\Lambda_j)}J_m(X_j,\Lambda_j)^\top\underbrace{\mathbb E[\xi_j\mid X_j,\Lambda_j]}_{=0}\nonumber\\
		&\quad=e^{-sH_{X_j,\Lambda_j}}\big(\vartheta_m(j\Delta)-\theta_m^\star\big).
		\label{eq:conditional_tangent_parameter_mean}
	\end{align}
	Subtracting this conditional mean from \Eqref{eq:frozen_sampled_training_solution} yields
	\begin{align}
		&\vartheta_m(j\Delta+s)-\mathbb E[\vartheta_m(j\Delta+s)\mid\vartheta_m(j\Delta),X_j,\Lambda_j]=-B_s(X_j,\Lambda_j)\frac{M(\Lambda_j)}{\rho(\Lambda_j)}J_m(X_j,\Lambda_j)^\top\xi_j.
		\label{eq:conditional_parameter_centering}
	\end{align}
	Taking the conditional outer product and using $\mathbb E[\xi_j\xi_j^\top\mid X_j,\Lambda_j]=\operatorname{Cov}(Y_j\mid X_j,\Lambda_j)$ gives \Eqref{eq:conditional_training_covariance}.

	By the law of total covariance,
	\begin{align}
		&\operatorname{Cov}(\vartheta_m(j\Delta+s))\nonumber\\
		&\quad=\operatorname{Cov}\!\left(\mathbb E[\vartheta_m(j\Delta+s)\mid\vartheta_m(j\Delta),X_j,\Lambda_j]\right)+\mathbb E\!\left[\operatorname{Cov}(\vartheta_m(j\Delta+s)\mid\vartheta_m(j\Delta),X_j,\Lambda_j)\right]\nonumber\\
		&\quad=\operatorname{Cov}\!\left(e^{-sH_{X_j,\Lambda_j}}\big(\vartheta_m(j\Delta)-\theta_m^\star\big)\right)\nonumber\\
		&\qquad+\mathbb E\!\left[B_s(X_j,\Lambda_j)\left(\frac{M(\Lambda_j)}{\rho(\Lambda_j)}\right)^2J_m(X_j,\Lambda_j)^\top\operatorname{Cov}(Y_j\mid X_j,\Lambda_j)J_m(X_j,\Lambda_j)B_s(X_j,\Lambda_j)\right].
		\label{eq:training_covariance_propagation_proof}
	\end{align}
	Since the frozen linearized estimation error is a deterministic linear transformation of $\vartheta_m(\tau)-\theta_m^\star$, its covariance satisfies \Eqref{eq:function_space_training_covariance}.

	By \Eqref{eq:elbo_bayes_decomposition} and \Eqref{eq:tangent_bayes_realizability},
	\begin{align}
		\mathbb E\!\left[\int_{\lambda_{\min}}^{\lambda_{\max}}\bar{\ell}_{\tau,m}(\lambda)\,\mathrm d\lambda\right]&=\frac12\mathbb E\!\left[(\vartheta_m(\tau)-\theta_m^\star)^\top H(\vartheta_m(\tau)-\theta_m^\star)\right]\nonumber\\
		&=\frac12(\mathbb E[\vartheta_m(\tau)]-\theta_m^\star)^\top H(\mathbb E[\vartheta_m(\tau)]-\theta_m^\star)+\frac12\operatorname{Tr}\big(H\operatorname{Cov}(\vartheta_m(\tau))\big).
		\label{eq:expected_excess_mean_covariance_identity_proof}
	\end{align}
	For fixed $\lambda$, $\bar{\ell}_{\tau,m}(\lambda)=\frac12(\vartheta_m(\tau)-\theta_m^\star)^\top H_\lambda(\vartheta_m(\tau)-\theta_m^\star)$. Expanding around the mean gives
	\begin{align}
		\mathbb E[\bar{\ell}_{\tau,m}(\lambda)]&=\frac12\big(\mathbb E[\vartheta_m(\tau)]-\theta_m^\star\big)^\top H_\lambda\big(\mathbb E[\vartheta_m(\tau)]-\theta_m^\star\big)+\frac12\operatorname{Tr}\big(H_\lambda\operatorname{Cov}(\vartheta_m(\tau))\big).
		\label{eq:expected_local_excess_mean_covariance_identity}
	\end{align}
\end{proof}

Under \cref{thm:covariance_propagation}, suppose the frozen linearized training trajectory remains in the active tangent parameter subspace, the initial parameters are deterministic, and the Hessian satisfies
\begin{equation}
	0\preceq H_{x,\lambda}\preceq LI,\qquad \mu I\preceq H=\int_{\lambda_{\min}}^{\lambda_{\max}}H_\lambda\,\mathrm d\lambda
	\label{eq:expected_tangent_curvature_condition}
\end{equation}
on that subspace, for constants $L<\infty$ and $\mu>0$. Define
\begin{equation}
	\Gamma_{\xi,m}(\lambda)\triangleq\frac{M(\lambda)^2}{\rho(\lambda)}\mathbb E_{x\sim q_\lambda}\!\left[J_m(x,\lambda)^\top\operatorname{Cov}(y_\lambda\mid x_\lambda=x)J_m(x,\lambda)\right].
	\label{eq:levelwise_irreducible_covariance_density}
\end{equation}
Assume
\begin{equation}
	V_{\xi,m}\triangleq\int_{\lambda_{\min}}^{\lambda_{\max}}\operatorname{Tr}(\Gamma_{\xi,m}(\lambda))\,\mathrm d\lambda<\infty.
	\label{eq:irreducible_gradient_variance_scalar}
\end{equation}

\begin{proof}[Proof of \Cref{thm:elbo_asymptotics}]
	By \Eqref{eq:frozen_sampled_training_solution} and \Eqref{eq:conditional_tangent_parameter_mean}, conditional centering makes the mixed terms in the squared parameter error vanish. Writing $u_j\triangleq\mathbb E\|\vartheta_m(j\Delta)-\theta_m^\star\|_2^2$,
	\begin{align}
		\mathbb E\|\vartheta_m(j\Delta+s)-\theta_m^\star\|_2^2
		&\quad=\mathbb E\!\left[(\vartheta_m(j\Delta)-\theta_m^\star)^\top e^{-2sH_{X_j,\Lambda_j}}(\vartheta_m(j\Delta)-\theta_m^\star)\right]\nonumber\\
		&\qquad+\mathbb E\!\left[\left(\frac{M(\Lambda_j)}{\rho(\Lambda_j)}\right)^2\|B_s(X_j,\Lambda_j)J_m(X_j,\Lambda_j)^\top\xi_j\|_2^2\right].
	\end{align}
	Since $0\preceq H_{x,\lambda}\preceq LI$, spectral calculus gives $e^{-2Ls}I\preceq e^{-2sH_{x,\lambda}}\preceq I-(1-e^{-2Ls})H_{x,\lambda}/L$ and $(1-e^{-Ls})I/L\preceq B_s(x,\lambda)\preceq sI$. Using $\mathbb E[H_{X_j,\Lambda_j}]=H\succeq\mu I$, independence of the current sample from $\vartheta_m(j\Delta)$, and \Eqref{eq:irreducible_gradient_variance_scalar},
	\begin{align}
		e^{-2Ls}u_j+\left(\frac{1-e^{-Ls}}{L}\right)^2V_{\xi,m}&\le\mathbb E\|\vartheta_m(j\Delta+s)-\theta_m^\star\|_2^2\nonumber\\
		&\le\left(1-\frac{\mu}{L}(1-e^{-2Ls})\right)u_j+s^2V_{\xi,m}.
		\label{eq:two_sided_parameter_second_moment}
	\end{align}
	Taking $s=\Delta$ by continuity gives $u_{j+1}\le q_\Delta u_j+\Delta^2V_{\xi,m}$. Iterating from the deterministic initialization yields
	\begin{align}
		u_j\le q_\Delta^j\|\vartheta_m(0)-\theta_m^\star\|_2^2+\Delta^2V_{\xi,m}\sum_{k=0}^{j-1}q_\Delta^k=q_\Delta^j\|\vartheta_m(0)-\theta_m^\star\|_2^2+\frac{\Delta^2V_{\xi,m}}{1-q_\Delta}(1-q_\Delta^j).
		\label{eq:expected_parameter_error_integral_bound}
	\end{align}
	By \Eqref{eq:expected_excess_mean_covariance_identity_proof} and \Eqref{eq:expected_tangent_curvature_condition},
	\begin{align}
		\mathbb E\!\left[\int_{\lambda_{\min}}^{\lambda_{\max}}\bar{\ell}_{j\Delta+s,m}(\lambda)\,\mathrm d\lambda\right]
		&\le\frac{L}{2}\mathbb E\|\vartheta_m(j\Delta+s)-\theta_m^\star\|_2^2\nonumber\\
		&\le\frac{L}{2}\left(u_j+s^2V_{\xi,m}\right)\nonumber\\
		&\le\frac{L}{2}\left[q_\Delta^j\|\vartheta_m(0)-\theta_m^\star\|_2^2+V_{\xi,m}\left(\frac{\Delta^2(1-q_\Delta^j)}{1-q_\Delta}+s^2\right)\right].
		\label{eq:expected_excess_risk_continuous_time_bound}
	\end{align}

	By \Eqref{eq:frozen_sampled_training_solution} and conditional centering, the second moments satisfy a positive linear recursion. For every positive semidefinite matrix $T$,
	\begin{align}
		\operatorname{Tr}\!\left(\mathbb E[e^{-\Delta H_{X_j,\Lambda_j}}T e^{-\Delta H_{X_j,\Lambda_j}}]\right)=\operatorname{Tr}\!\left(T\mathbb E[e^{-2\Delta H_{X_j,\Lambda_j}}]\right)\le q_\Delta\operatorname{Tr}(T).
	\end{align}
	Thus the homogeneous second-moment map is contractive in trace norm, and the recursion has a unique positive semidefinite stationary second-moment matrix $T_\Delta$. By \Eqref{eq:frozen_sampled_training_solution} and \Eqref{eq:conditional_training_covariance},
	\begin{equation}
		\begin{split}
			T_\Delta
			&=\mathbb E\!\left[e^{-\Delta H_{X_j,\Lambda_j}}T_\Delta e^{-\Delta H_{X_j,\Lambda_j}}\right]\\
			&\quad+\mathbb E\!\left[\left(\frac{M(\Lambda_j)}{\rho(\Lambda_j)}\right)^2B_\Delta(X_j,\Lambda_j)J_m(X_j,\Lambda_j)^\top\operatorname{Cov}(Y_j\mid X_j,\Lambda_j)J_m(X_j,\Lambda_j)B_\Delta(X_j,\Lambda_j)\right].
		\end{split}
	\end{equation}
	By \Eqref{eq:expected_parameter_error_integral_bound}, $\operatorname{Tr}(T_\Delta)\le\Delta^2V_{\xi,m}/(1-q_\Delta)=O(\Delta)$ as $\Delta\downarrow0$.

	The bounded Hessian implies $e^{-\Delta H_{x,\lambda}}=I-\Delta H_{x,\lambda}+O(\Delta^2)$ and $B_\Delta(x,\lambda)=\Delta I+O(\Delta^2)$, uniformly in $(x,\lambda)$. Substituting these expansions into the stationary equation and applying \Eqref{eq:irreducible_gradient_covariance} and \Eqref{eq:levelwise_irreducible_covariance_density},
	\begin{align}
		T_\Delta
		&=T_\Delta-\Delta(HT_\Delta+T_\Delta H)+\Delta^2\int_{\lambda_{\min}}^{\lambda_{\max}}\Gamma_{\xi,m}(\lambda')\,\mathrm d\lambda'+O\!\left(\Delta^2\operatorname{Tr}(T_\Delta)+\Delta^3\right)\nonumber\\
		&=T_\Delta-\Delta(HT_\Delta+T_\Delta H)+\Delta^2\int_{\lambda_{\min}}^{\lambda_{\max}}\Gamma_{\xi,m}(\lambda')\,\mathrm d\lambda'+O(\Delta^3).
	\end{align}
	The remainder is controlled in trace norm, and the last equality uses $\operatorname{Tr}(T_\Delta)=O(\Delta)$. Rearranging gives the Lyapunov relation~\cite{mandt2017stochastic}
	\begin{equation}
		HT_\Delta+T_\Delta H=\Delta\int_{\lambda_{\min}}^{\lambda_{\max}}\Gamma_{\xi,m}(\lambda')\,\mathrm d\lambda'+O(\Delta^2).
	\end{equation}
	Since $H\succeq\mu I$ on the active tangent subspace, the Lyapunov operator is invertible. Integrating the corresponding matrix differential identity yields
	\begin{align}
		T_\Delta
		&=\int_0^\infty e^{-tH}(HT_\Delta+T_\Delta H)e^{-tH}\,\mathrm dt\nonumber\\
		&=\Delta\int_0^\infty e^{-tH}\left(\int_{\lambda_{\min}}^{\lambda_{\max}}\Gamma_{\xi,m}(\lambda')\,\mathrm d\lambda'\right)e^{-tH}\,\mathrm dt+O(\Delta^2).
		\label{eq:stationary_parameter_small_interval_expansion}
	\end{align}

	By \Eqref{eq:frozen_sampled_training_solution}, the stationary within-interval second moment differs from $T_\Delta$ by $O(\Delta^2)$ uniformly for $0\le s<\Delta$, since $s\le\Delta$ and $\operatorname{Tr}(T_\Delta)=O(\Delta)$. Therefore, \Eqref{eq:expected_local_excess_mean_covariance_identity} gives
	\begin{align}
		\lim_{j\to\infty}\mathbb E[\bar{\ell}_{j\Delta+s,m}(\lambda)]
		&=\frac12\operatorname{Tr}(H_\lambda T_\Delta)+O_\lambda(\Delta^2)\nonumber\\
		&=\frac{\Delta}{2}\int_0^\infty\operatorname{Tr}\!\left(H_\lambda e^{-tH}\left(\int_{\lambda_{\min}}^{\lambda_{\max}}\Gamma_{\xi,m}(\lambda')\,\mathrm d\lambda'\right)e^{-tH}\right)\mathrm dt+O_\lambda(\Delta^2)\nonumber\\
		&=\frac{\Delta}{2}\int_{\lambda_{\min}}^{\lambda_{\max}}\int_0^\infty\operatorname{Tr}\!\left(H_\lambda e^{-tH}\Gamma_{\xi,m}(\lambda')e^{-tH}\right)\mathrm dt\,\mathrm d\lambda'+O_\lambda(\Delta^2).
		\label{eq:local_long_time_excess_asymptotic_proof}
	\end{align}
	The integrand is nonnegative because $H_\lambda$ and $\Gamma_{\xi,m}(\lambda')$ are positive semidefinite. The leading coefficient is strictly positive precisely when the propagated irreducible gradient covariance is nonzero in directions observed through $H_\lambda$.

	Using $\int_{\lambda_{\min}}^{\lambda_{\max}}H_\lambda\,\mathrm d\lambda=H$,
	\begin{align}
		\lim_{j\to\infty}\mathbb E\!\left[\int_{\lambda_{\min}}^{\lambda_{\max}}\bar{\ell}_{j\Delta+s,m}(\lambda)\,\mathrm d\lambda\right]
		&=\int_{\lambda_{\min}}^{\lambda_{\max}}\lim_{j\to\infty}\mathbb E[\bar{\ell}_{j\Delta+s,m}(\lambda)]\mathrm d\lambda\nonumber\\
		&=\frac{\Delta}{2}\int_0^\infty\operatorname{Tr}\!\left(H e^{-tH}\left(\int_{\lambda_{\min}}^{\lambda_{\max}}\Gamma_{\xi,m}(\lambda')\,\mathrm d\lambda'\right)e^{-tH}\right)\mathrm dt+O(\Delta^2)\nonumber\\
		&=\frac{\Delta}{4}\int_{\lambda_{\min}}^{\lambda_{\max}}\operatorname{Tr}(\Gamma_{\xi,m}(\lambda'))\,\mathrm d\lambda'+O(\Delta^2)\nonumber\\
		&=\frac{V_{\xi,m}}{4}\Delta+O(\Delta^2).
		\label{eq:global_long_time_excess_asymptotic_proof}
	\end{align}
	The third equality follows by integrating the derivative of $\operatorname{Tr}(e^{-tH}\Gamma_{\xi,m}(\lambda')e^{-tH})$ over $t\in[0,\infty)$ and using $H\succeq\mu I$. Interchange of the long-time limit with the evaluation-level integral, as well as the integrated $O(\Delta^2)$ remainder, follows from trace-norm convergence of stationary second moments and $\int H_\lambda\,\mathrm d\lambda=H$. The expansion is uniform over the sampling phase.
\end{proof}
\section{How Representation Degradation Emerges}
\label{apdx:sec:spectral_degradation}

\begin{proof}[Proof of \Cref{prop:quantitative_representation_degradation}]
By \Eqref{eq:stochastic_local_elbo_loss_quantified}, the readout factorization, and \Eqref{eq:bayes_residual_split_quantified},
\begin{align}
\nabla_{\theta_h}\widehat\ell_m&=\frac{M(\lambda)}{\rho(\lambda)}(\nabla_{\theta_h}f_{\theta,m}(x_\lambda,\lambda))^\top(f_{\theta,m}(x_\lambda,\lambda)-y_\lambda)\nonumber\\
&=\frac{M(\lambda)}{\rho(\lambda)}(\nabla_{\theta_h}h_{\theta,m})^\top W_{\theta,m}^\top(r_{\theta,m}(x_\lambda,\lambda)+\xi_\lambda)\nonumber\\
&=\frac{M(\lambda)}{\rho(\lambda)}(\nabla_{\theta_h}h_{\theta,m})^\top W_{\theta,m}^\top r_{\theta,m}(x_\lambda,\lambda)+\frac{M(\lambda)}{\rho(\lambda)}(\nabla_{\theta_h}h_{\theta,m})^\top W_{\theta,m}^\top\xi_\lambda.\label{eq:quantitative_representation_gradient_split_modified}
\end{align}
Conditional on $(x_\lambda,\lambda)$ and fixed parameters, the only random factor is $\xi_\lambda$. Using $\mathbb E[\xi_\lambda\mid x_\lambda,\lambda]=0$ and $\mathbb E[\xi_\lambda\xi_\lambda^\top\mid x_\lambda,\lambda]=\operatorname{Cov}(y_\lambda\mid x_\lambda)$, the conditional covariance of its linear transformation gives \Eqref{eq:representation_gradient_conditional_covariance}.
\end{proof}

Fix a noise level $\lambda$ and inputs $x_1,\ldots,x_n$ independently of the training trajectory. Set $J_i=J_m(x_i,\lambda)$ and define the unnormalized scalar Gram matrix induced by the joint NTK in \Eqref{eq:joint_ntk}:
\begin{equation}
K_\lambda(i,j)=\operatorname{Tr}\!\left(\Theta_{\theta_0,m}((x_i,\lambda),(x_j,\lambda))\right)=\langle J_i,J_j\rangle_{\mathrm F}.
\label{eq:ntk_scalar_gram_definition}
\end{equation}

\begin{proof}[Proof of \Cref{lem:ntk_mode_accessibility}]
\begin{align}
\|B_a^\lambda\|_{\mathrm F}^2&=\left\langle\sum_i(u_a^\lambda)_iJ_i,\sum_j(u_a^\lambda)_jJ_j\right\rangle_{\mathrm F}\nonumber\\
&=\sum_{i,j}(u_a^\lambda)_i(u_a^\lambda)_j\langle J_i,J_j\rangle_{\mathrm F}\nonumber\\
&=(u_a^\lambda)^\top K_\lambda u_a^\lambda\nonumber\\
&=\kappa_a^\lambda.
\end{align}
Therefore, $\|B_a^\lambda v\|_2^2\le\|B_a^\lambda\|_{\mathrm F}^2\|v\|_2^2=\kappa_a^\lambda\|v\|_2^2$.

Partition the parameter Jacobian into derivatives with respect to $W_{\theta,m}$ and the remaining parameters. The corresponding scalar Gram matrices add because the parameter coordinates are disjoint. Since $f_{\theta,m}=W_{\theta,m}h_{\theta,m}$,
\begin{align}
\left\langle\nabla_W f_{\theta_0,m}(x_i,\lambda),\nabla_W f_{\theta_0,m}(x_j,\lambda)\right\rangle_{\mathrm F}&=\sum_{k=1}^d\sum_{r=1}^m(h_{\theta_0,m}(x_i,\lambda))_r(h_{\theta_0,m}(x_j,\lambda))_r\nonumber\\
&=dG_\lambda(i,j).
\end{align}
The remaining parameter block contributes a positive semidefinite Gram matrix, since $v^\top(K_\lambda-dG_\lambda)v=\left\|\sum_i v_i\nabla_{\theta_h}f_{\theta_0,m}(x_i,\lambda)\right\|_{\mathrm F}^2\ge0$. Therefore $K_\lambda\succeq dG_\lambda$. By the Courant--Fischer characterization,
\begin{align}
d\lambda_a(G_\lambda)&=\min_{\dim U=n-a+1}\ \max_{v\in U,\,\|v\|_2=1}v^\top(dG_\lambda)v\nonumber\\
&\le\min_{\dim U=n-a+1}\ \max_{v\in U,\,\|v\|_2=1}v^\top K_\lambda v\nonumber\\
&=\kappa_a^\lambda.
\end{align}
Since $0\preceq dG_\lambda\preceq K_\lambda$, $\ker K_\lambda\subseteq\ker G_\lambda$, proving the rank inequality.
\end{proof}

Under the frozen tangent model of \cref{thm:elbo_asymptotics}, define the linearized hidden representation
\begin{equation}
\widetilde h_{\theta,m}(x,\lambda)=h_{\theta_0,m}(x,\lambda)+A_m(x,\lambda)(\theta-\theta_0),\qquad A_m(x,\lambda)=\nabla_\theta h_{\theta_0,m}(x,\lambda).
\label{eq:tangent_hidden_representation}
\end{equation}
The frozen output NTK assumption alone does not imply this approximation for arbitrary nonlinear hidden layers. Assume the evaluation feature Jacobians have finite second moments whenever an expectation over evaluation inputs is taken.

\begin{proof}[Proof of \Cref{thm:spectral_local_elbo_bound}]
Let $\delta_k=\vartheta_m(k\Delta+s)-\theta_m^\star$. The stationary mean of $\delta_k$ is zero by conditional centering and the contraction condition on $H$. By \Eqref{eq:stationary_parameter_small_interval_expansion} and the within-interval second-moment estimate,
\begin{align}
\lim_{k\to\infty}\mathbb E[\delta_k\delta_k^\top]&=T_\Delta+O(\Delta^2)\nonumber\\
&=\Delta\int_0^\infty e^{-uH}\left(\int_{\lambda_{\min}}^{\lambda_{\max}}\Gamma_{\xi,m}(\nu)\,\mathrm d\nu\right)e^{-uH}\,\mathrm du+O(\Delta^2)\nonumber\\
&=\Delta\int_{\lambda_{\min}}^{\lambda_{\max}}\int_0^\infty e^{-uH}\Gamma_{\xi,m}(\nu)e^{-uH}\,\mathrm du\,\mathrm d\nu+O(\Delta^2).
\label{eq:representation_stationary_second_moment}
\end{align}
Since $B_a^\lambda$ is deterministic and independent of training,
\begin{align}
&\lim_{k\to\infty}\operatorname{Cov}(B_a^\lambda\vartheta_m(k\Delta+s),B_b^\lambda\vartheta_m(k\Delta+s))\nonumber\\
&\quad=B_a^\lambda\left(\lim_{k\to\infty}\mathbb E[\delta_k\delta_k^\top]\right)(B_b^\lambda)^\top\nonumber\\
&\quad=\Delta\int_{\lambda_{\min}}^{\lambda_{\max}}\int_0^\infty B_a^\lambda e^{-uH}\Gamma_{\xi,m}(\nu)e^{-uH}(B_b^\lambda)^\top\,\mathrm du\,\mathrm d\nu+O_\lambda(\Delta^2).
\end{align}
Taking the trace for $a=b$ and using positivity of $\Gamma_{\xi,m}(\nu)$,
\begin{align}
\lim_{k\to\infty}\mathbb E\|B_a^\lambda\delta_k\|_2^2&=\operatorname{Tr}\!\left(B_a^\lambda\left(\lim_{k\to\infty}\mathbb E[\delta_k\delta_k^\top]\right)(B_a^\lambda)^\top\right)\nonumber\\
&=\Delta\int_{\lambda_{\min}}^{\lambda_{\max}}\int_0^\infty\operatorname{Tr}\!\left(B_a^\lambda e^{-uH}\Gamma_{\xi,m}(\nu)e^{-uH}(B_a^\lambda)^\top\right)\,\mathrm du\,\mathrm d\nu+O_\lambda(\Delta^2)\nonumber\\
&=\Delta\int_{\lambda_{\min}}^{\lambda_{\max}}\int_0^\infty\|B_a^\lambda e^{-uH}\Gamma_{\xi,m}(\nu)^{1/2}\|_{\mathrm F}^2\,\mathrm du\,\mathrm d\nu+O_\lambda(\Delta^2).\label{eq:ntk_mode_stationary_asymptotic}
\end{align}
By $H\succeq\mu I$ and \cref{lem:ntk_mode_accessibility},
\begin{align}
\|B_a^\lambda e^{-uH}\Gamma_{\xi,m}(\nu)^{1/2}\|_{\mathrm F}^2&\le\|B_a^\lambda\|_{\mathrm F}^2\|e^{-uH}\|_{\mathrm op}^2\|\Gamma_{\xi,m}(\nu)^{1/2}\|_{\mathrm F}^2\nonumber\\
&\le\kappa_a^\lambda e^{-2\mu u}\operatorname{Tr}(\Gamma_{\xi,m}(\nu)).
\end{align}
Integrating gives
\begin{align}
\int_{\lambda_{\min}}^{\lambda_{\max}}\int_0^\infty\|B_a^\lambda e^{-uH}\Gamma_{\xi,m}(\nu)^{1/2}\|_{\mathrm F}^2\,\mathrm du\,\mathrm d\nu
&\le\kappa_a^\lambda\int_0^\infty e^{-2\mu u}\,\mathrm du\int_{\lambda_{\min}}^{\lambda_{\max}}\operatorname{Tr}(\Gamma_{\xi,m}(\nu))\,\mathrm d\nu\nonumber\\
&=\frac{\kappa_a^\lambda V_{\xi,m}}{2\mu}.
\end{align}
Summing over $a=r+1,\ldots,n$ gives
\begin{equation}
\lim_{k\to\infty}\mathbb E\sum_{a=r+1}^n\|B_a^\lambda\delta_k\|_2^2\le\frac{\Delta V_{\xi,m}}{2\mu}\sum_{a=r+1}^n\kappa_a^\lambda+O_\lambda(\Delta^2).
\label{eq:ntk_tail_stationary_bound}
\end{equation}
The $O(\Delta^2)$ remainder is valid for fixed finite $n$.

By \Eqref{eq:levelwise_irreducible_covariance_density},
\begin{align}
&\|B_a^\lambda e^{-uH}\Gamma_{\xi,m}(\nu)^{1/2}\|_{\mathrm F}^2\nonumber\\
&\quad=\operatorname{Tr}\!\left(B_a^\lambda e^{-uH}\Gamma_{\xi,m}(\nu)e^{-uH}(B_a^\lambda)^\top\right)\nonumber\\
&\quad=\frac{M(\nu)^2}{\rho(\nu)}\mathbb E_{x'\sim q_\nu}\operatorname{Tr}\!\left(B_a^\lambda e^{-uH}J_m(x',\nu)^\top\operatorname{Cov}(y_\nu\mid x_\nu=x')J_m(x',\nu)e^{-uH}(B_a^\lambda)^\top\right)\nonumber\\
&\quad=\frac{M(\nu)^2}{\rho(\nu)}\mathbb E_{x'\sim q_\nu}\left\|B_a^\lambda e^{-uH}J_m(x',\nu)^\top\operatorname{Cov}(y_\nu\mid x_\nu=x')^{1/2}\right\|_{\mathrm F}^2.\label{eq:ntk_mode_source_noise_decomposition}
\end{align}
At $u=0$, $B_a^\lambda J_m(x',\nu)^\top=\sum_i(u_a^\lambda)_i\Theta_{\theta_0,m}((x_i,\lambda),(x',\nu))$.

By \Eqref{eq:tangent_hidden_representation},
\begin{align}
\widetilde h_{\vartheta_m(k\Delta+s),m}(x,\lambda)-\widetilde h_{\theta_m^\star,m}(x,\lambda)&=A_m(x,\lambda)(\vartheta_m(k\Delta+s)-\theta_m^\star)\nonumber\\
&=A_m(x,\lambda)\delta_k.
\end{align}
For fixed evaluation pairs,
\begin{align}
&\lim_{k\to\infty}\operatorname{Cov}(\widetilde h_{\vartheta_m(k\Delta+s),m}(x,\lambda),\widetilde h_{\vartheta_m(k\Delta+s),m}(x',\lambda'))\nonumber\\
&\quad=A_m(x,\lambda)\left(\lim_{k\to\infty}\mathbb E[\delta_k\delta_k^\top]\right)A_m(x',\lambda')^\top\nonumber\\
&\quad=\Delta\int_{\lambda_{\min}}^{\lambda_{\max}}\int_0^\infty A_m(x,\lambda)e^{-uH}\Gamma_{\xi,m}(\nu)e^{-uH}A_m(x',\lambda')^\top\,\mathrm du\,\mathrm d\nu+O_{x,x',\lambda,\lambda'}(\Delta^2).
\end{align}
Setting $(x',\lambda')=(x,\lambda)$ and taking traces gives
\begin{align}
&\lim_{k\to\infty}\mathbb E\|\widetilde h_{\vartheta_m(k\Delta+s),m}(x,\lambda)-\widetilde h_{\theta_m^\star,m}(x,\lambda)\|_2^2\nonumber\\
&\quad=\Delta\int_{\lambda_{\min}}^{\lambda_{\max}}\int_0^\infty\operatorname{Tr}\!\left(A_m(x,\lambda)e^{-uH}\Gamma_{\xi,m}(\nu)e^{-uH}A_m(x,\lambda)^\top\right)\,\mathrm du\,\mathrm d\nu+O_{x,\lambda}(\Delta^2)\nonumber\\
&\quad=\Delta\int_{\lambda_{\min}}^{\lambda_{\max}}\int_0^\infty\|A_m(x,\lambda)e^{-uH}\Gamma_{\xi,m}(\nu)^{1/2}\|_{\mathrm F}^2\,\mathrm du\,\mathrm d\nu+O_{x,\lambda}(\Delta^2).\label{eq:feature_stationary_mean_square_asymptotic}
\end{align}
\end{proof}

\begin{proof}[Proof of \Cref{cor:representation_linear_probe_risk}]
Write $h^*(x)=\widetilde h_{\theta_m^\star,m}(x,\lambda)$ and $\delta_k=\vartheta_m(k\Delta+s)-\theta_m^\star$.
\begin{align}
&(b^\top\widetilde h_{\vartheta_m(k\Delta+s),m}(x_\lambda,\lambda)-z)^2\nonumber\\
&\quad=(b^\top h^*(x_\lambda)-z+b^\top A_m(x_\lambda,\lambda)\delta_k)^2\nonumber\\
&\quad=(b^\top h^*(x_\lambda)-z)^2+2(b^\top h^*(x_\lambda)-z)b^\top A_m(x_\lambda,\lambda)\delta_k+(b^\top A_m(x_\lambda,\lambda)\delta_k)^2.
\end{align}
By independence and the vanishing stationary mean,
\begin{align}
&\lim_{k\to\infty}\mathbb E[(b^\top h^*(x_\lambda)-z)b^\top A_m(x_\lambda,\lambda)\delta_k]\nonumber\\
&\quad=\mathbb E_{x_\lambda,z}[(b^\top h^*(x_\lambda)-z)b^\top A_m(x_\lambda,\lambda)]\lim_{k\to\infty}\mathbb E[\delta_k]\nonumber\\
&\quad=0.
\end{align}
By \Eqref{eq:representation_stationary_second_moment},
\begin{align}
&\lim_{k\to\infty}\mathbb E[(b^\top A_m(x_\lambda,\lambda)\delta_k)^2]\nonumber\\
&\quad=\mathbb E_{x\sim q_\lambda}\operatorname{Tr}\!\left(b^\top A_m(x,\lambda)\left(\lim_{k\to\infty}\mathbb E[\delta_k\delta_k^\top]\right)A_m(x,\lambda)^\top b\right)\nonumber\\
&\quad=\Delta\int_{\lambda_{\min}}^{\lambda_{\max}}\int_0^\infty\mathbb E_{x\sim q_\lambda}\operatorname{Tr}\!\left(b^\top A_m(x,\lambda)e^{-uH}\Gamma_{\xi,m}(\nu)e^{-uH}A_m(x,\lambda)^\top b\right)\,\mathrm du\,\mathrm d\nu+O_\lambda(\Delta^2)\nonumber\\
&\quad=\Delta\int_{\lambda_{\min}}^{\lambda_{\max}}\int_0^\infty\mathbb E_{x\sim q_\lambda}\|b^\top A_m(x,\lambda)e^{-uH}\Gamma_{\xi,m}(\nu)^{1/2}\|_2^2\,\mathrm du\,\mathrm d\nu+O_\lambda(\Delta^2).
\end{align}

For an optimized population probe, let $C=\mathbb E[h^*(x_\lambda)h^*(x_\lambda)^\top]$, and let $Q_\Delta$ be the stationary covariance of the feature perturbation averaged over evaluation samples. Then $Q_\Delta=\Delta Q+O(\Delta^2)$, where
\begin{equation}
Q=\int_{\lambda_{\min}}^{\lambda_{\max}}\int_0^\infty\mathbb E_{x\sim q_\lambda}[A_m(x,\lambda)e^{-uH}\Gamma_{\xi,m}(\nu)e^{-uH}A_m(x,\lambda)^\top]\,\mathrm du\,\mathrm d\nu.
\end{equation}
Let $b_0$ minimize the reference population squared risk and $b_\Delta$ minimize the joint stationary population risk. Their normal equations are $Cb_0=\mathbb E[h^*(x_\lambda)z]$ and $(C+Q_\Delta)b_\Delta=\mathbb E[h^*(x_\lambda)z]$. Subtracting and using $C\succ0$ and $Q_\Delta=O(\Delta)$ gives
\begin{align*}
(C+Q_\Delta)(b_\Delta-b_0)&=-Q_\Delta b_0,\\
b_\Delta-b_0&=-(C+Q_\Delta)^{-1}Q_\Delta b_0=O(\Delta).
\end{align*}
Completing the square around $b_\Delta$ yields
\begin{align}
&\mathbb E[(b_\Delta^\top\widetilde h_{\vartheta_m,m}(x_\lambda,\lambda)-z)^2]-\mathbb E[(b_0^\top h^*(x_\lambda)-z)^2]\nonumber\\
&\quad=b_0^\top Q_\Delta b_0-(b_\Delta-b_0)^\top(C+Q_\Delta)(b_\Delta-b_0)\nonumber\\
&\quad=\Delta b_0^\top Qb_0+O(\Delta^2).
\end{align}
This comparison concerns the joint stationary population, not probes fitted separately to individual training realizations.
\end{proof}

\section{Representation Preservation Diffusion Models}
\label{apdx:sec:unified_weight}

\begin{proof}[Proof of \Cref{prop:unified_target_adaptive_weight}]
Under the scalar tangent model with $d=1$ and $J_m=1$, normalization gives $H=\int_{\lambda_{\min}}^{\lambda_{\max}}M(\lambda)\,\mathrm d\lambda=1$. By \cref{thm:elbo_asymptotics}, the leading coefficient of the stationary parameter variance satisfies
\begin{align}
S_{\xi,m}(M)&=\int_0^\infty e^{-2u}\,\mathrm du\int_{\lambda_{\min}}^{\lambda_{\max}}\Gamma_{\xi,m}(\lambda)\,\mathrm d\lambda\nonumber\\
&=\frac12\int_{\lambda_{\min}}^{\lambda_{\max}}\frac{M(\lambda)^2}{\rho(\lambda)}\mathbb E_{x\sim q_\lambda}[\operatorname{Var}(y_\lambda\mid x)]\,\mathrm d\lambda.
\label{eq:scalar_stationary_variance}
\end{align}

Suppose that $M_\star$ satisfies \Eqref{eq:optimal_variance_allocation} with constant $C\geq0$. For any admissible allocation $M$ with finite stationary variance, we have
\begin{align}
S_{\xi,m}(M)-S_{\xi,m}(M_\star)&=\frac12\int_{\lambda_{\min}}^{\lambda_{\max}}\frac{M(\lambda)^2-M_\star(\lambda)^2}{\rho(\lambda)}\mathbb E_{x\sim q_\lambda}[\operatorname{Var}(y_\lambda\mid x)]\,\mathrm d\lambda\nonumber\\
&=\frac12\int_{\lambda_{\min}}^{\lambda_{\max}}\frac{(M(\lambda)-M_\star(\lambda))^2}{\rho(\lambda)}\mathbb E_{x\sim q_\lambda}[\operatorname{Var}(y_\lambda\mid x)]\,\mathrm d\lambda+C\int_{\lambda_{\min}}^{\lambda_{\max}}(M(\lambda)-M_\star(\lambda))\,\mathrm d\lambda\nonumber\\
&=\frac12\int_{\lambda_{\min}}^{\lambda_{\max}}\frac{(M(\lambda)-M_\star(\lambda))^2}{\rho(\lambda)}\mathbb E_{x\sim q_\lambda}[\operatorname{Var}(y_\lambda\mid x)]\,\mathrm d\lambda\geq0.
\label{eq:optimal_variance_allocation_proof}
\end{align}
The last equality follows from the normalization constraint, and the inequality follows from the nonnegativity of conditional variance. Integrands involving division by $\rho(\lambda)$ contribute zero outside its support. Thus, $M_\star$ minimizes the leading coefficient of the stationary parameter variance for both $C>0$ and $C=0$.
\end{proof}

\begin{proof}[Derivation of the RPDM allocation]
Define $\widetilde y_\lambda=c_{x,\lambda}\alpha_\lambda x_0+c_{\epsilon,\lambda}\sigma_\lambda\epsilon$.
\begin{align}
\frac{1}{d}\mathbb E\|\widetilde y_\lambda\|_2^2&=\frac{1}{d}\mathbb E\left[\left\|c_{x,\lambda}\alpha_\lambda x_0+c_{\epsilon,\lambda}\sigma_\lambda\epsilon\right\|_2^2\right]\nonumber\\
&=\frac{1}{d}\mathbb E\left[c_{x,\lambda}^2\alpha_\lambda^2\|x_0\|_2^2+c_{\epsilon,\lambda}^2\sigma_\lambda^2\|\epsilon\|_2^2+2c_{x,\lambda}c_{\epsilon,\lambda}\alpha_\lambda\sigma_\lambda\langle x_0,\epsilon\rangle\right]\nonumber\\
&=\frac{c_{x,\lambda}^2\alpha_\lambda^2}{d}\mathbb E\|x_0\|_2^2+\frac{c_{\epsilon,\lambda}^2\sigma_\lambda^2}{d}\mathbb E\|\epsilon\|_2^2+\frac{2c_{x,\lambda}c_{\epsilon,\lambda}\alpha_\lambda\sigma_\lambda}{d}\mathbb E\langle x_0,\epsilon\rangle.
\label{eq:effective_score_expansion_unified_weight}
\end{align}

By independence of $x_0$ and $\epsilon$ and $\mathbb E[\epsilon]=0$,
\begin{equation}
\mathbb E\langle x_0,\epsilon\rangle=\mathbb E_{x_0}\left[x_0^\top\mathbb E[\epsilon\mid x_0]\right]=\mathbb E_{x_0}\left[x_0^\top\mathbb E[\epsilon]\right]=0.
\label{eq:cross_moment_identity_unified_weight}
\end{equation}

Writing $\mu_x=\mathbb E[x_0]$ and $\Sigma_x=\operatorname{Cov}(x_0)$,
\begin{align}
\mathbb E\|x_0\|_2^2&=\mathbb E\|x_0-\mu_x+\mu_x\|_2^2\nonumber\\
&=\mathbb E\|x_0-\mu_x\|_2^2+2\langle\mathbb E[x_0-\mu_x],\mu_x\rangle+\|\mu_x\|_2^2\nonumber\\
&=\operatorname{Tr}(\Sigma_x)+\|\mu_x\|_2^2=dq_x.
\label{eq:second_moment_decomposition_unified_weight}
\end{align}
Since $\mathbb E\|\epsilon\|_2^2=\sum_{i=1}^d\mathbb E[\epsilon_i^2]=d$, substituting into \Eqref{eq:effective_score_expansion_unified_weight} gives
\begin{align}
\frac{1}{d}\mathbb E\|\widetilde y_\lambda\|_2^2&=\frac{c_{x,\lambda}^2\alpha_\lambda^2}{d}\,dq_x+\frac{c_{\epsilon,\lambda}^2\sigma_\lambda^2}{d}\,d+\frac{2c_{x,\lambda}c_{\epsilon,\lambda}\alpha_\lambda\sigma_\lambda}{d}\cdot0\nonumber\\
&=c_{x,\lambda}^2\alpha_\lambda^2q_x+c_{\epsilon,\lambda}^2\sigma_\lambda^2.
\label{eq:effective_score_squared_unified_weight}
\end{align}

The RPDM loss weight is proportional to the square root of \Eqref{eq:effective_score_squared_unified_weight}, with the proportionality constant determined by the equal-integral constraint. Multiplying by $\rho(\lambda)|\dot{\lambda}_{t(\lambda)}|$ yields the optimization allocation in \Eqref{eq:effective_amplitude_score}. Under uniform timestep sampling, this factor is constant, so the optimization allocation is proportional to the loss weight.
\end{proof}

\begin{proof}[Proof of \Cref{thm:erd_representation_improvement}]
By \cref{thm:spectral_local_elbo_bound,cor:representation_linear_probe_risk}, the reference risks coincide because $\theta_m^\star$ realizes the Bayes predictor for both allocations. Thus,
\begin{align}
\mathcal R_\lambda^{(1)}(b)-\mathcal R_\lambda^{(0)}(b)&=\Delta\mathbb E_{x\sim q_\lambda}\left[b^\top A_m(x,\lambda)S_{\xi,m}^{(1)}A_m(x,\lambda)^\top b\right]-\Delta\mathbb E_{x\sim q_\lambda}\left[b^\top A_m(x,\lambda)S_{\xi,m}^{(0)}A_m(x,\lambda)^\top b\right]+O_\lambda(\Delta^2)\nonumber\\
&=\Delta\mathbb E_{x\sim q_\lambda}\left[b^\top A_m(x,\lambda)(S_{\xi,m}^{(1)}-S_{\xi,m}^{(0)})A_m(x,\lambda)^\top b\right]+O_\lambda(\Delta^2).
\end{align}
A negative leading coefficient therefore yields a $\Theta(\Delta)$ reduction in stationary probe risk.

For the scalar tangent model with normalized allocations and baseline $M_0=\rho$, \Eqref{eq:scalar_stationary_variance} gives
\begin{align}
S_{\xi,m}^{(1)}-S_{\xi,m}^{(0)}&=\frac12\int_{\lambda_{\min}}^{\lambda_{\max}}\frac{M_1(\lambda)^2-M_0(\lambda)^2}{\rho(\lambda)}\mathbb E_{x\sim q_\lambda}[\operatorname{Var}(y_\lambda\mid x)]\,\mathrm d\lambda\nonumber\\
&=\frac12\mathbb E_{\lambda\sim\rho}\left[\left(\frac{M_1(\lambda)^2}{\rho(\lambda)^2}-1\right)\mathbb E_{x\sim q_\lambda}[\operatorname{Var}(y_\lambda\mid x)]\right].
\end{align}

When $M_1$ satisfies \Eqref{eq:optimal_variance_allocation}, its optimality follows from \cref{prop:unified_target_adaptive_weight}. In particular, \Eqref{eq:optimal_variance_allocation_proof} gives
\begin{equation}
S_{\xi,m}^{(0)}-S_{\xi,m}^{(1)}=\frac12\int_{\lambda_{\min}}^{\lambda_{\max}}\frac{(M_0(\lambda)-M_1(\lambda))^2}{\rho(\lambda)}\mathbb E_{x\sim q_\lambda}[\operatorname{Var}(y_\lambda\mid x)]\,\mathrm d\lambda\geq0.
\end{equation}
For a positive constant in \Eqref{eq:optimal_variance_allocation}, strict inequality holds if $M_1/\rho$ is not $\rho$-almost everywhere constant. By \Eqref{eq:erd_probe_risk_improvement}, this yields a $\Theta(\Delta)$ improvement when $\mathbb E_{x\sim q_\lambda}\|A_m(x,\lambda)^\top b\|_2^2>0$.

To quantify the deviation from optimality, suppose that $\frac{M_1(\lambda)}{\rho(\lambda)}\mathbb E_{x\sim q_\lambda}[\operatorname{Var}(y_\lambda\mid x)]$ is continuous and strictly positive on the noise range. For any nonnegative normalized allocation $M$ with finite stationary variance, Cauchy--Schwarz gives
\begin{align}
1&=\left(\int_{\lambda_{\min}}^{\lambda_{\max}}M(\lambda)\,\mathrm d\lambda\right)^2\nonumber\\
&\leq\left(\int_{\lambda_{\min}}^{\lambda_{\max}}\frac{M(\lambda)^2}{\rho(\lambda)}\mathbb E_{x\sim q_\lambda}[\operatorname{Var}(y_\lambda\mid x)]\,\mathrm d\lambda\right)\mathbb E_{\lambda\sim\rho}\left[\left(\mathbb E_{x\sim q_\lambda}[\operatorname{Var}(y_\lambda\mid x)]\right)^{-1}\right].
\end{align}
Equality is attained by the normalized allocation satisfying \Eqref{eq:optimal_variance_allocation} with a positive constant, yielding the minimum coefficient $S_{\xi,m}^{\min}=\big(2\mathbb E_{\lambda\sim\rho}[(\mathbb E_{x\sim q_\lambda}[\operatorname{Var}(y_\lambda\mid x)])^{-1}]\big)^{-1}$. The ratio to the minimum satisfies
\begin{align}
\frac{S_{\xi,m}^{(1)}}{S_{\xi,m}^{\min}}&=\mathbb E_{\lambda\sim\rho}\left[\frac{M_1(\lambda)^2}{\rho(\lambda)^2}\mathbb E_{x\sim q_\lambda}[\operatorname{Var}(y_\lambda\mid x)]\right]\mathbb E_{\lambda\sim\rho}\left[\left(\mathbb E_{x\sim q_\lambda}[\operatorname{Var}(y_\lambda\mid x)]\right)^{-1}\right]\nonumber\\
&=\mathbb E_{\lambda\sim\rho}\left[\frac{M_1(\lambda)}{\rho(\lambda)}\left(\frac{M_1(\lambda)}{\rho(\lambda)}\mathbb E_{x\sim q_\lambda}[\operatorname{Var}(y_\lambda\mid x)]\right)\right]\mathbb E_{\lambda\sim\rho}\left[\frac{M_1(\lambda)}{\rho(\lambda)}\left(\frac{M_1(\lambda)}{\rho(\lambda)}\mathbb E_{x\sim q_\lambda}[\operatorname{Var}(y_\lambda\mid x)]\right)^{-1}\right]\nonumber\\
&\leq\frac{\max_{\lambda\in[\lambda_{\min},\lambda_{\max}]}\left(\frac{M_1(\lambda)}{\rho(\lambda)}\mathbb E_{x\sim q_\lambda}[\operatorname{Var}(y_\lambda\mid x)]\right)}{\min_{\lambda\in[\lambda_{\min},\lambda_{\max}]}\left(\frac{M_1(\lambda)}{\rho(\lambda)}\mathbb E_{x\sim q_\lambda}[\operatorname{Var}(y_\lambda\mid x)]\right)}.
\end{align}
The inequality follows from $\mathbb E_{\lambda\sim\rho}[M_1(\lambda)/\rho(\lambda)]=1$. The ratio tends to one as the product converges uniformly to a positive constant. By \cref{cor:representation_linear_probe_risk}, the comparison extends to optimized linear probes under the stated positive-definiteness and sensitivity conditions.
\end{proof}

\section{Toy Experiments on 2D GMM Dataset}
\label{sec:toy_experiments}

To investigate how diffusion training affects initially structured representations, we conduct controlled experiments on a 2D Gaussian Mixture Model (GMM). Motivated by semantically structured latent spaces, such as those produced by Representation Autoencoders (RAEs)~\cite{zheng2025diffusion}, we construct a simple setting with four well-separated modes as an idealized proxy for distinct semantic components. This controlled setting allows us to examine how diffusion corruption, prediction parameterization, and shared-parameter optimization relate to irreducible gradient fluctuations, excess risk, hidden representations, and empirical NTK geometry.

\subsection{Experimental Setup}
\label{subsec:toy_setup}

\textbf{Dataset Construction.} We synthesize a 2D GMM with 4 symmetric clusters centered at $\mu \in \{(\pm 2.0, \pm 2.0)\}$ as the initial data distribution $q(x_0)$. Samples are generated by uniformly selecting a cluster center and adding Gaussian noise ($\sigma_0 = 0.3$), yielding well-separated modes with nonzero intra-cluster variance. This construction provides a simplified setting for studying the preservation of initially distinguishable latent modes, rather than reproducing the full distribution of empirical RAE latents.

\textbf{Diffusion Processes and Schedules.} For clean data $x_0$, the perturbed data marginal at continuous time $t \sim \mathcal{U}(0, 1)$ is $x_t = \alpha_t x_0 + \sigma_t \epsilon$, where $\epsilon \sim \mathcal{N}(0, I)$. To examine these effects under different Signal-to-Noise Ratio (SNR) trajectories~\cite{kingma2021variational,kingma2023understanding}, we evaluate three schedule interpolants:
\begin{itemize}[leftmargin=0.2in]%,itemsep=1.0pt]
\item \textit{Linear (Flow Matching)}~\cite{lipman2022flow}: $\alpha_t = 1 - t$, $\sigma_t = t$.
\item \textit{Variance Preserving (VP)}~\cite{song2020score}: Standard Score SDE schedule mapped to $t \in [0, 1]$, with a linear noise schedule $\beta(t) = \beta_{\text{min}} + t(\beta_{\text{max}} - \beta_{\text{min}})$, where $\beta_{\text{min}} = 0.1$ and $\beta_{\text{max}} = 20.0$.
\item \textit{Generalized Variance Preserving (GVP)}~\cite{ma2024sit}: A spherical interpolant defined by $\alpha_t = \cos(t\pi/2)$ and $\sigma_t = \sin(t\pi/2)$.
\end{itemize}

\textbf{Prediction Target Formulations.} To compare the effects of output parameterization, we evaluate four prediction targets. The network $f_\theta(x_t, t)$ predicts either the noise $\epsilon$ (standard DDPM~\cite{ho2020denoising}), clean data $x_0$, velocity~\cite{salimans2022progressive} $v_t = \alpha_t \epsilon - \sigma_t x_0$, or flow matching vector field~\cite{lipman2022flow} $u_t = \epsilon - x_0$.

\textbf{Network Architecture.} The prediction network is a time-conditioned MLP mapping 2D inputs $(x_t, t)$ to 2D outputs. Continuous time $t$ is projected into a 64-dimensional sinusoidal positional embedding and concatenated with $x_t$. The feature extractor is a 3-layer MLP (hidden dimension $m=256$, SiLU activations), followed by a linear readout head.

\textbf{Training Details.} We train separate models across the four target formulations ($\epsilon, x_0, v, u$) using the Mean Squared Error (MSE) objective. We use Adam (learning rate $1 \times 10^{-3}$, batch size $512$) for $2,000$ iterations. Alongside global models trained over $t \in [0, 1]$, we train piecewise independent models on restricted time bins (\eg, $t \in [0.0, 0.2]$) to isolate local learning dynamics and establish empirical Bayes baselines.

\begin{figure}[htbp]
\centering
% \vspace{-2.0em}
% Row 1: Left - epsilon-prediction, Right - x-prediction
\begin{subfigure}{0.49\textwidth}
    \centering
    \includegraphics[width=\linewidth]{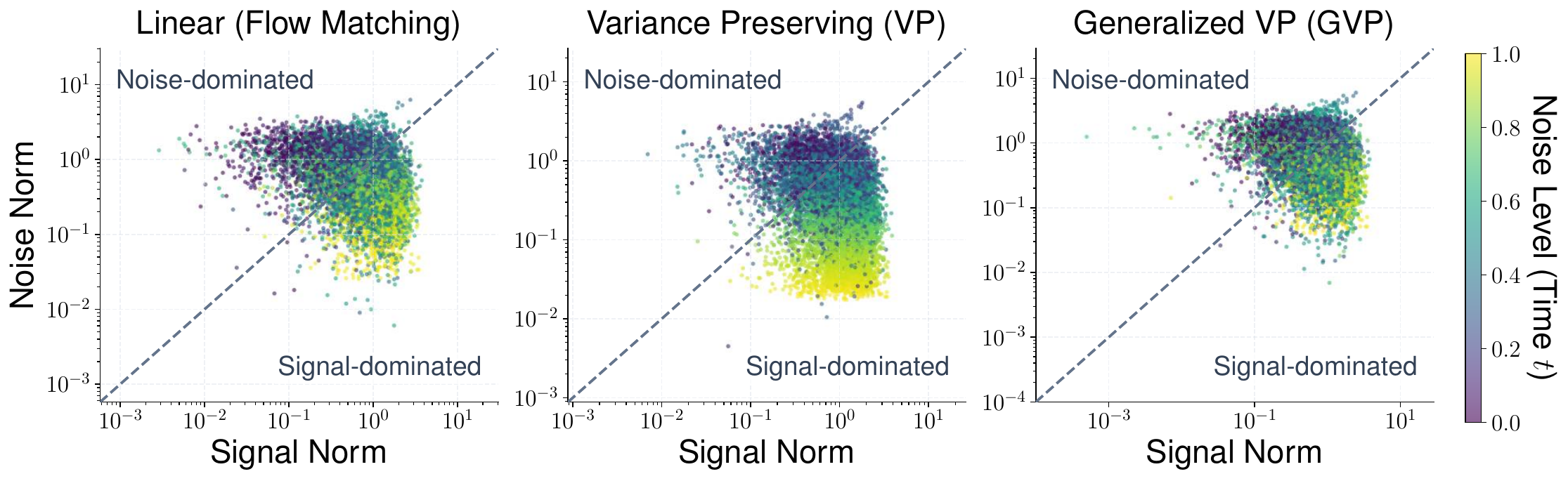}
    % \vspace{-1.5em}
    \caption{$\epsilon_\theta$-prediction phase space}
    \label{fig:gradient_decomposition_eps}
\end{subfigure}
\hfill
\begin{subfigure}{0.49\textwidth}
    \centering
    \includegraphics[width=\linewidth]{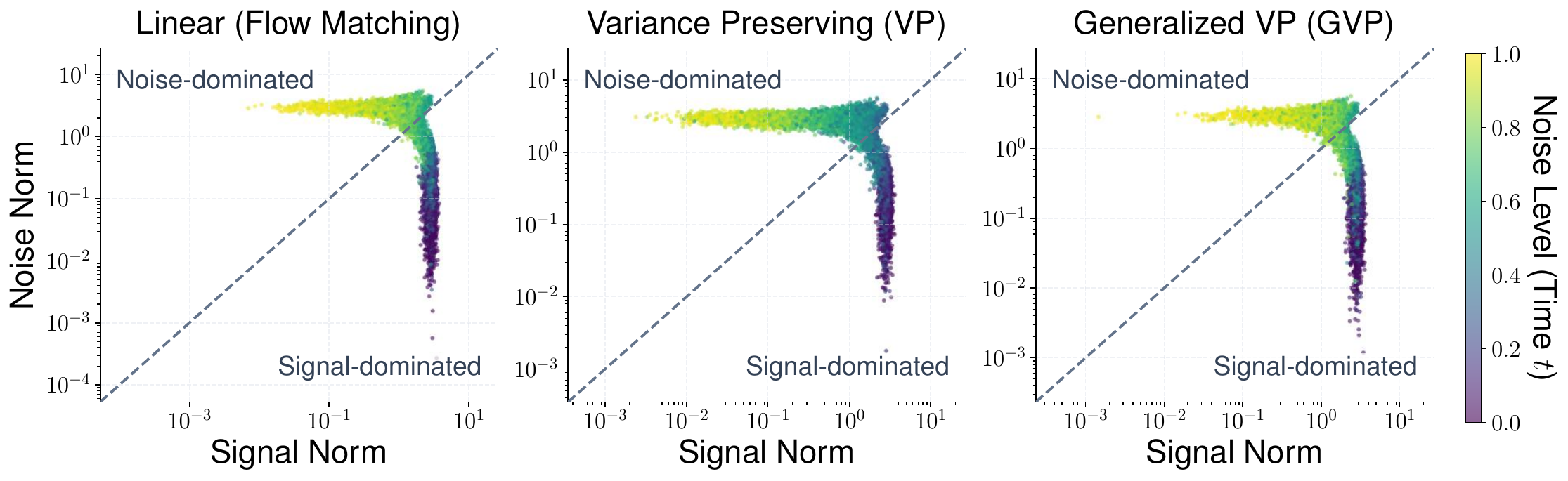}
    % \vspace{-1.5em}
    \caption{$x_\theta$-prediction phase space}
    \label{fig:gradient_decomposition_x0}
\end{subfigure}

\vspace{0.5em}

% Row 2: Left - velocity-prediction, Right - vector-prediction
\begin{subfigure}{0.49\textwidth}
    \centering
    \includegraphics[width=\linewidth]{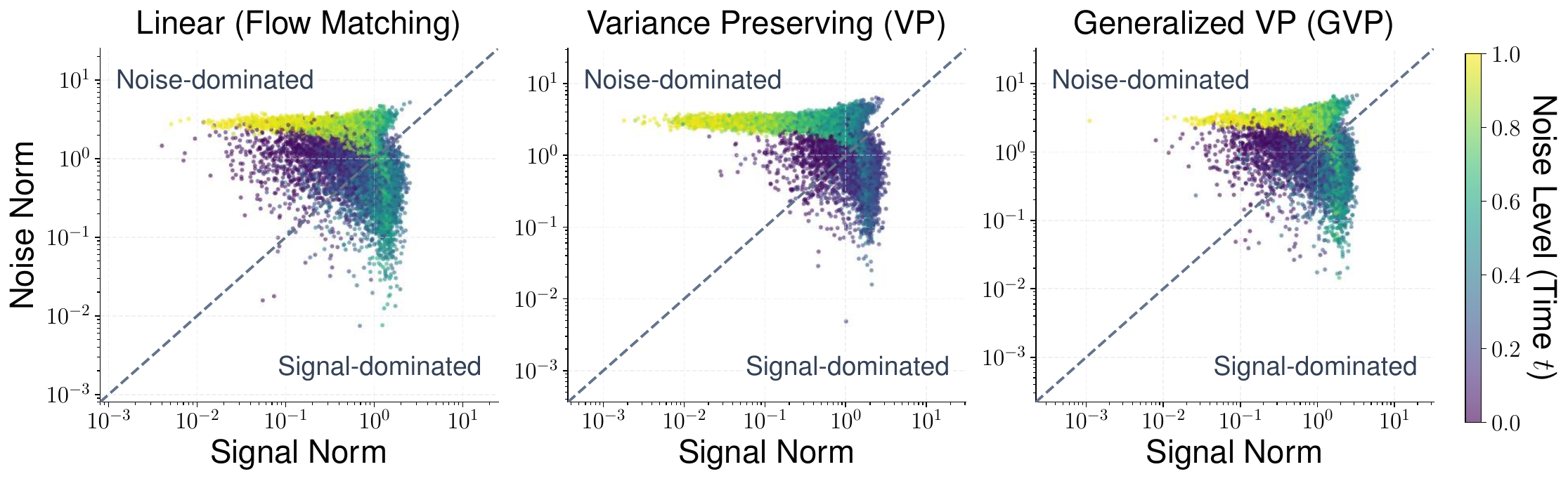}
    % \vspace{-1.5em}
    \caption{$v_\theta$-prediction phase space}
    \label{fig:gradient_decomposition_v}
\end{subfigure}
\hfill
\begin{subfigure}{0.49\textwidth}
    \centering
    \includegraphics[width=\linewidth]{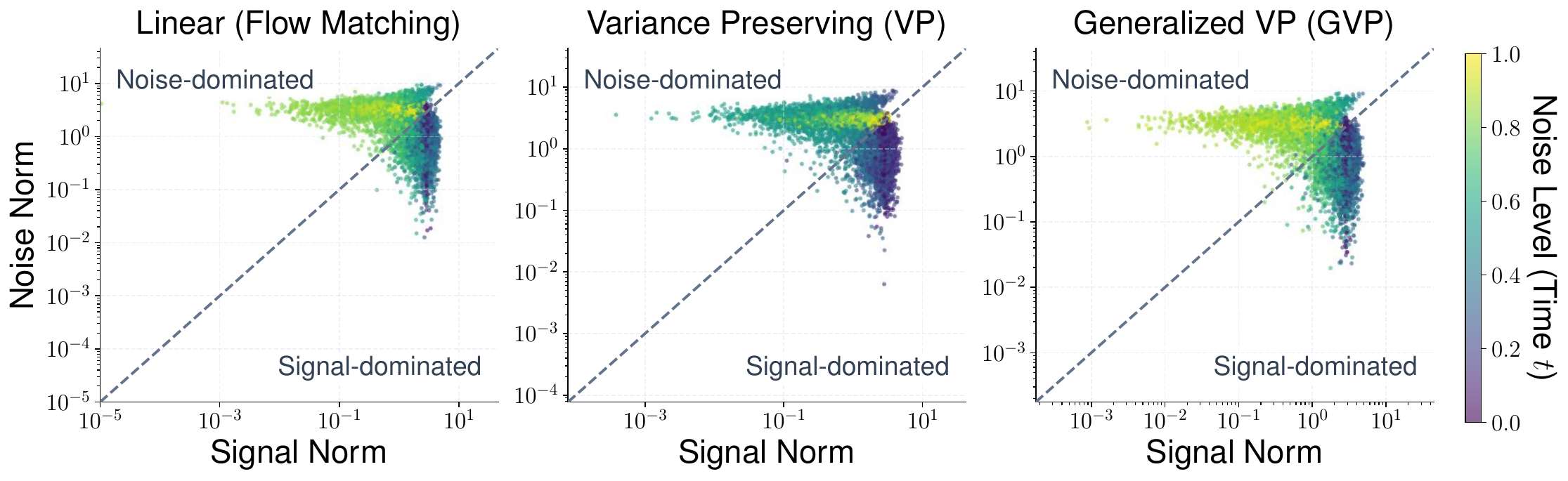}
    % \vspace{-1.5em}
    \caption{$u_\theta$-prediction phase space}
    \label{fig:gradient_decomposition_u}
\end{subfigure}
% \vspace{-0.5em}
\caption{\textbf{Target-Dependent Signal--Noise Phase Spaces.} Across different diffusion schedules, colors indicate diffusion time $t$. The dashed diagonal marks equal signal and Bayes residual norms, separating noise-dominated (above) and signal-dominated (below) regimes.}
\label{fig:gradient_decomposition_all}
% \vspace{-1.5em}
\end{figure}

\subsection{Irreducible Residuals and Excess Risk}
\label{subsec:toy_gradient_contamination}

We first examine the relative magnitudes of Bayes predictions and irreducible residuals, followed by the resulting local excess-risk gaps under shared-parameter training, as characterized in \cref{apdx:sec:elbo_dynamics,apdx:sec:spectral_degradation}.

\paragraph{Signal and Residual Norms.}
To characterize the noise-level dependence of irreducible residuals, we compare the Bayes predictor and irreducible residual norms across prediction targets and diffusion schedules. \Cref{lem:gradient_forcing,prop:quantitative_representation_degradation} establish that the irreducible residual contributes to stochastic gradients, including those of representation parameters. The phase-space visualization therefore provides a diagnostic of the relative magnitudes of the two components across noise levels.

As shown in \cref{fig:gradient_decomposition_all}, their relative magnitudes vary substantially across diffusion times and schedules, with distinct regimes in which irreducible residuals exceed Bayes predictions in magnitude. These noise-dominated regimes highlight the potential significance of irreducible gradient fluctuations: although the residual cannot be eliminated by improving the predictor beyond its Bayes optimum, it continues to contribute to stochastic parameter updates. Through shared parameters, such fluctuations can affect representation learning across noise levels.

\paragraph{Excess Risk across Noise Levels.}
By \cref{lem:elbo_bayes_decomposition}, the local loss decomposes into the irreducible Bayes risk $\ell_\lambda^\star$ and the excess risk $\bar{\ell}_{\tau,m}(\lambda)$. \Cref{thm:elbo_asymptotics} further characterizes the long-time excess risk induced by stochastic gradient fluctuations under shared-parameter training. We compare globally trained models with piecewise independent models to examine the effect of parameter sharing on noise-level-specific optimization.

\begin{figure}[htbp]
\centering
% \vspace{-1.0em}
% Top row: epsilon and start_x predictions
\begin{subfigure}{0.49\textwidth}
    \centering
    \includegraphics[width=\textwidth]{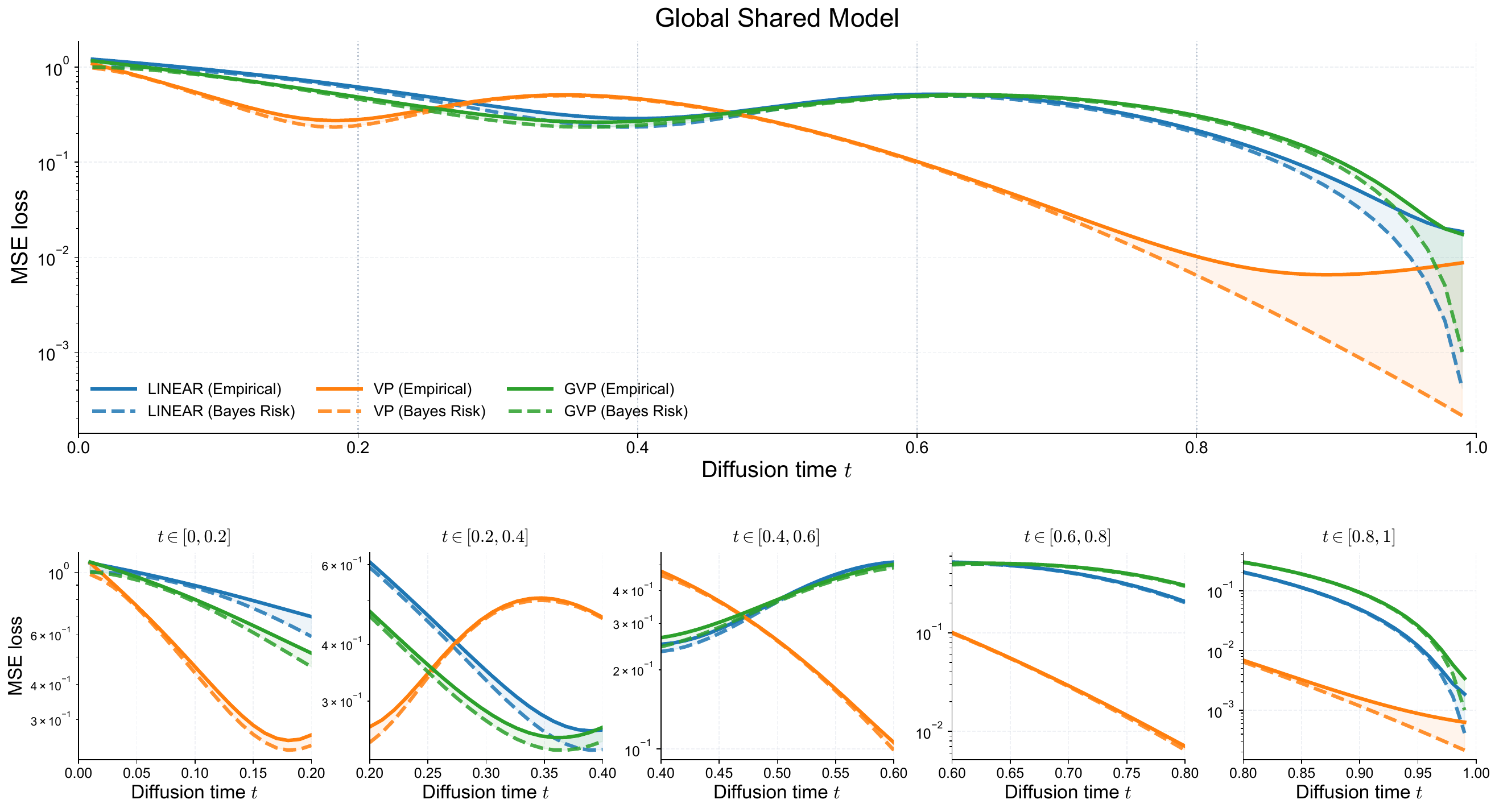}
    \caption{Empirical Loss vs. Bayes Risk ($\epsilon_\theta$)}
    \label{fig:toy_bayes_risk_eps}
\end{subfigure}
\hfill
\begin{subfigure}{0.49\textwidth}
    \centering
    \includegraphics[width=\textwidth]{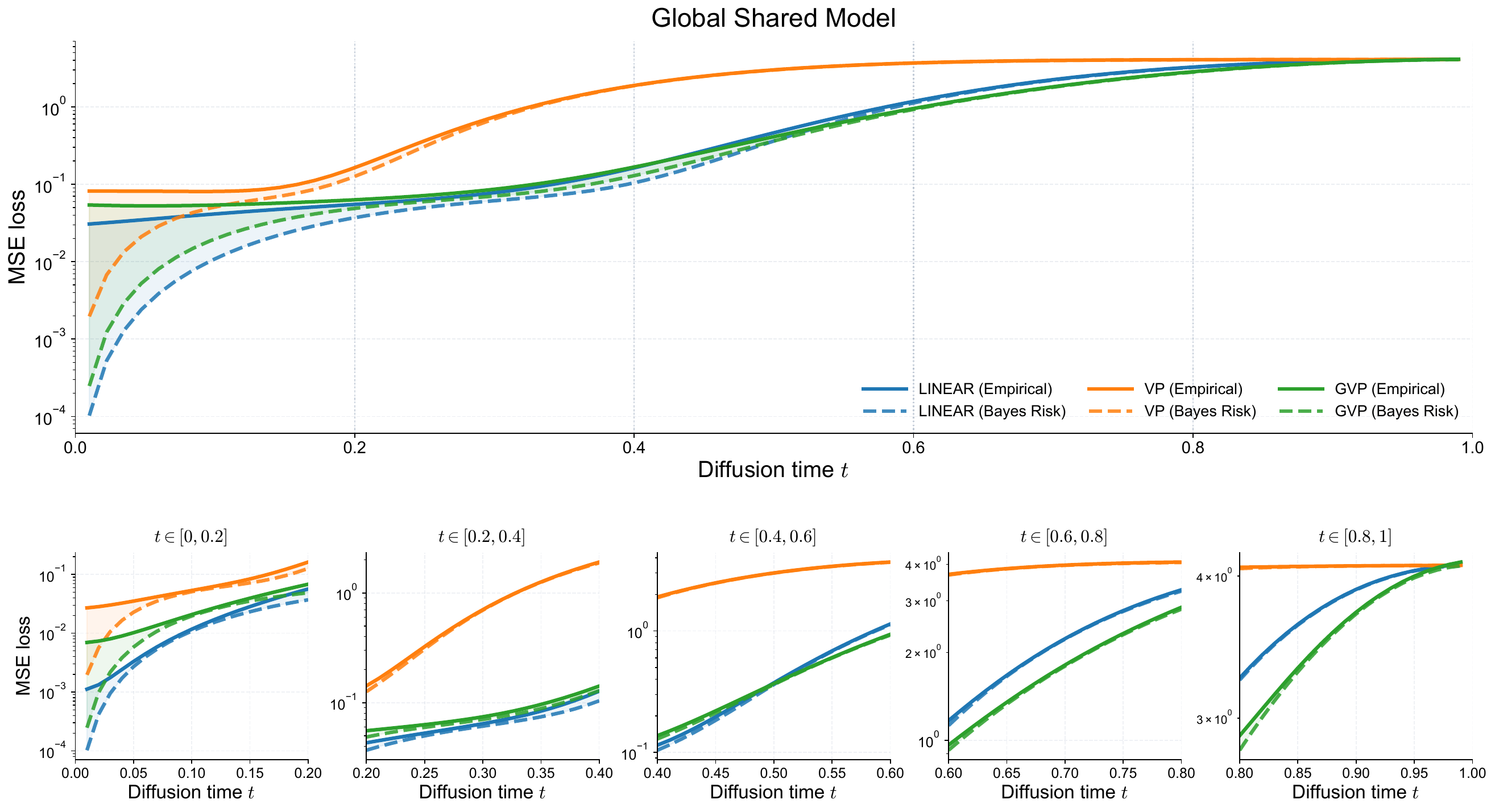}
    \caption{Empirical Loss vs. Bayes Risk ($x_\theta$)}
    \label{fig:toy_bayes_risk_x0}
\end{subfigure}

\vspace{1.0em}

% Bottom row: velocity and vector predictions
\begin{subfigure}{0.49\textwidth}
    \centering
    \includegraphics[width=\textwidth]{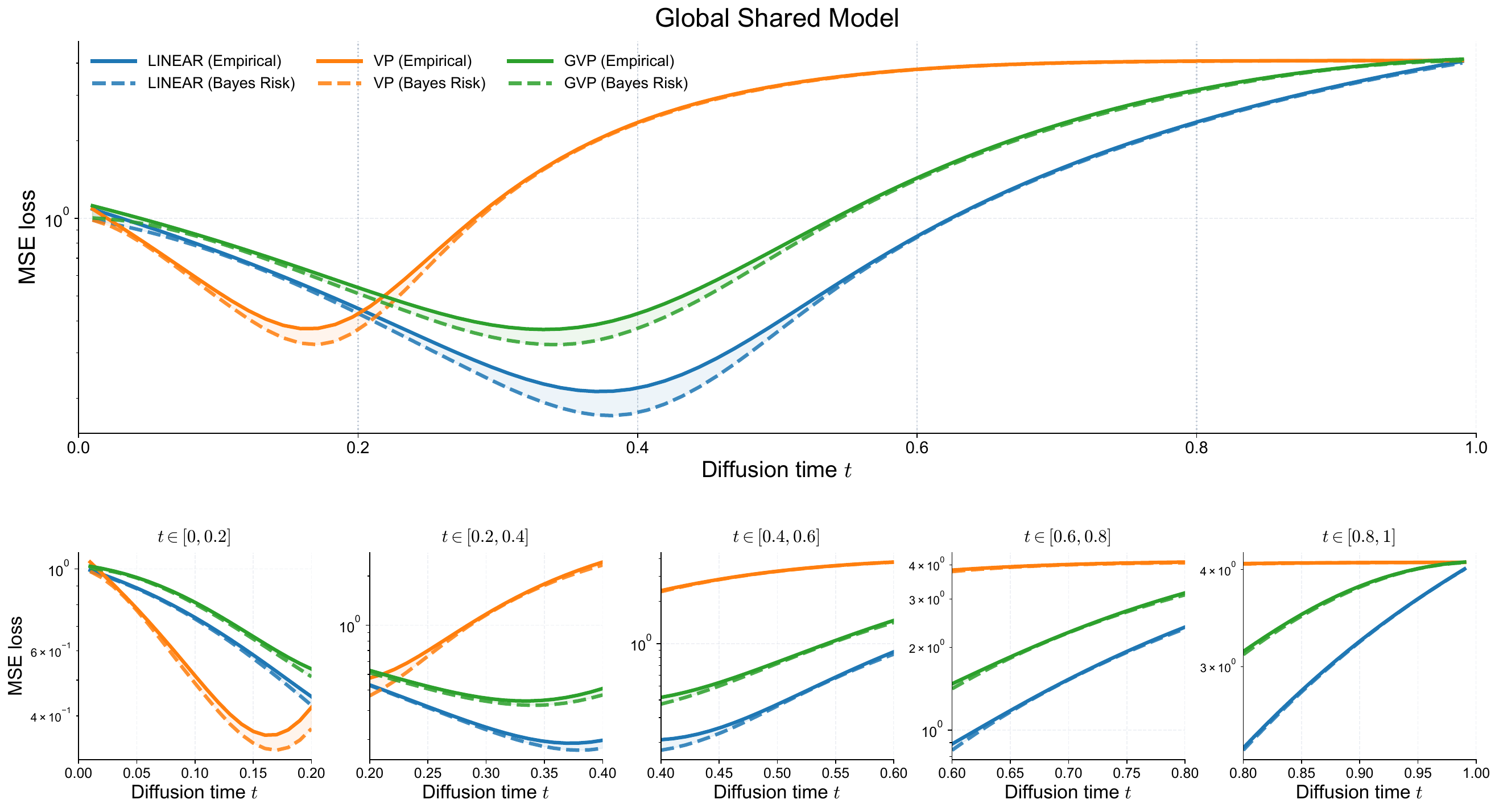}
    \caption{Empirical Loss vs. Bayes Risk ($v_\theta$)}
    \label{fig:toy_bayes_risk_v}
\end{subfigure}
\hfill
\begin{subfigure}{0.49\textwidth}
    \centering
    \includegraphics[width=\textwidth]{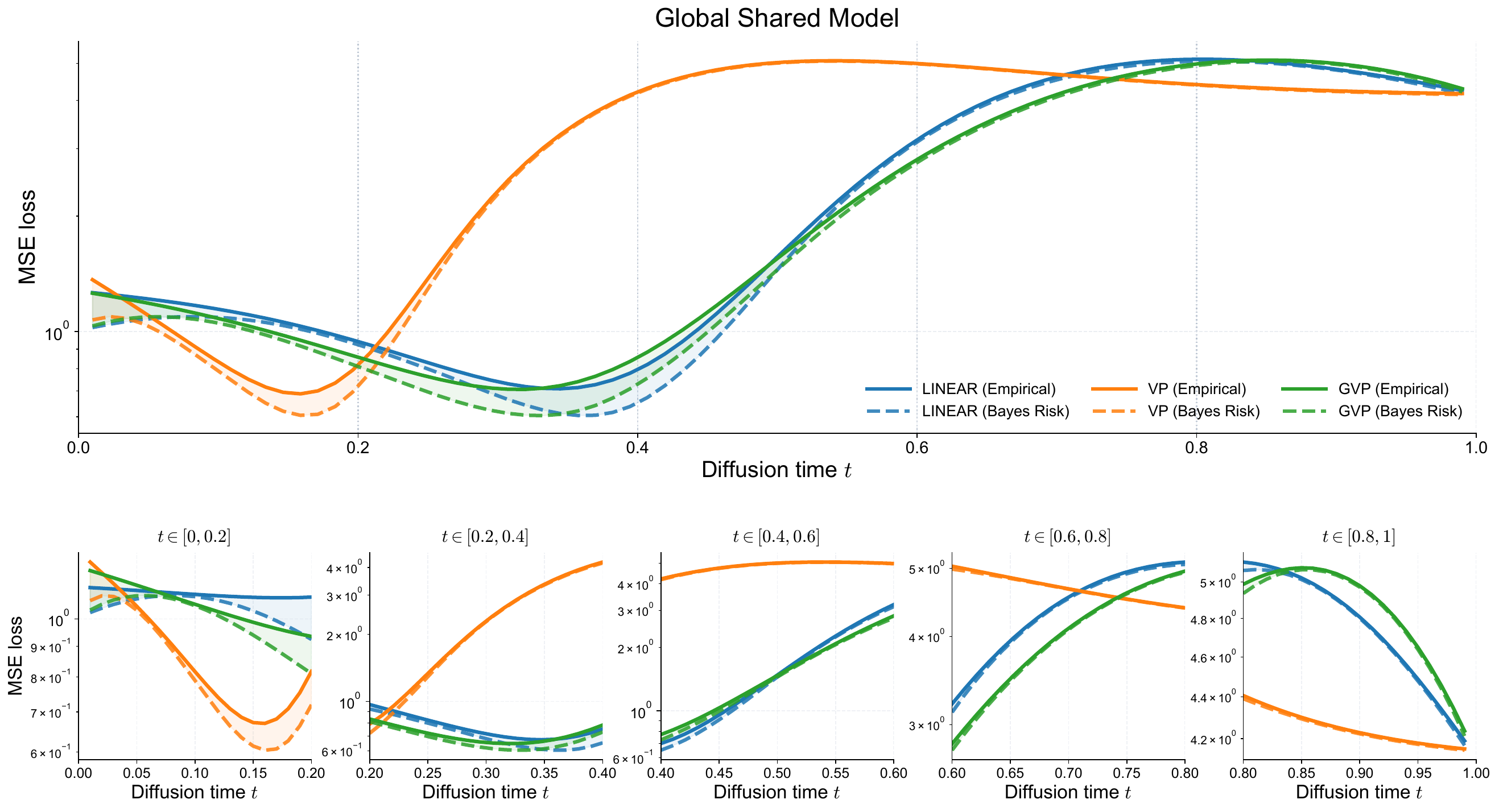}
    \caption{Empirical Loss vs. Bayes Risk ($u_\theta$)}
    \label{fig:toy_bayes_risk_u}
\end{subfigure}
% \vspace{-0.5em}
\caption{\textbf{Empirical Loss and Analytical Bayes Risk across Different Targets.} We evaluate globally trained shared models (top row of each subfigure) and piecewise independent models (bottom row of each subfigure) across the full diffusion time $t \in [0, 1]$.}
\label{fig:toy_bayes_risk}
% \vspace{-1.5em}
\end{figure}

As shown in \cref{fig:toy_bayes_risk}, globally trained models exhibit nonuniform excess-risk gaps across noise levels. In particular, substantial excess risk persists even in regimes with low Bayes risk (\eg, $t \to 1$ for $\epsilon_\theta$), whereas other regimes approach their Bayes-optimal values more closely. Training independent models on restricted time bins reduces these gaps and brings their losses closer to the corresponding Bayes risks. This contrast is consistent with cross-noise-level interference arising from global parameter sharing, as characterized by the long-time excess-risk analysis.

\subsection{Degradation of Semantically Structured Representations}
\label{subsec:toy_representation_collapse}

The four-component GMM provides an initially well-separated representation space in which the evolution of distinct modes can be examined explicitly. \Cref{thm:spectral_local_elbo_bound} characterizes the long-time fluctuations of NTK modes and hidden representations induced by irreducible gradient fluctuations. We therefore investigate how the geometry of learned representations evolves across diffusion noise levels.

We extract the 256-dimensional hidden representations $h_{\theta,m}(x_t, t)$ immediately preceding the readout head at four diffusion times $t \in \{0.1, 0.4, 0.7, 0.9\}$. For a consistent comparison, we fit PCA to representations at $t=0.1$ and project representations from the remaining diffusion times onto the same principal subspace.

\begin{figure}[!ht]
% \vspace{-0.5em}
%\begin{figure}[htbp]
\centering
% ---- Row 1 ----
\begin{adjustbox}{width=1.0\textwidth}
\begin{subfigure}{0.49\textwidth}
\centering
\includegraphics[width=\textwidth]{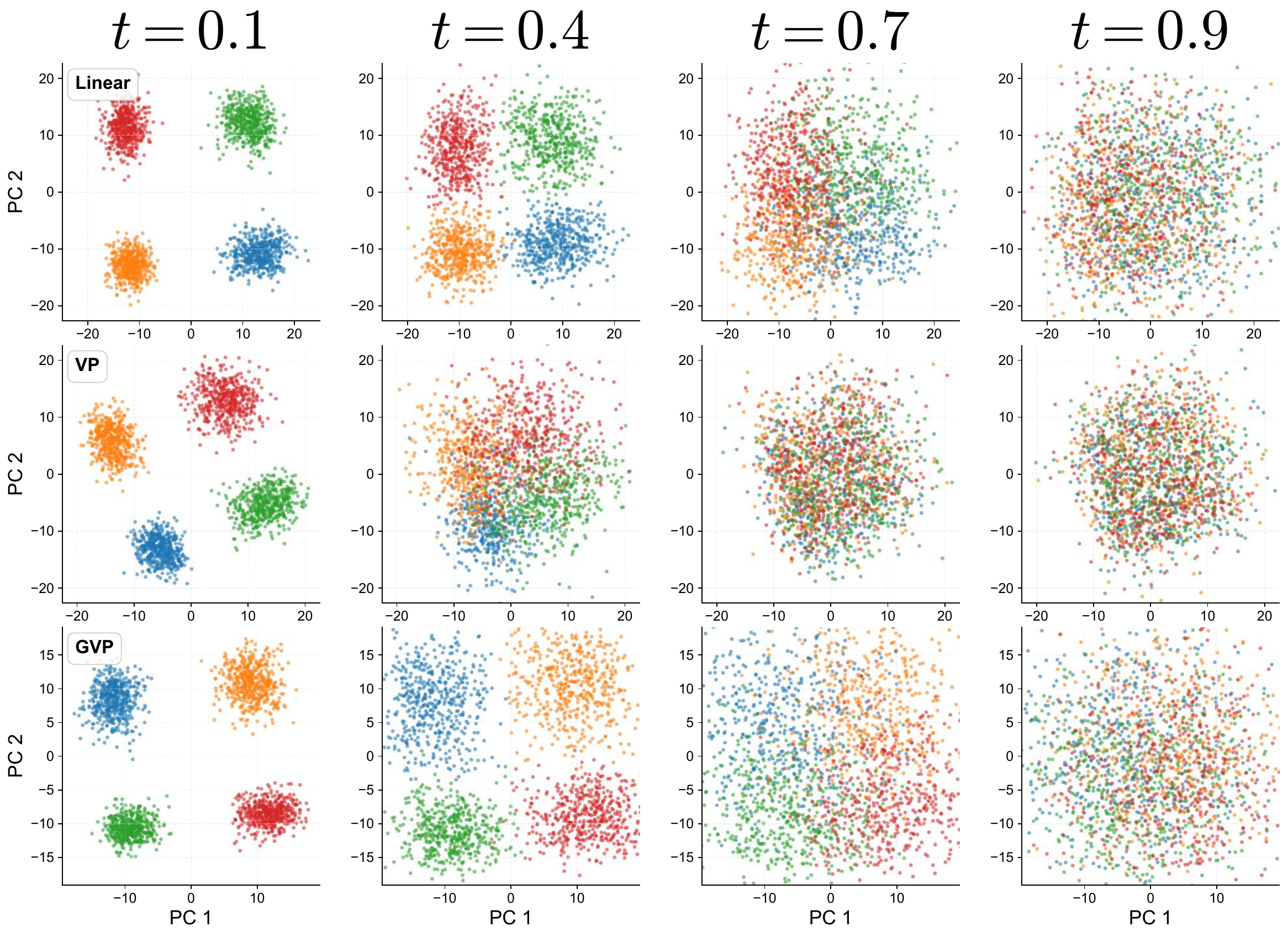}
% \vspace{-1.5em}
\caption{$\epsilon_\theta$-prediction}
\end{subfigure}
\hspace{2.0em}%\hfill
\begin{subfigure}{0.49\textwidth}
\centering
\includegraphics[width=\textwidth]{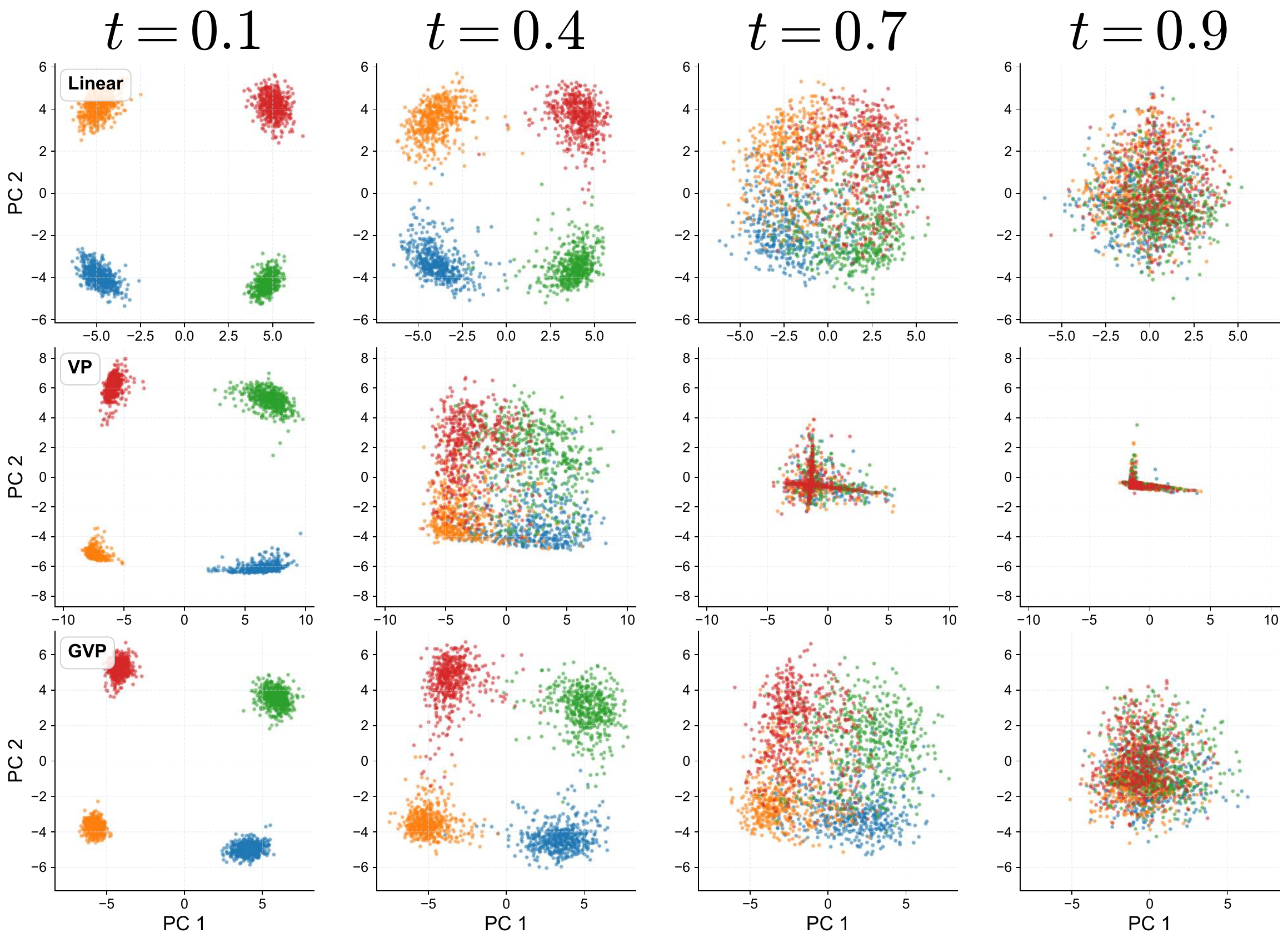}
% \vspace{-1.5em}
\caption{$x_\theta$-prediction}
\end{subfigure}
\end{adjustbox}

\vspace{0.5em}

% ---- Row 2 ----
\begin{adjustbox}{width=1.0\textwidth}
\begin{subfigure}{0.49\textwidth}
\centering
\includegraphics[width=\textwidth]{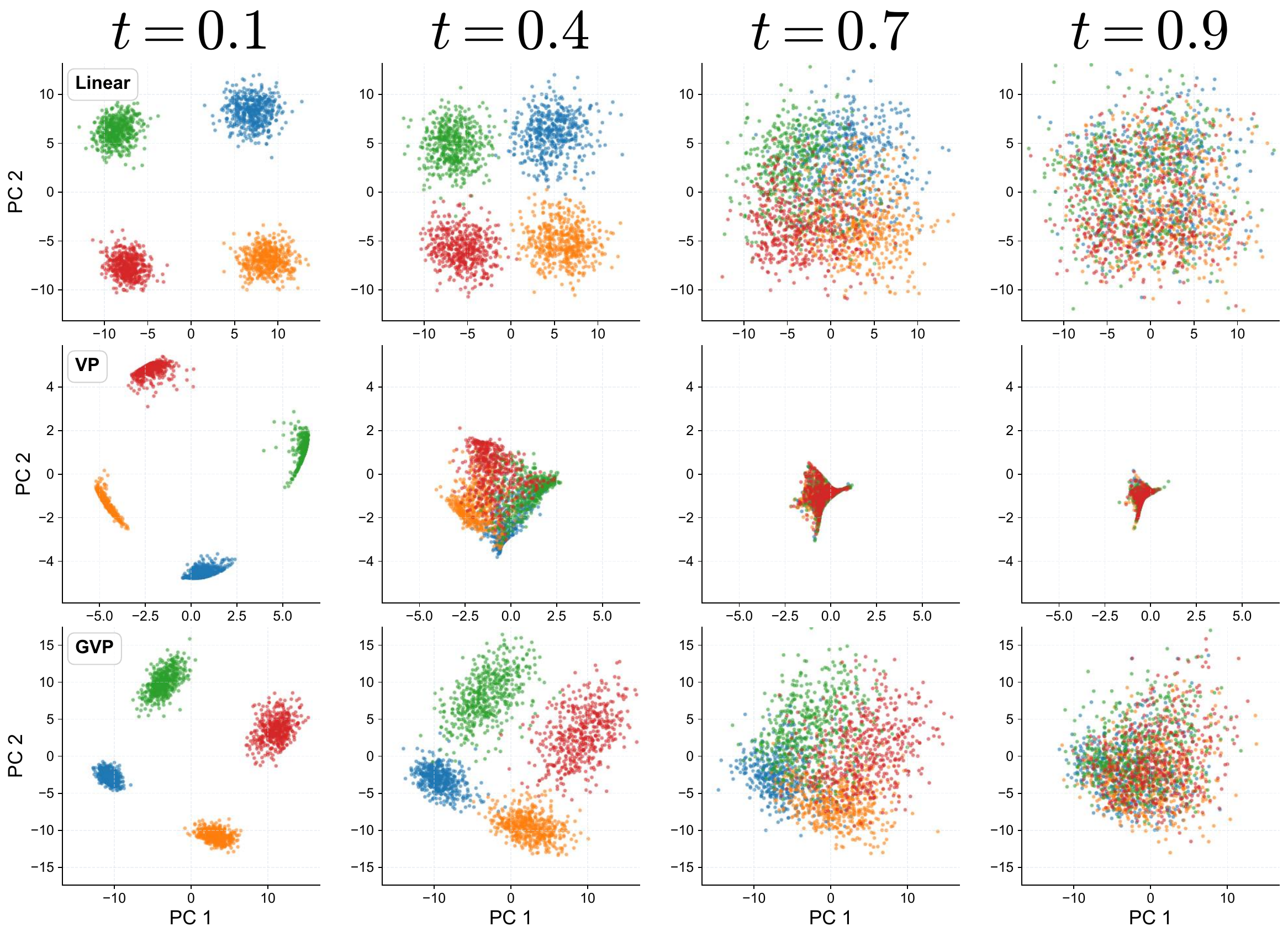}
% \vspace{-1.5em}
\caption{$v_\theta$-prediction}
\end{subfigure}
\hspace{2.0em}%\hfill
\begin{subfigure}{0.49\textwidth}
\centering
\includegraphics[width=\textwidth]{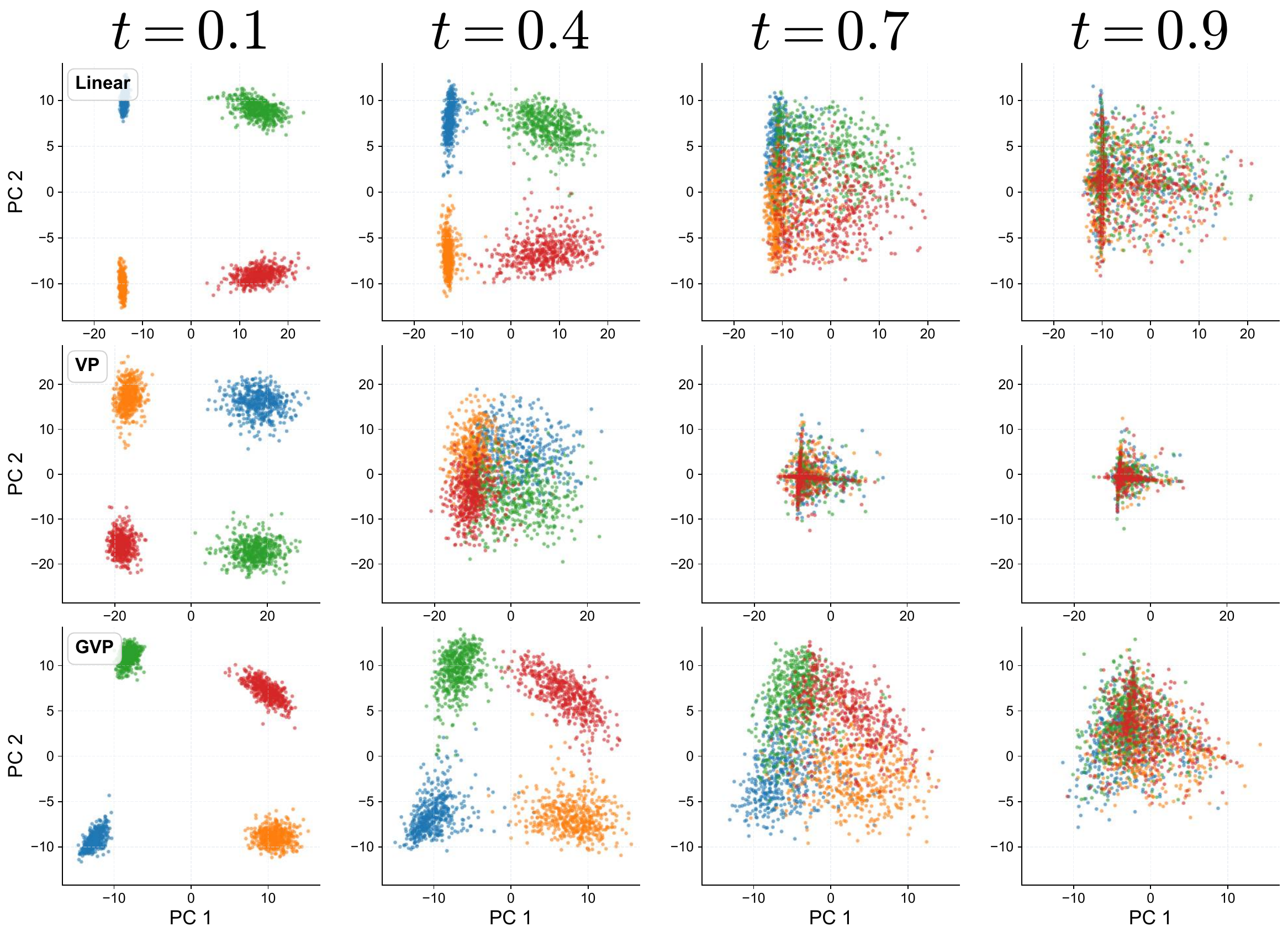}
% \vspace{-1.5em}
\caption{$u_\theta$-prediction}
\end{subfigure}
\end{adjustbox}
\vspace{-1.0em}
\caption{\textbf{Degradation of Semantically Structured Representations.} The deep hidden representations $h_{\theta,m}(x_t, t)$ are tracked across four increasing diffusion times. Projected onto the PCA basis fitted at $t=0.1$, the initially separated clusters exhibit progressively reduced separation at higher diffusion times.}
\label{fig:toy_pca_evolution}
% \vspace{-2.0em}
\end{figure}

As shown in \cref{fig:toy_pca_evolution}, the projected hidden representations exhibit progressively reduced cluster separation as the diffusion noise level increases. Modes that are initially well separated become increasingly entangled in the projected feature space, indicating a reduction in their discriminative structure. This geometric evolution is consistent with the representation-level fluctuations characterized by \cref{thm:spectral_local_elbo_bound}. Since the input distribution is also increasingly corrupted at higher diffusion times, however, the observed changes reflect both input corruption and the learned feature mapping, rather than isolating the effect of gradient fluctuations alone.

\subsection{Neural Tangent Kernel Analysis}
\label{sec:appendix_extended_ntk}

We analyze the function-space geometry underlying these representation changes through empirical NTK analysis across four prediction targets ($\epsilon$, $x_0$, $v$, and $u$). The relations in \cref{lem:ntk_mode_accessibility} connect fixed-noise NTK modes with hidden-feature spectra, while \cref{thm:spectral_local_elbo_bound} characterizes their long-time response to irreducible gradient fluctuations. We examine the eigenvalue spectra, effective ranks, and Gram matrix structures to characterize how the accessible function-space modes and pairwise similarities vary across noise levels.

% ==========================================
% 1. Spectrum (Text + Figure)
% ==========================================
\paragraph{NTK Eigenvalue Evolution.}
To characterize the noise-level dependence of the NTK spectrum, we compute its top eigenvalues ($\kappa_1,\kappa_2,\kappa_3$) across normalized diffusion time. As shown in \cref{fig:appendix_ntk_spectrum}, the spectra exhibit pronounced target- and schedule-dependent variation. For $x_\theta$ and $u_\theta$, the leading eigenvalues generally decrease at higher noise levels, whereas $\epsilon_\theta$ exhibits U-shaped trajectories under Linear and VP schedules but increasing eigenvalues under GVP. By \cref{lem:ntk_mode_accessibility}, these eigenvalues characterize the accessibility of corresponding NTK modes to parameter perturbations. \Cref{thm:spectral_local_elbo_bound} further bounds their stationary mean-square response to irreducible gradient fluctuations in terms of $\kappa_a^\lambda$.

\begin{figure}[htbp]
% \vspace{-0.5em}
\centering
\begin{subfigure}{0.49\textwidth}
    \centering
    \includegraphics[width=\textwidth]{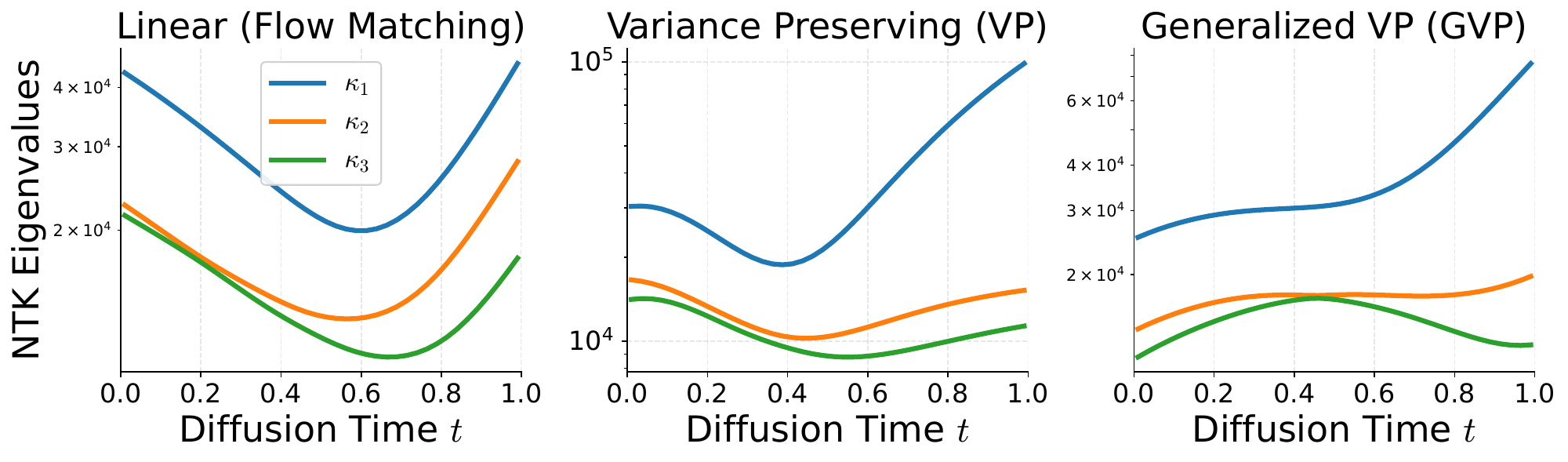}
    %\vspace{-1.5em}
    \caption{$\epsilon_\theta$-prediction}
\end{subfigure}\hfill
\begin{subfigure}{0.49\textwidth}
    \centering
    \includegraphics[width=\textwidth]{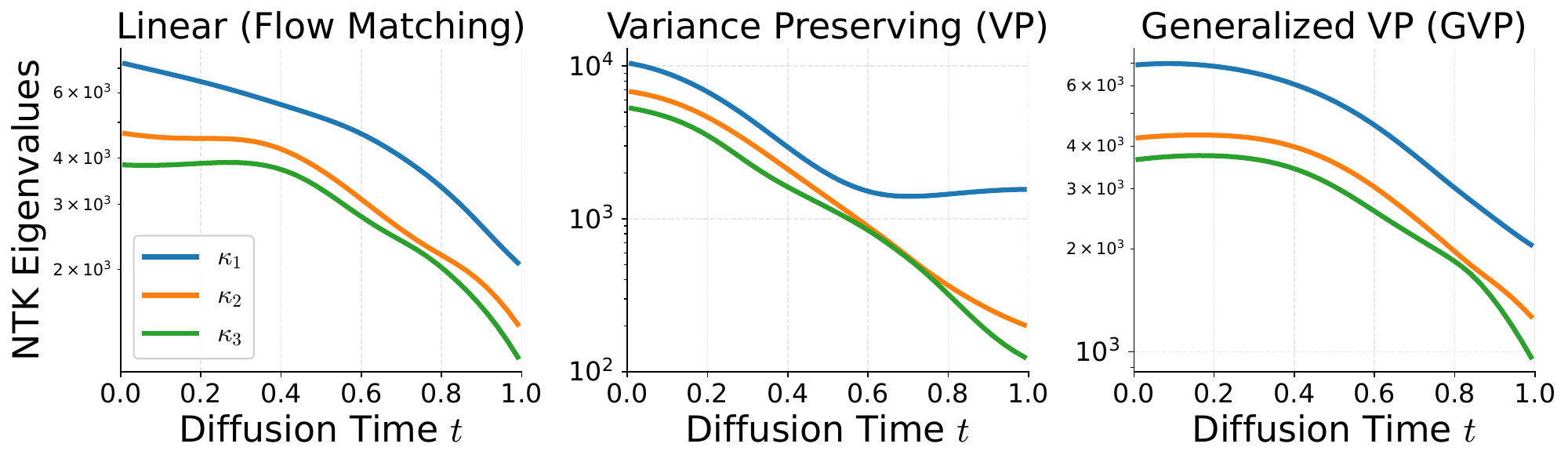}
    %\vspace{-1.5em}
    \caption{$x_\theta$-prediction}
\end{subfigure}

\vspace{0.5em}

\begin{subfigure}{0.49\textwidth}
    \centering
    \includegraphics[width=\textwidth]{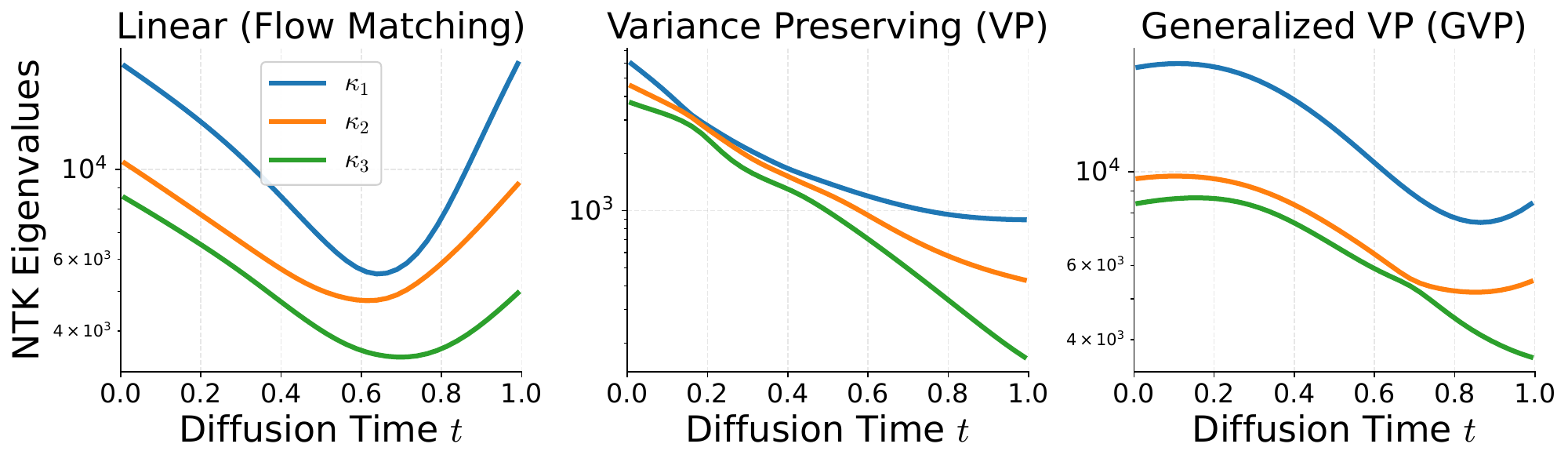}
    %\vspace{-1.5em}
    \caption{$v_\theta$-prediction}
\end{subfigure}\hfill
\begin{subfigure}{0.49\textwidth}
    \centering
    \includegraphics[width=\textwidth]{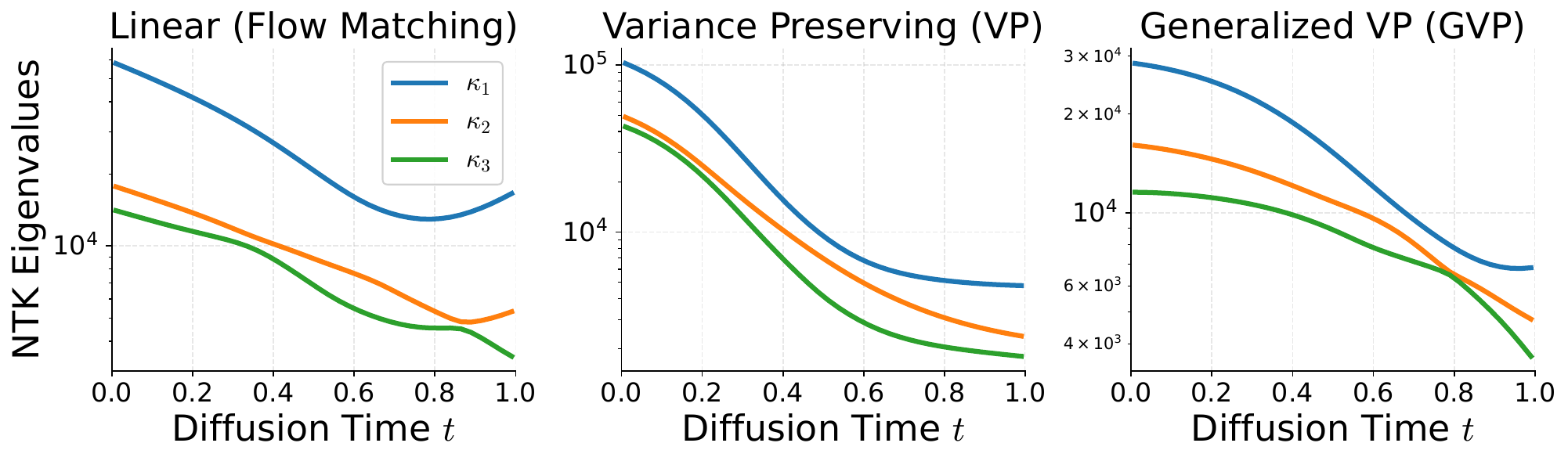}
    %\vspace{-1.5em}
    \caption{$u_\theta$-prediction}
\end{subfigure}
% \vspace{-0.5em}
\caption{\textbf{NTK Eigenvalue Evolution.} Top-3 empirical NTK eigenvalues across diffusion time for four prediction targets, illustrating target- and schedule-dependent spectral variation associated with the NTK modes studied in \cref{lem:ntk_mode_accessibility,thm:spectral_local_elbo_bound}.}
\label{fig:appendix_ntk_spectrum}
\vspace{-0.5em}
\end{figure}

% ==========================================
% 2. Rank (Text + Figure)
% ==========================================
\paragraph{Effective Rank Variation.} To characterize the concentration of the empirical NTK spectrum, we measure its Shannon entropy-based effective rank~\cite{roy2007effective}. As shown in \cref{fig:appendix_ntk_rank}, the effective rank varies nonmonotonically across diffusion time, with pronounced declines at higher noise levels for several targets and schedules. However, some configurations, particularly $x_\theta$ under Linear and GVP schedules, exhibit increasing effective ranks. These variations reveal target- and schedule-dependent changes in spectral concentration, complementing the mode-accessibility relation in \cref{lem:ntk_mode_accessibility} and the stationary response bounds in \cref{thm:spectral_local_elbo_bound}.

\begin{figure}[htbp]
% \vspace{-0.5em}
\centering
\begin{subfigure}{0.49\textwidth}
    \centering
    \includegraphics[width=\textwidth]{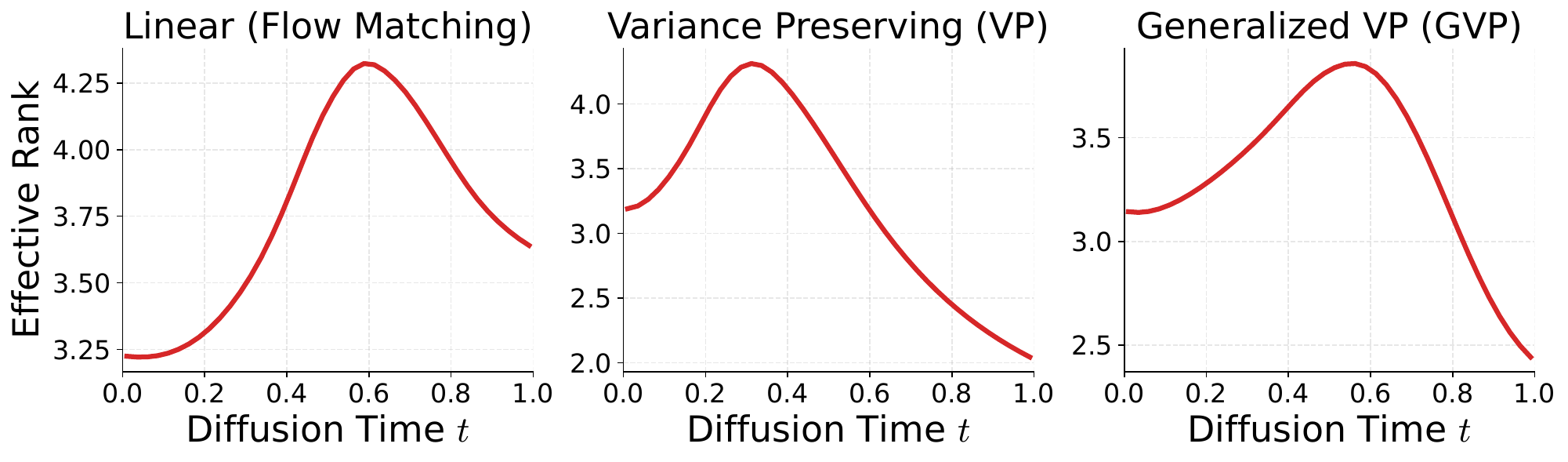}
    %\vspace{-1.5em}
    \caption{$\epsilon_\theta$-prediction}
\end{subfigure}\hfill
\begin{subfigure}{0.49\textwidth}
    \centering
    \includegraphics[width=\textwidth]{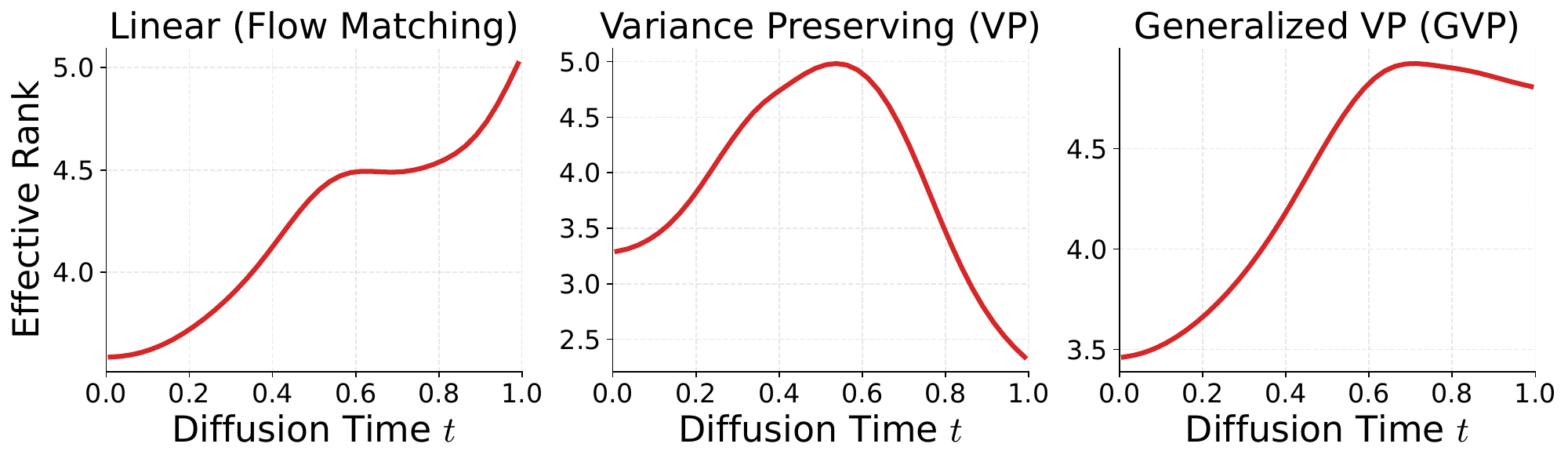}
    %\vspace{-1.5em}
    \caption{$x_\theta$-prediction}
\end{subfigure}

%\vspace{0.5em}

\begin{subfigure}{0.49\textwidth}
    \centering
    \includegraphics[width=\textwidth]{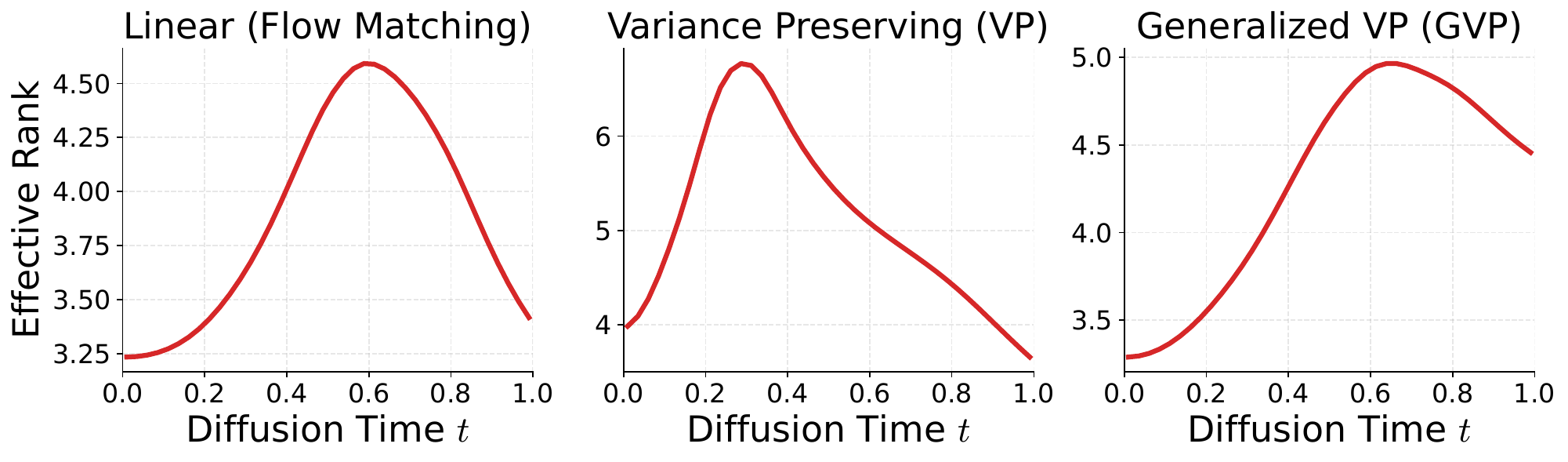}
    %\vspace{-1.5em}
    \caption{$v_\theta$-prediction}
\end{subfigure}\hfill
\begin{subfigure}{0.49\textwidth}
    \centering
    \includegraphics[width=\textwidth]{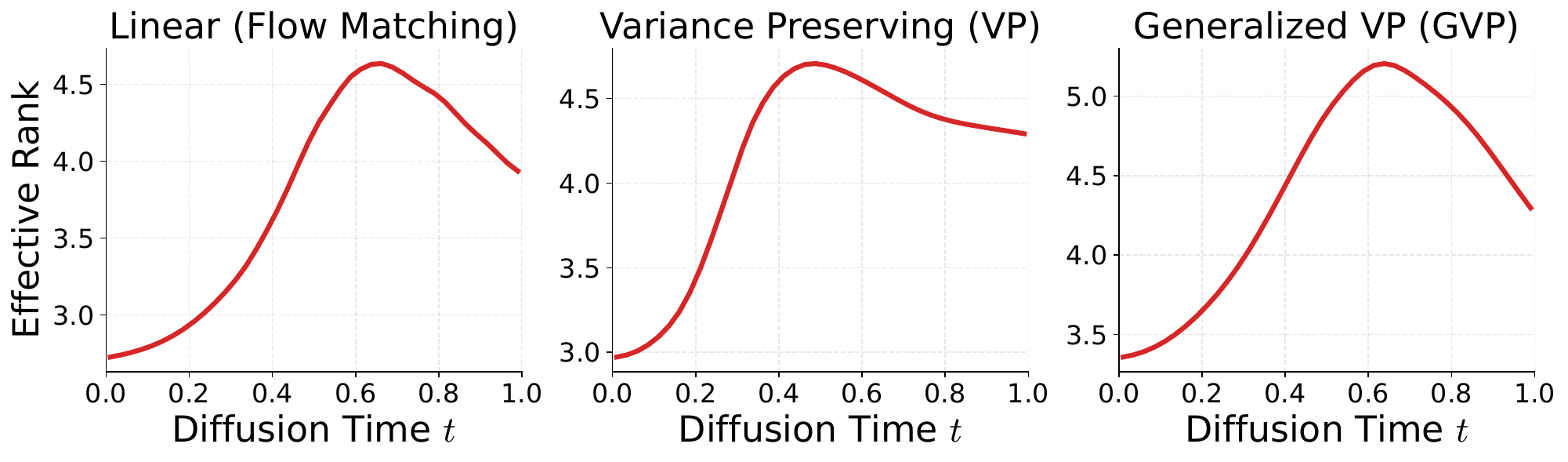}
    %\vspace{-1.5em}
    \caption{$u_\theta$-prediction}
\end{subfigure}
% \vspace{-0.5em}
% \caption{\textbf{Effective Rank Collapse.} Shannon entropy-based effective rank of the empirical NTK across diffusion time and prediction targets.}
% \label{fig:appendix_ntk_rank}
\caption{\textbf{Effective Rank Variation.} Shannon entropy-based effective rank of the NTK across diffusion time and prediction targets, showing target- and schedule-dependent spectral concentration.}
\label{fig:appendix_ntk_rank}
\vspace{-0.5em}
\end{figure}

% ==========================================
% 3. Heatmaps (Text + Figure)
% ==========================================
\paragraph{NTK Gram Matrix Evolution.}
\Cref{fig:appendix_ntk_heatmaps} displays normalized empirical NTK Gram matrices across noise levels, providing a pairwise view of the function-space geometry. At low diffusion times ($t=0.05$), the matrices exhibit pronounced cluster-dependent block structures, with greater similarity among inputs from the same cluster. At high diffusion times ($t=0.95$), these structures become less distinct, with target- and schedule-dependent changes in similarity patterns. This evolution provides a complementary function-space perspective on the changes in hidden-feature geometry observed above.

\begin{figure}[htbp]
% \vspace{-0.5em}
\centering
\begin{subfigure}{\textwidth}
    \centering
    \includegraphics[width=1.0\textwidth]{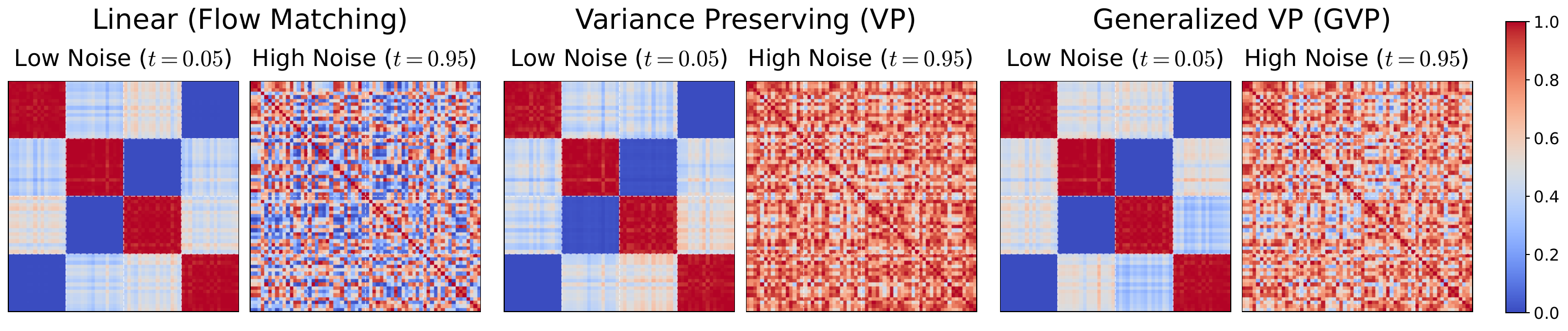}
    % \vspace{-0.5em}
    \caption{$\epsilon_\theta$-prediction formulation.}
\end{subfigure}

\begin{subfigure}{\textwidth}
    \centering
    \includegraphics[width=1.0\textwidth]{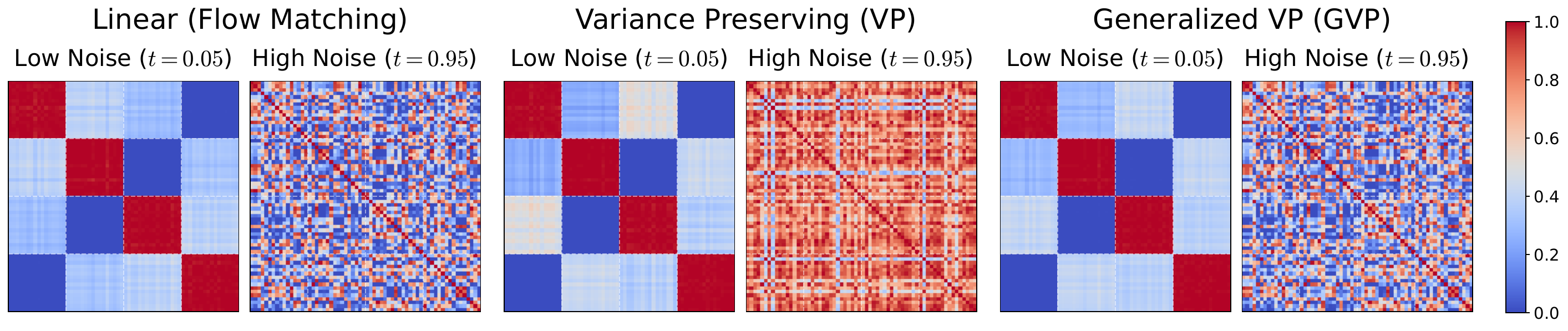}
    % \vspace{-0.5em}
    \caption{$x_\theta$-prediction formulation.}
\end{subfigure}

\begin{subfigure}{\textwidth}
    \centering
    \includegraphics[width=1.0\textwidth]{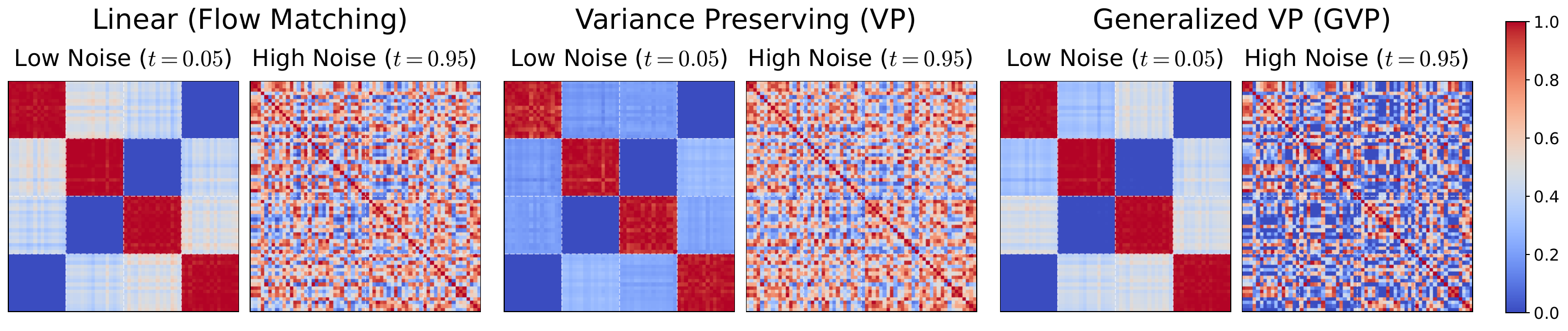}
    % \vspace{-0.5em}
    \caption{$v_\theta$-prediction formulation.}
\end{subfigure}

\begin{subfigure}{\textwidth}
    \centering
    \includegraphics[width=1.0\textwidth]{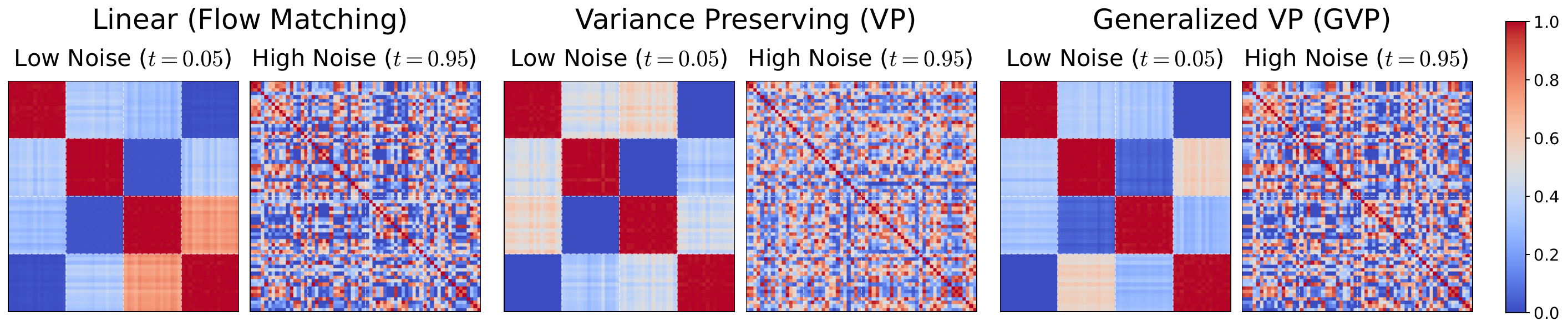}
    % \vspace{-0.5em}
    \caption{$u_\theta$-prediction formulation.}
\end{subfigure}
% \vspace{-1.5em}
\caption{\textbf{NTK Gram Matrix Evolution.} Normalized empirical NTK Gram matrices at low and high diffusion times for four prediction targets, showing changes in cluster-dependent similarity structures.}
\label{fig:appendix_ntk_heatmaps}
% \vspace{-1.0em}
\end{figure}
\section{Experimental Supplement}
\label{app:sec:exp_sup}

\subsection{Implementation Details.}
Table~\ref{tab:hyperparam} summarizes the training and architecture hyperparameters, where \emph{Convolution} indicates whether a 3$\times$3 convolutional block is added before the output.

\textbf{Basic setting.}
Unless otherwise specified, we adopt the AdamW optimizer~\cite{kingma2014adam,loshchilov2017decoupled} with a constant learning rate of $1 \times 10^{-4}$ and no weight decay. ImageNet experiments were conducted with 8 NVIDIA A100 GPUs (80GB) and a global batch size of 256. To accelerate training, we employ mixed-precision training (fp16).

For ImageNet~\cite{Deng2009}, we pre-compute 32$\times$32$\times$4 latent representations using the Stable Diffusion VAE~\cite{rombach2022high} without data augmentation, consistent with EDM2~\cite{karras2024analyzing}, and use the \texttt{stabilityai/sd-vae-ft-ema} decoder. For CelebA~\cite{liu2015deep}, we follow the preprocessing protocol of SBDM~\cite{song2020score}, center-cropping aligned images to 140$\times$140 and resizing them to 64$\times$64 for unconditional generation in pixel space. For CIFAR-10~\cite{krizhevsky2009learning}, we use the original 32$\times$32 images without additional preprocessing.

\textbf{Diffusion setting.}
Unless otherwise specified, we adopt the discrete-time DDPM formulation~\cite{ho2020denoising} with $T=1000$ diffusion steps and the cosine noise schedule following IDDPM~\cite{nichol2021improved}. To enable classifier-free guidance (CFG)~\cite{ho2022classifier}, we randomly replace class conditioning labels with a learned null token with probability $p_{\mathrm{uncond}} = 0.1$ during class-conditional training. Following standard practice, we maintain an Exponential Moving Average (EMA) of the model weights with a decay rate of $0.9999$ for image generation evaluations. For sampling, we use the deterministic EDM Heun sampler~\cite{karras2022elucidating} with $N=50$ integration steps.

\textbf{Backbone architecture.}
To evaluate the architectural generality of our method, we primarily adopt DiT~\cite{peebles2023scalable} and U-ViT~\cite{bao2023all}, and additionally consider UNet~\cite{ronneberger2015u} under the ADM framework~\cite{dhariwal2021diffusion}. Compared with DiT's homogeneous sequence processing, U-ViT adopts a symmetric encoder--decoder architecture with a middle block and long-range skip connections between corresponding encoder and decoder layers. Following the original design, hierarchical encoder features are concatenated with corresponding decoder features and projected through a dense linear layer.

\subsection{Evaluation Details}
\label{appen:eval}

We adopt the same evaluation protocol as ADM~\cite{dhariwal2021diffusion}, and use the identical reference batches provided in their official implementation.\footnote{\url{https://github.com/openai/guided-diffusion/tree/main/evaluations}} Unless otherwise specified, generation metrics are evaluated on 50,000 samples from EMA models, \ie, FID-50K for FID evaluation. For efficiency, we enable mixed-precision (fp16) during sample generation, and empirically observe no noticeable degradation compared to full fp32 precision. Below, we briefly describe the evaluation metrics used throughout our experiments.

\begin{itemize}[leftmargin=0.2in]
\item \textit{Fréchet Inception Distance (FID)}~\cite{heusel2017gans} quantifies the discrepancy between real and generated image distributions in a learned feature space. The features are extracted using an Inception-v3 network~\cite{szegedy2016rethinking}, and the distance is computed by approximating both feature distributions as multivariate Gaussians.

\item \textit{Spatial FID (sFID)}~\cite{nash2021generating} extends FID by leveraging intermediate spatial feature maps of Inception-v3, allowing it to better reflect spatial consistency in generated images.

\item \textit{Inception Score (IS)}~\cite{salimans2016improved} evaluates sample quality and diversity using the predicted class probabilities from Inception-v3. It computes the exponential of the expected Kullback--Leibler divergence of the conditional label distribution relative to the marginal label distribution, $\exp(\mathbb{E}_{x}[D_{\mathrm{KL}}(p(y|x)\|p(y))])$.

\item \textit{Precision and Recall}~\cite{kynkaanniemi2019improved} assess the fidelity and coverage of generated samples, respectively. Precision measures the proportion of generated samples within the real data manifold, while recall measures the proportion of real samples covered by the generated data manifold.
\end{itemize}

\begin{table}[!t]
% \vspace{-1.5em}
\centering%\small
\caption{Hyperparameter setup.}
\begin{adjustbox}{width=\textwidth}
\begin{tabular}{lccccccc}
\toprule
\textbf{Related result} & \multicolumn{6}{c}{\cref{tab:wo_cfg}} & \cref{tab:main} \\
\midrule
\textbf{Architecture} \\
Model & DiT-B/2 & DiT-L/2 & DiT-XL/2 & U-ViT-M/2 & U-ViT-L/2 & U-ViT-H/2 & U-ViT-H/2 \\
Input dim. & 32$\times$32$\times$4 & 32$\times$32$\times$4 & 32$\times$32$\times$4 & 32$\times$32$\times$4 & 32$\times$32$\times$4 & 32$\times$32$\times$4 & 32$\times$32$\times$4 \\
Num. layers & 12 & 24 & 28 & 17 & 21 & 29 & 29 \\
Hidden dim. & 768 & 1,024 & 1,152 & 768 & 1,024 & 1,152 & 1,152 \\
Num. heads & 12 & 16 & 16 & 12 & 16 & 16 & 16 \\
Convolution & $\varnothing$ & $\varnothing$ & $\varnothing$ & \ding{51} & \ding{51} & \ding{55} & \ding{55} \\
Params (M) & 130 & 458 & 675 & 131 & 287 & 501 & 501 \\
\midrule
\textbf{Optimization} \\
Training iteration / epoch & 400K & 400K & 400K & 400K & 400K & 400K & 800 \\
Batch size & 256 & 256 & 256 & 256 & 256 & 256 & 256 \\
Optimizer & AdamW & AdamW & AdamW & AdamW & AdamW & AdamW & AdamW \\
Learning rate & 1e-4 & 1e-4 & 1e-4 & 1e-4 & 1e-4 & 1e-4 & 1e-4 \\
$(\beta_1, \beta_2)$ & (0.9, 0.999) & (0.9, 0.999) & (0.9, 0.999) & (0.99, 0.99) & (0.99, 0.99) & (0.99, 0.99) & (0.99, 0.99) \\
\midrule
\textbf{Diffusion settings} \\
Noise schedule & cosine & cosine & cosine & cosine & cosine & cosine & cosine \\
Training objective & $\epsilon$-prediction & $\epsilon$-prediction & $\epsilon_\theta$-prediction & $\epsilon_\theta$-prediction & $\epsilon_\theta$-prediction & $\epsilon_\theta$-prediction & $\epsilon_\theta$-prediction \\
Sampler & EDM Heun & EDM Heun & EDM Heun & EDM Heun & EDM Heun & EDM Heun & EDM Heun \\
Sampling steps & 50 & 50 & 50 & 50 & 50 & 50 & 50 \\
Guidance scale & - & - & - & - & - & - & 1.35 \\
\bottomrule
\end{tabular}
\end{adjustbox}
\label{tab:hyperparam}
% \vspace{-1.0em}
\end{table}

\subsection{Computational Overhead Analysis}

We analyze the computational overhead of RPDM relative to standard diffusion training and external representation alignment methods.

In our experiments, RPDM realizes its optimization allocation through loss weighting. For a batch of $N$ diffusion samples, the training objective is
\begin{equation}
\mathcal{L}_{\mathrm{RPDM}}=\frac{1}{N}\sum_{i=1}^{N}\frac{M(\lambda_i)}{\rho(\lambda_i)|\dot{\lambda}_{t(\lambda_i)}|}\ell_i.
\end{equation}
Under uniform timestep sampling, $\rho(\lambda)|\dot{\lambda}_{t(\lambda)}|$ is constant. The required target coefficients and interpolation factors $\alpha(\lambda)$ and $\sigma(\lambda)$ are already available during diffusion training. RPDM therefore requires no extra network forward pass, feature extraction, Jacobian computation, or backward propagation. Computing the per-sample weights adds $\mathcal{O}(N)$ scalar operations per batch, negligible relative to diffusion training.

We further compare RPDM with REPA~\cite{yu2024representation}, which uses a pretrained visual encoder for representation supervision. For REPA-Offline, we report one-time DINOv2-B/14~\cite{oquab2023dinov2} feature extraction cost on ImageNet-1K. For REPA-Online, we report cumulative feature extraction cost over training iterations. All results in \cref{tab:compute_overhead} assume batch size 256 and a 400K-iteration training budget.

\begin{table}[!ht]
% \vspace{-0.5em}
\centering
\caption{Comparison of additional computational overhead.}
\label{tab:compute_overhead}
\resizebox{\linewidth}{!}{
\begin{tabular}{lccc}
\toprule
Method & Additional Operation & Per Step & Total Extra Compute \\
\midrule
Baseline & None & 0 & 0 \\
RPDM & Scalar loss weighting & $\approx2.0\times10^{-6}$ GFLOPs & $\approx0.8$--$1.0$ GFLOPs \\
REPA (Offline) & DINOv2-B/14 extraction & -- & 59.35 PFLOPs \\
REPA (Online) & DINOv2-B/14 forward & 11.86 TFLOPs & 4.74 EFLOPs \\
\bottomrule
\end{tabular}
}
\end{table}

For fairness, the reported REPA costs include only the DINOv2-B/14 forward pass and exclude data loading, preprocessing, and feature storage. As shown in \cref{tab:compute_overhead}, RPDM introduces virtually no additional computational overhead while remaining lightweight and architecture-agnostic.

\subsection{Ablation Study}

We conduct ablation studies to evaluate the effects of model architecture and training strategy on image generation performance. For training-strategy ablations, we train DiT in latent space on ImageNet for 400K iterations with a batch size of 256. For architecture ablations, we train UNet (following the architecture configurations in~\cite{hang2023efficient}) and DiT in pixel space on CelebA for 500K iterations with a batch size of 128. Both models are optimized using Adam with a cosine noise schedule~\cite{nichol2021improved}. We use betas $(0.9, 0.95)$ for DiT and $(0.9, 0.999)$ for UNet, with a higher learning rate of $10^{-3}$ for DiT. \cref{tab:train_ablation} shows consistent improvements across different training strategies, while \cref{tab:net_ablation} demonstrates the effectiveness of RPDM across architectures.

\begin{table}[htbp]
    \centering
    \begin{minipage}[t]{0.485\textwidth}
        \centering
        \captionof{table}{Training strategy ablation of DiT~\cite{peebles2023scalable} in latent space~\cite{rombach2022high} w/o classifier-free guidance (CFG).}
        \label{tab:train_ablation}

        \begin{tabular*}{\linewidth}{@{}l@{\extracolsep{\fill}}c@{}}
            \toprule
            Training Strategy & FID$\downarrow$ \\
            \midrule
            \textit{Baseline Improvement} & \\
            DiT-S/2~\cite{peebles2023scalable} & 69.35 \\
            \textcolor{gray}{\sout{+ Ours}} & \textcolor{gray}{\sout{62.92}} \\
            \midrule
            \textit{Training Trick} & \\
            DiT-S/2~\cite{peebles2023scalable} & 69.35 \\
            + Adam $\beta_2=0.95$~\cite{ahmadian2023intriguing} & 68.77 \\
            + Cosine schedule~\cite{nichol2021improved} & 65.68 \\
            + RPDM (ours) & 57.54 \\
            \bottomrule
        \end{tabular*}
    \end{minipage}%
    \hfill
    \begin{minipage}[t]{0.485\textwidth}
        \centering
        \captionof{table}{FID results on CelebA 64$\times$64 from architectural ablations centered on UNet and DiT frameworks.}
        \label{tab:net_ablation}

        \begin{tabular*}{\linewidth}{@{}l@{\extracolsep{\fill}}cc@{}}
            \toprule
            Method & \#Params & FID$\downarrow$ \\
            \midrule
            \textit{Ablation on U-Net} & & \\
            DDIM~\cite{song2020denoising} & 79M & 3.26 \\
            Soft Truncation~\cite{song2020denoising} & 62M & 1.90 \\
            \textbf{UNet (RPDM)} & 59M & \textbf{1.53} \\
            \midrule
            \textit{Ablation on Transformer} & & \\
            U-ViT-S~\cite{bao2023all} & 44M & 2.87 \\
            DiT-S & 33M & 3.00 \\
            \textbf{DiT-S (RPDM)} & 33M & \textbf{2.37} \\
            \bottomrule
        \end{tabular*}
    \end{minipage}
\end{table}

\subsection{Further Empirical Analysis}

\paragraph{Comparison across Prediction Targets}
Referring to \cref{tab:train_ablation}, we configured the cosine learning rate scheduler with the Adam optimizer's betas set to $(0.9, 0.95)$, and conducted the training with a batch size of 256 over 400K steps, ensuring robust convergence and optimization. We evaluated different prediction objectives, $\epsilon_{\theta}$, $x_{\theta}$, and $v_{\theta}$~\cite{salimans2022progressive}, by comparing the Baseline, Min-SNR~\cite{hang2023efficient}, P2~\cite{choi2022perception}, and RPDM for training DiT-S on ImageNet. As illustrated in \cref{fig:target_comparison}, RPDM consistently achieves the lowest FID-50K scores across all prediction objectives, demonstrating superior performance over other weighting methods throughout the training iterations.

\begin{figure}[htbp]
    \centering
    \subfloat[$\epsilon_{\theta}$-prediction]{\includegraphics[width=0.33\textwidth]{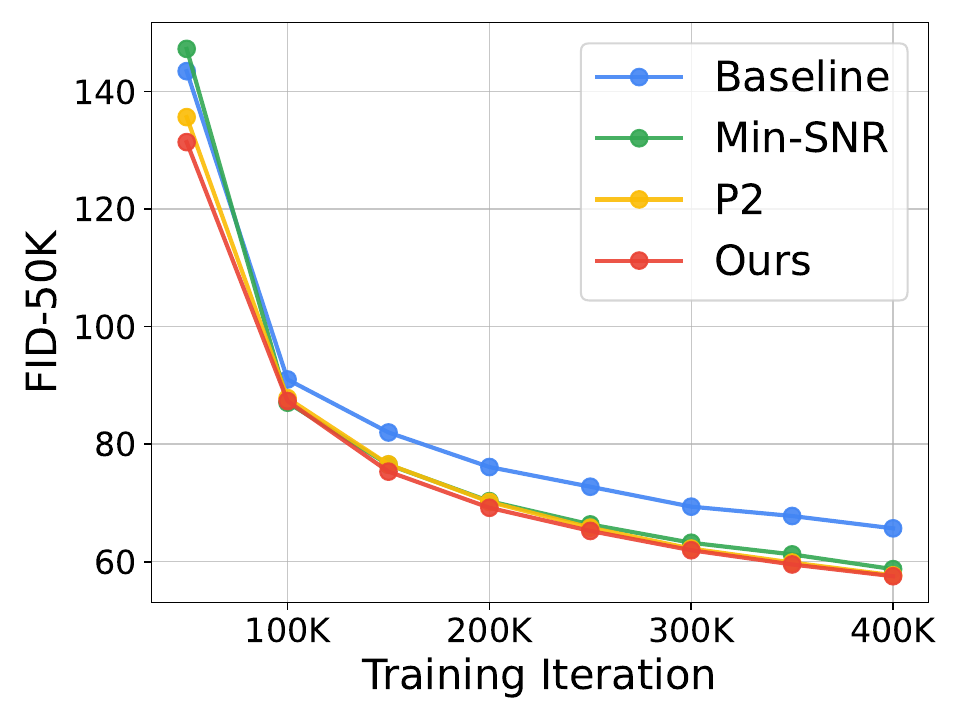}\vspace{-0.5em}} 
    \subfloat[$x_{\theta}$-prediction]{\includegraphics[width=0.33\textwidth]{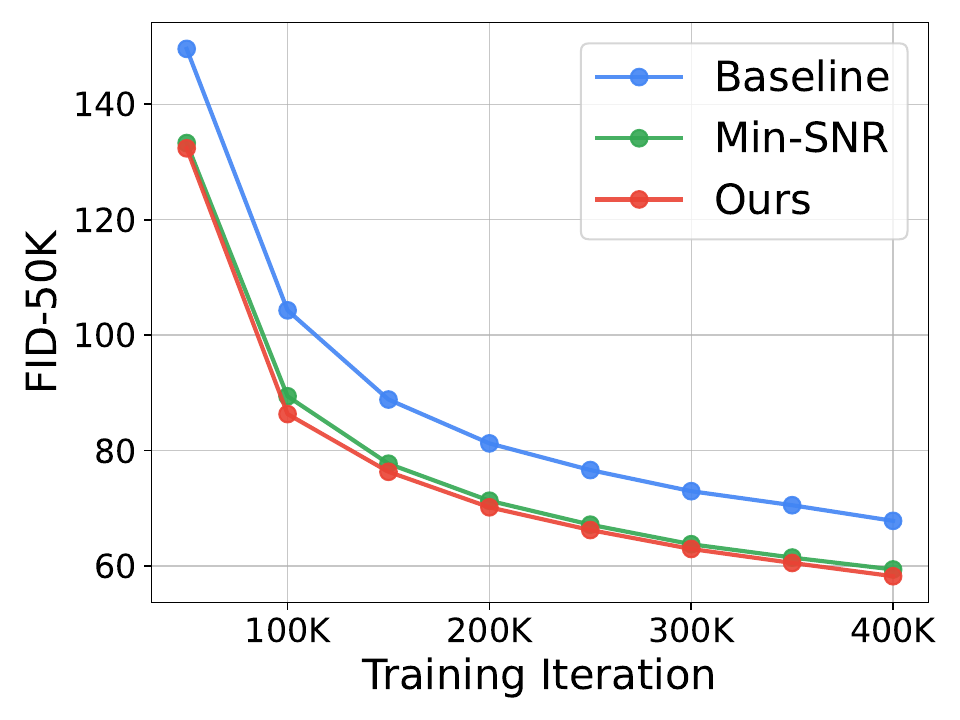}\vspace{-0.5em}}
    \subfloat[$v_{\theta}$-prediction]{\includegraphics[width=0.33\textwidth]{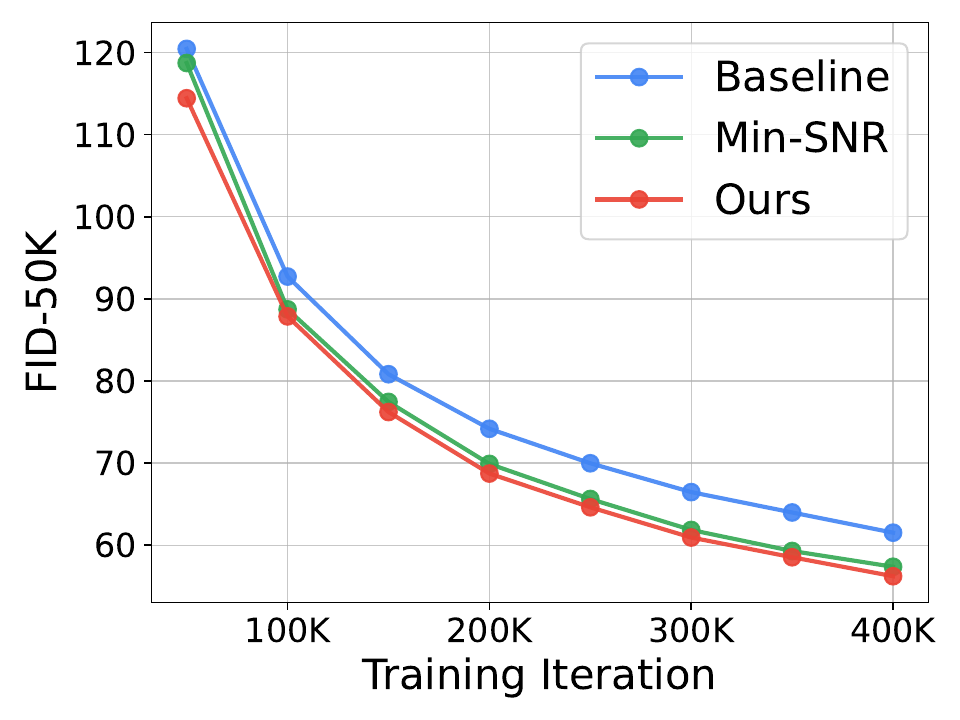}\vspace{-0.5em}}
    % \vspace{-0.5em}
    \caption{Comparing different loss weighting designs by predicting $\epsilon_{\theta}$, $x_{\theta}$ and $v_{\theta}$ on DiT-S.}
    \label{fig:target_comparison}
    % \vspace{-1em}
\end{figure}

\paragraph{Comparison across Model Architectures.}
\label{par:weighting_comparison}

We further compare RPDM with existing noise-dependent loss-weighting methods under the same 100K-iteration setting.

\begin{wraptable}{r}{0.50\textwidth}
\vspace{-1.0em}
\centering
\small
\caption{FID-50K comparison on ImageNet $256\times256$ without CFG after 100K training iterations under identical settings.}
\vspace{-0.5em}
\label{tab:weighting_comparison}
\resizebox{\linewidth}{!}{%
    \begin{tabular}{lcccc}
    \toprule
    Model & Baseline & SiT & Min-SNR & RPDM \\
    \midrule
    DiT-B/2 & 75.8 & 62.2 & 62.57 & \textbf{60.0} \\
    DiT-L/2 & 48.0 & 42.5 & 43.44 & \textbf{39.9} \\
    DiT-XL/2 & 44.5 & 38.7 & 37.82 & \textbf{35.9} \\
    U-ViT-H/2 & 33.1 & -- & 29.5 & \textbf{21.6} \\
    \bottomrule
    \end{tabular}%
}
\vspace{-1.0em}
\end{wraptable}

As shown in \cref{tab:weighting_comparison}, RPDM consistently achieves the lowest FID across all evaluated DiT and U-ViT configurations. The improvement is observed across architectures and model scales, indicating that the benefit of target-dependent optimization allocation is not restricted to a particular model configuration.

In contrast to Min-SNR, which allocates optimization effort according to truncated marginal SNR, RPDM specifies its optimization allocation based on the prediction-target amplitude.

\paragraph{Influence of Latent-Space Statistics.}
\label{par:latent_statistics}
The RPDM allocation in \Eqref{eq:effective_amplitude_score} uses the average per-dimension second moment $q_x=d^{-1}\mathbb E[\|x_0\|_2^2]$, without assuming zero mean or identity covariance. In our main experiments, we use $q_x=1$, motivated by the standard VAE latent rescaling factor $0.18215$.\footnote{\url{https://github.com/huggingface/diffusers/issues/437}} We also estimate $\hat q_x$ from training latents and compare the following $x_0$-prediction allocation profiles:
\begin{equation}
M_{\mathrm{simple}}(\lambda)\propto\alpha(\lambda),\qquad M_{\mathrm{generalized}}(\lambda)\propto\alpha(\lambda)\sqrt{\hat q_x}.
\end{equation}
Since $\hat q_x$ is independent of $\lambda$, these profiles coincide after allocation normalization. Their nearly identical empirical results in \cref{tab:latent_statistics_ablation} are consistent with this equivalence.

\begin{table}[!htbp]
\centering
\caption{Comparison of simplified and generalized RPDM allocations under DiT-B $x_0$-prediction.}
\label{tab:latent_statistics_ablation}
\vspace{-0.5em}
\resizebox{0.95\linewidth}{!}{%
    \begin{tabular}{lcccccc}
    \toprule
    Method & RPDM Allocation & FID$\downarrow$ & sFID$\downarrow$ & IS$\uparrow$ & Prec.$\uparrow$ & Rec.$\uparrow$ \\
    \midrule
    Baseline & None & 47.76 & 8.34 & 33.55 & 0.44 & 0.62 \\
    Simple RPDM & $M(\lambda)\propto\alpha(\lambda)$ & 36.36 & \textbf{7.18} & 44.78 & \textbf{0.50} & \textbf{0.64} \\
    Generalized RPDM & $M(\lambda)\propto\alpha(\lambda)\sqrt{\hat q_x}$ & \textbf{36.34} & 7.19 & \textbf{44.82} & 0.49 & \textbf{0.64} \\
    \bottomrule
    \end{tabular}%
}
\end{table}

\paragraph{Complementarity with External Representation Alignment.}
\label{par:repa_complementarity}
RPDM and REPA improve representation optimization through complementary mechanisms. REPA introduces external semantic supervision using a pretrained visual encoder, whereas RPDM reallocates optimization effort across noise levels without relying on an external teacher. We therefore evaluate their combination using
\begin{equation}
\mathcal{L}_{\mathrm{joint}}=\mathcal{L}_{\mathrm{RPDM}}+\beta\mathcal{L}_{\mathrm{REPA}},
\end{equation}
\begin{wraptable}{r}{0.50\textwidth}
\vspace{-0.5em}
\centering
\small
\caption{FID convergence of RPDM and REPA under the same DiT training configuration.}
\label{tab:erd_repa_complementarity}
\vspace{-0.5em}
\resizebox{\linewidth}{!}{%
    \begin{tabular}{lcccc}
    \toprule
    Method & 100K & 200K & 300K & 400K \\
    \midrule
    Baseline & 75.80 & 57.50 & 48.20 & 43.40 \\
    RPDM & 60.00 & 43.60 & 36.80 & 32.30 \\
    REPA & 55.65 & 42.68 & 37.08 & 34.47 \\
    RPDM + REPA & \textbf{52.58} & \textbf{36.62} & \textbf{31.20} & \textbf{28.87} \\
    \bottomrule
    \end{tabular}%
}
\vspace{-1.0em}
\end{wraptable}
where $\beta$ controls the strength of representation alignment. We align the eighth transformer layer and use $\beta=0.5$ for REPA. For RPDM+REPA, we use $\beta=0.1$ to balance representation alignment with the modified diffusion objective.

As shown in \cref{tab:erd_repa_complementarity}, RPDM+REPA improves FID throughout training and outperforms either method alone, indicating that optimization allocation and external representation supervision are complementary, despite partially overlapping benefits at later stages.

\subsection{Discussion on Layer-wise Allocation}
\label{sec:layerwise_allocation}

The NTK naturally admits a layer-wise decomposition:
\begin{equation}
    \Theta_{\theta,m}(z,z')=\sum_l J_l(z)J_l(z')^\top,
    \qquad J_l(z)=\nabla_{\theta_l}f_{\theta,m}(z),
\end{equation}
where $z=(x,\lambda)$ and $\theta_l$ denotes the $l$-th parameter group. While our analysis focuses on the full NTK, \cref{fig:effective_rank} shows that RPDM yields depth-dependent improvements in effective feature rank, particularly near the output layer. This motivates a natural extension from $M(\lambda)$ to layer-dependent allocation $M_l(\lambda)$.

As a preliminary exploration, we interpolate between the original RPDM sample weight and its batch mean:
\begin{equation}
    w_{i,l}=(1-\gamma_l)\bar{w}_B+\gamma_l w_i,
    \qquad \bar{w}_B=\frac{1}{B}\sum_{i=1}^{B}w_i.
\end{equation}
\par\vspace{-3pt}
\begin{wraptable}{r}{0.50\textwidth}
    \centering
    \vspace{0em}
    \caption{Preliminary results of layer-wise optimization allocation on DiT-B/2 after 100K iterations on ImageNet $256\times256$ without CFG. Arrows indicate the variation of $\gamma_l$ from shallow to deep layers.}
    \vspace{-0.5em}
    \label{tab:layerwise_allocation}
    \small
    \setlength{\tabcolsep}{3pt}
    \resizebox{\linewidth}{!}{
        \begin{tabular}{lcccccc}
            \toprule
            Method & $\gamma_l$ & FID$\downarrow$ & sFID$\downarrow$ & IS$\uparrow$ & Prec.$\uparrow$ & Rec.$\uparrow$ \\
            \midrule
            RPDM (ours) & $1$ & 60.00 & \textbf{8.40} & 21.80 & 0.40 & 0.59 \\
            \midrule
            \multirow{2}{*}{+ Layer-wise} & $0\to1$ & 71.41 & 11.77 & 17.49 & 0.35 & 0.55 \\
                                         & $1\to0$ & \textbf{57.83} & 8.76 & \textbf{22.20} & \textbf{0.41} & \textbf{0.60} \\
            \bottomrule
            \vspace{-2em}
        \end{tabular}
    }
\end{wraptable}

% \begin{table}[htbp]
%     \centering
%     \begin{minipage}{0.62\textwidth}
%         \centering
%         \caption{Preliminary results of layer-wise optimization allocation on DiT-B/2 after 100K iterations on ImageNet $256\times256$ without CFG. Arrows indicate the variation of $\gamma_l$ from shallow to deep layers.}
%         \label{tab:layerwise_allocation}
%         \small
%         \setlength{\tabcolsep}{3pt}
%         \resizebox{\linewidth}{!}{
%             \begin{tabular}{lcccccc}
%                 \toprule
%                 Method & $\gamma_l$ & FID$\downarrow$ & sFID$\downarrow$ & IS$\uparrow$ & Prec.$\uparrow$ & Rec.$\uparrow$ \\
%                 \midrule
%                 RPDM (ours) & $1$ & 60.00 & \textbf{8.40} & 21.80 & 0.40 & 0.59 \\
%                 \midrule
%                 \multirow{2}{*}{+ Layer-wise} & $0\to1$ & 71.41 & 11.77 & 17.49 & 0.35 & 0.55 \\
%                                              & $1\to0$ & \textbf{57.83} & 8.76 & \textbf{22.20} & \textbf{0.41} & \textbf{0.60} \\
%                 \bottomrule
%             \end{tabular}
%         }
%     \end{minipage}
% \end{table}
where $w_i$ denotes the RPDM weight of sample $i$, $B$ is the batch size, and $l=0,\ldots,L-1$ indexes the Transformer blocks from shallow to deep. We consider two simple linear schedules, $\gamma_l=l/(L-1)$ and $\gamma_l=1-l/(L-1)$. The resulting allocation is implemented by rescaling backward gradients with adjacent weight ratios $w_{i,l}/w_{i,l+1}$, setting $w_{i,L}=w_i$, without altering the forward computation.

As shown in \cref{tab:layerwise_allocation}, the two schedules yield substantially different generation performance despite identical depth-averaged interpolation coefficients. The decreasing schedule further improves FID over RPDM without fine-grained tuning, highlighting the potential of depth-dependent optimization allocation. We leave a systematic analysis of layer-wise NTK dynamics and the development of more refined allocation strategies for future work.

\subsection{Detailed Quantitative Results}

We evaluate RPDM across architectures and model scales, training efficiency, comparisons with existing loss-weighting methods, robustness to latent-space statistics, complementarity with external representation alignment, and classifier-free guidance. Unless otherwise specified, all results are evaluated on ImageNet $256\times256$ without classifier-free guidance.

\paragraph{Training Convergence Curves Comparison.}
\label{par:training_efficiency}
\cref{fig:convergence_curve} shows FID convergence curves across all backbones. RPDM achieves baseline performance with substantially fewer optimization steps, consistently accelerating convergence for both U-ViT and DiT families. These results demonstrate that RPDM consistently improves training efficiency across architectures and model scales.

\begin{figure}[!htbp]
\vspace{-1.0em}
\centering
\begin{subfigure}{0.32\textwidth}
\centering
\includegraphics[width=\linewidth]{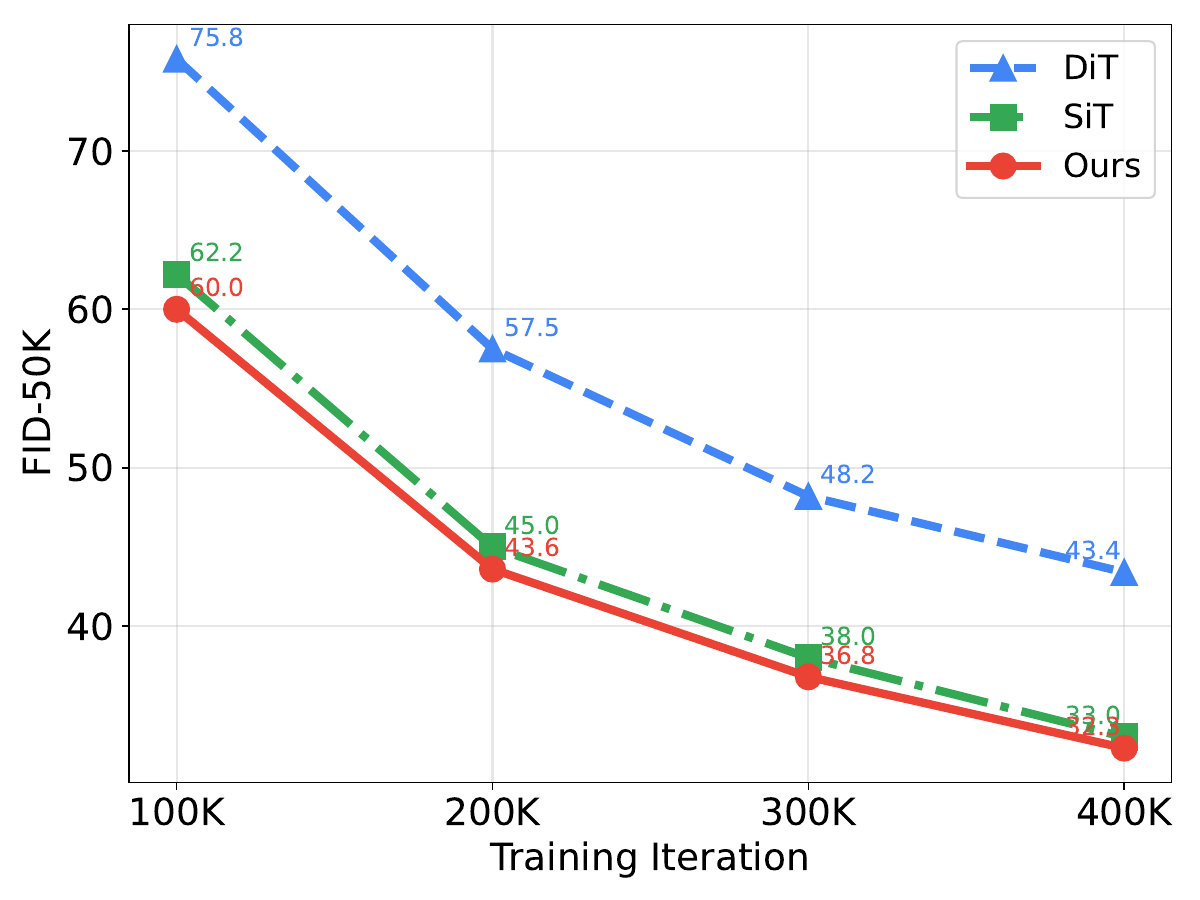}
\vspace{-1.5em}
\caption{DiT-B/2}
\label{fig:conv_dit_b2}
\end{subfigure}
\hfill
\begin{subfigure}{0.32\textwidth}
\centering
\includegraphics[width=\linewidth]{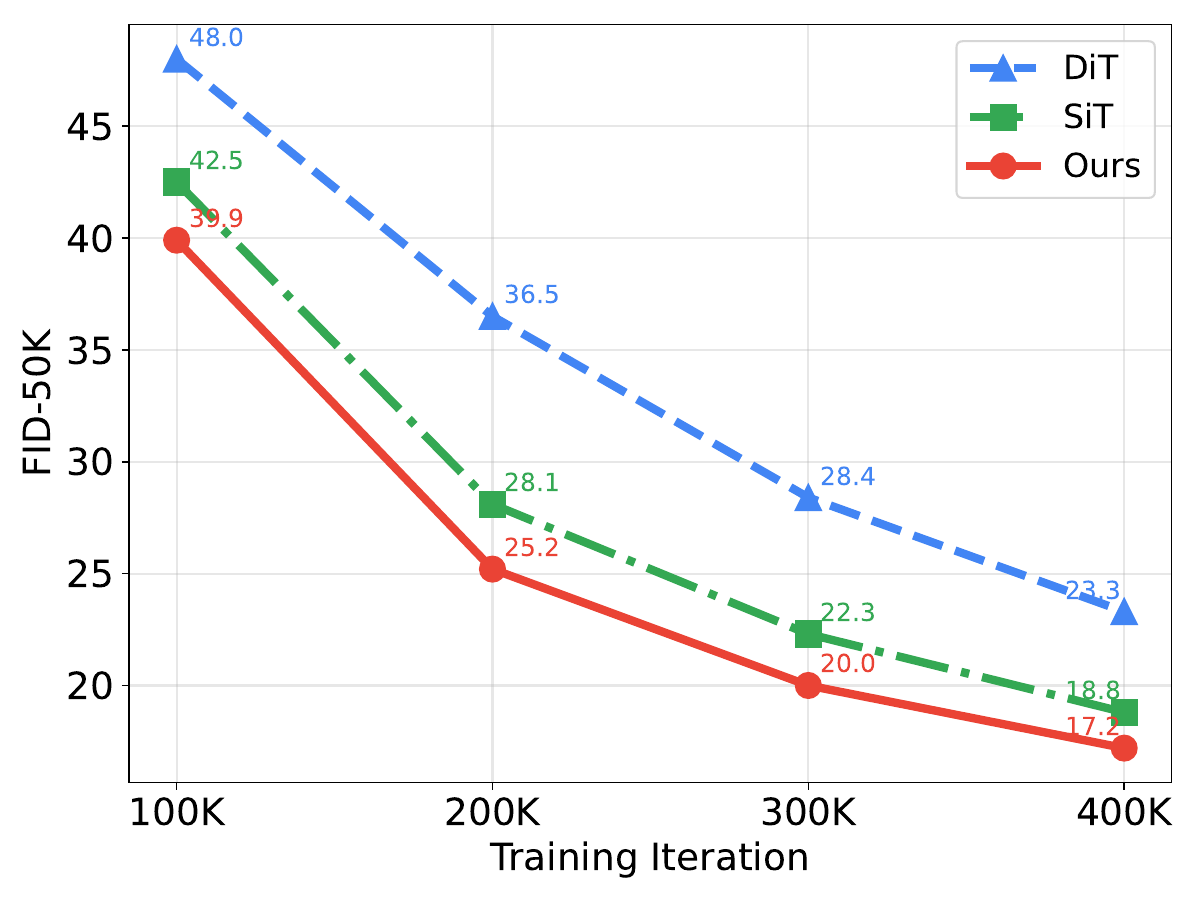}
\vspace{-1.5em}
\caption{DiT-L/2}
\label{fig:conv_dit_l2}
\end{subfigure}
\hfill
\begin{subfigure}{0.32\textwidth}
\centering
\includegraphics[width=\linewidth]{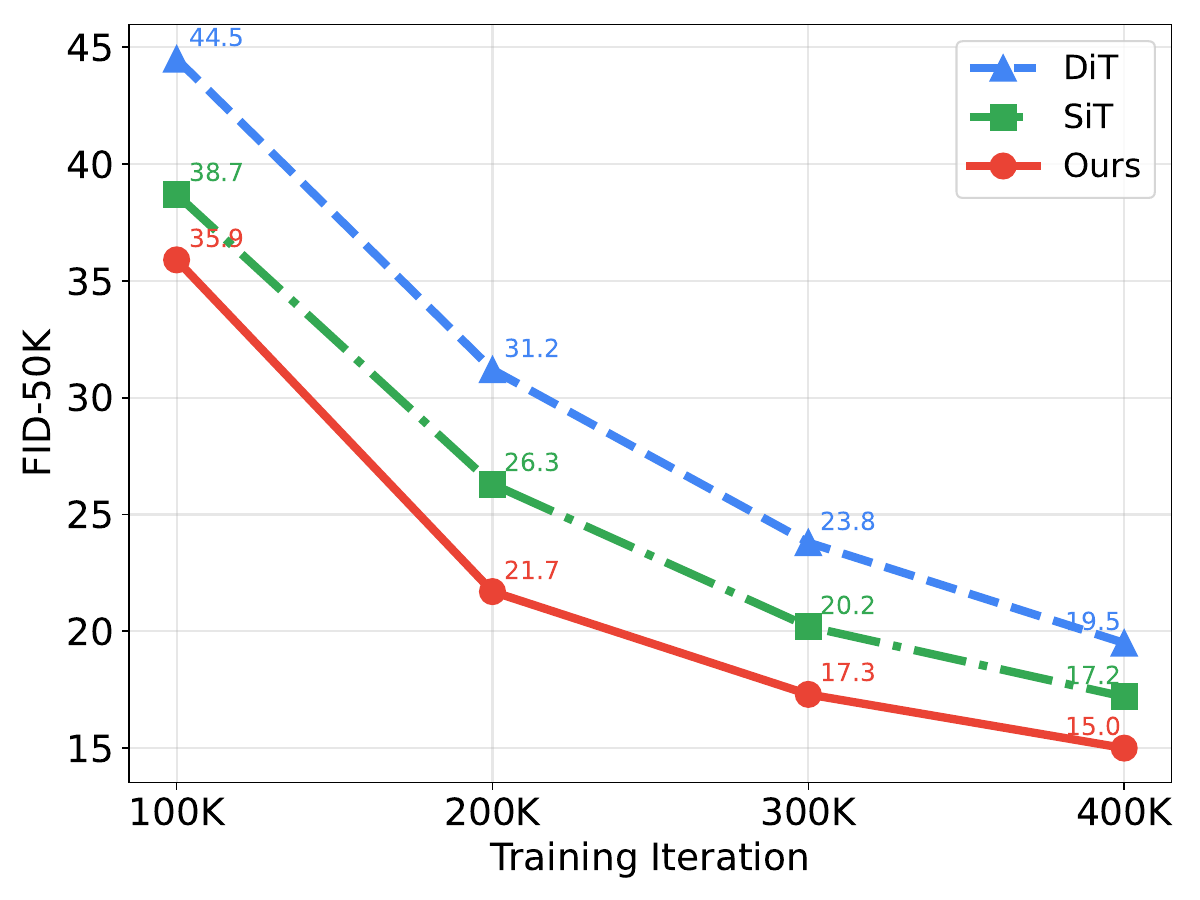}
\vspace{-1.5em}
\caption{DiT-XL/2}
\label{fig:conv_dit_xl2}
\end{subfigure}

\vspace{0.5em}

\begin{subfigure}{0.32\textwidth}
\centering
\includegraphics[width=\linewidth]{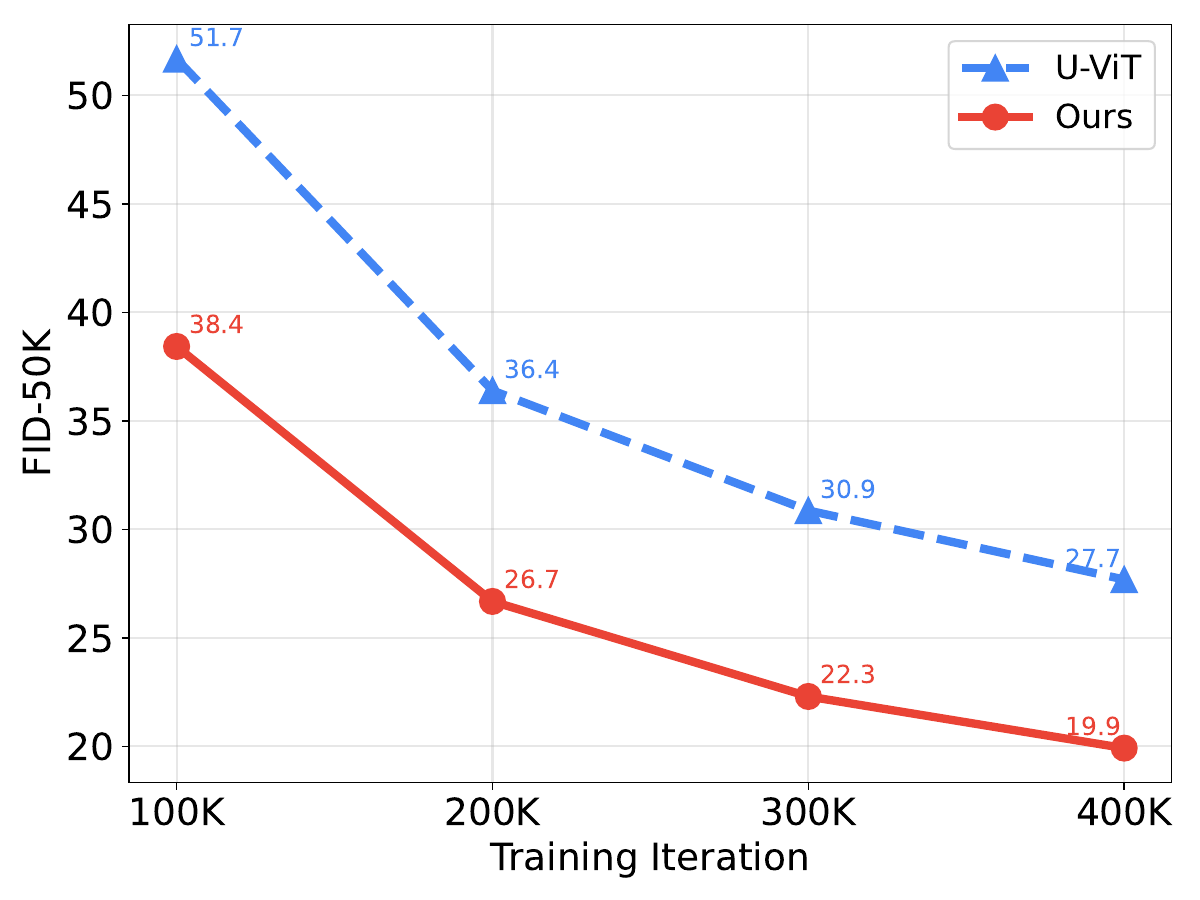}
\vspace{-1.5em}
\caption{U-ViT-M/2}
\label{fig:conv_uvit_m2}
\end{subfigure}
\hfill
\begin{subfigure}{0.32\textwidth}
\centering
\includegraphics[width=\linewidth]{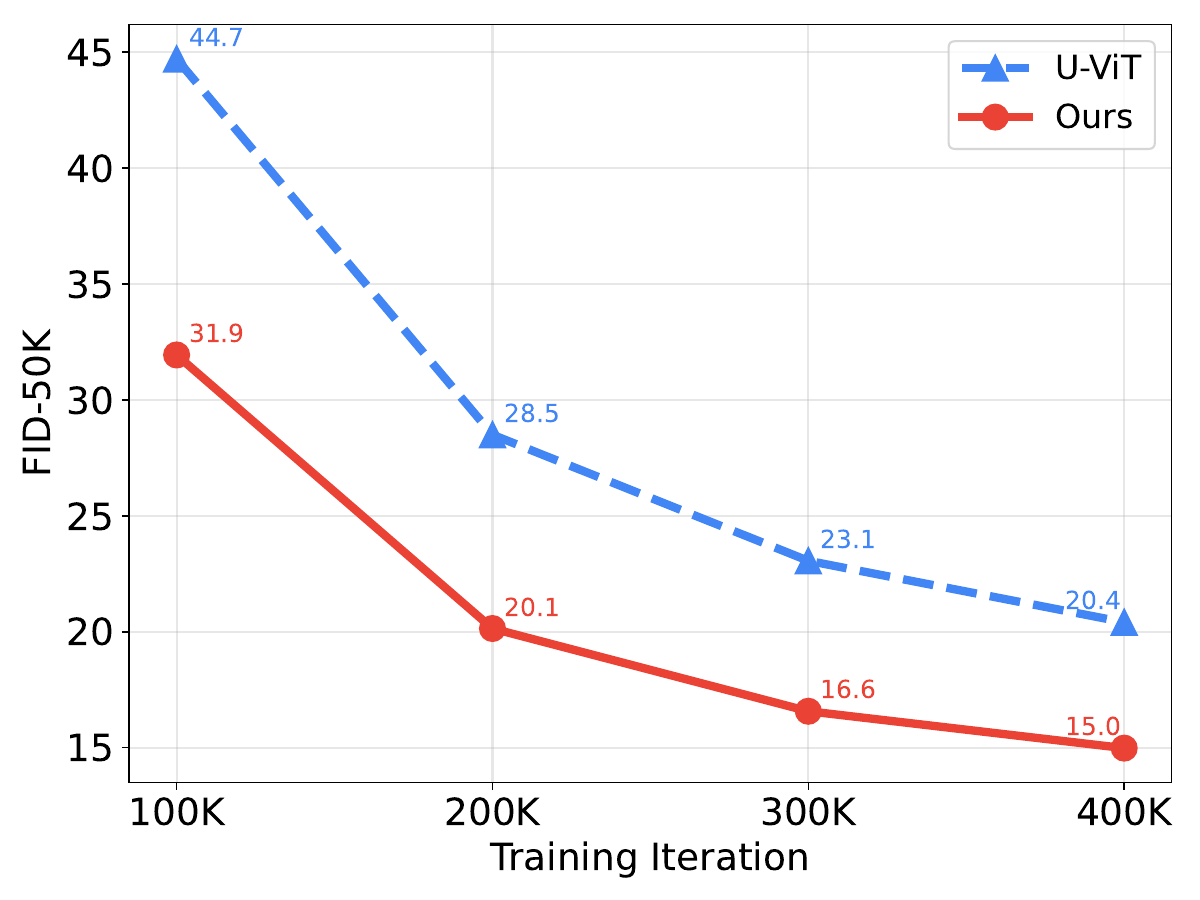}
\vspace{-1.5em}
\caption{U-ViT-L/2}
\label{fig:conv_uvit_l2}
\end{subfigure}
\hfill
\begin{subfigure}{0.32\textwidth}
\centering
\includegraphics[width=\linewidth]{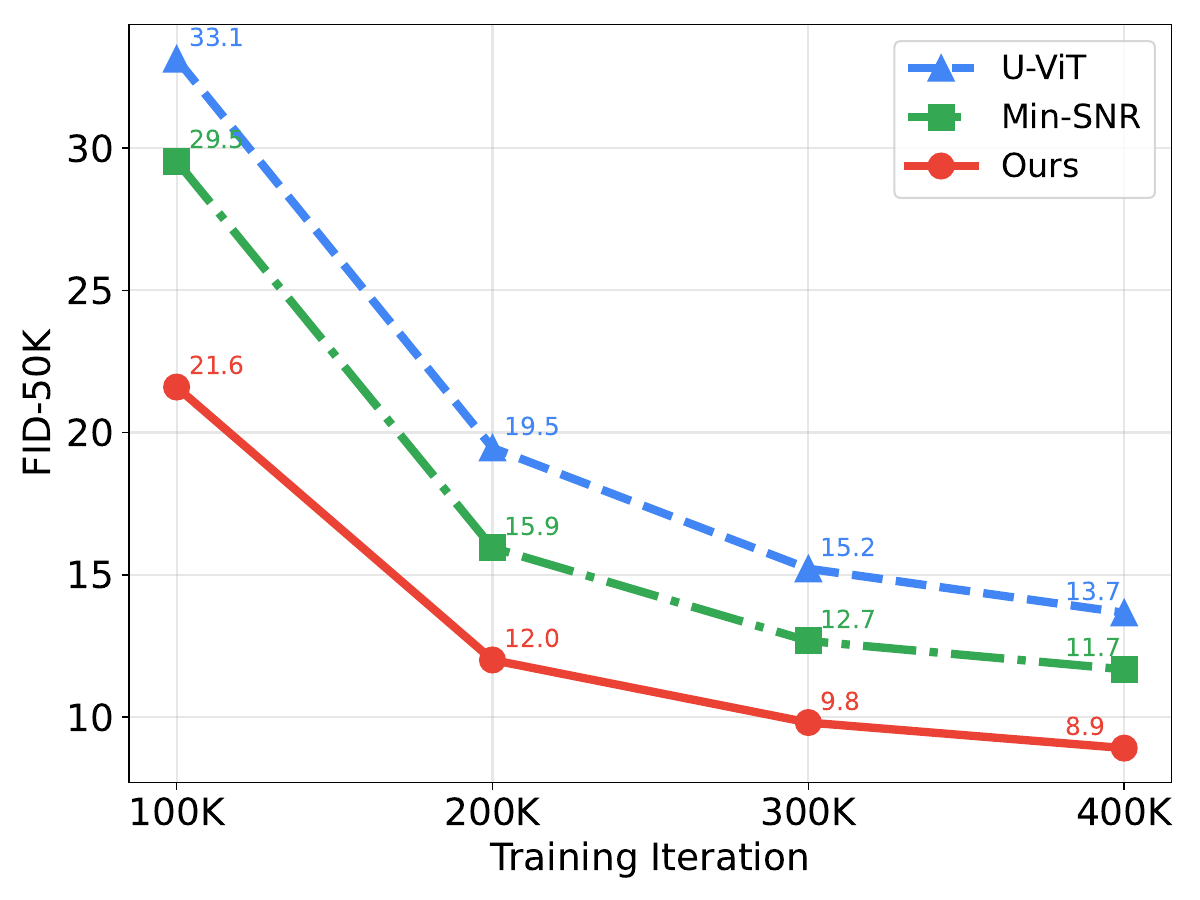}
\vspace{-1.5em}
\caption{U-ViT-H/2}
\label{fig:conv_uvit_h2}
\end{subfigure}
\vspace{-0.5em}
\caption{Convergence curves on ImageNet $256\times256$. Intermediate FID values of DiT are approximated from the training curves in~\cite{ma2024sit,jiang2025no} for visualization only.}
\label{fig:convergence_curve}
\vspace{-1.0em}
\end{figure}

\paragraph{Detailed Generation Results Comparison.}
\label{par:architecture_scaling}
\cref{tab:dit_comparison,tab:uvit_comparison} report quantitative results for DiT and U-ViT backbones, respectively. RPDM consistently outperforms the corresponding baselines across architectures and model scales under the same training budget. These improvements across both model families demonstrate the scalability and architecture-agnostic effectiveness of RPDM.

\begin{table}[!htbp]
    % \vspace{-0.5em}
    \centering
    \caption{Detailed results across DiT model scales without classifier-free guidance.}
    \label{tab:dit_comparison}
    \vspace{-0.5em}
    \resizebox{0.95\linewidth}{!}{
        \begin{tabular}{lccccccc}
            \toprule
            Model & \#Params & Iter. & FID$\downarrow$ & sFID$\downarrow$ & IS$\uparrow$ & Prec.$\uparrow$ & Rec.$\uparrow$ \\
            \midrule
            \rowcolor{orange!10}
            SiT-B/2~\cite{ma2024sit} & 130M & 400K & 33.0 & 6.5 & 43.7 & 0.53 & 0.63 \\
            \rowcolor{cyan!8}
            DiT-B/2~\cite{peebles2023scalable} & 130M & 400K & 43.4 & - & - & - & - \\
            + RPDM (ours) & 130M & 100K & 60.0 & 8.4 & 21.8 & 0.40 & 0.59 \\
            + RPDM (ours) & 130M & 200K & 43.6 & 7.6 & 32.6 & 0.47 & 0.64 \\
            + RPDM (ours) & 130M & 300K & 36.8 & 7.3 & 40.5 & 0.51 & 0.66 \\
            + RPDM (ours) & 130M & 400K & 32.3 & 7.0 & 46.6 & 0.53 & 0.66 \\
            \midrule
            \rowcolor{orange!10}
            SiT-L/2~\cite{ma2024sit} & 458M & 400K & 18.8 & 5.3 & 72.0 & 0.64 & 0.64 \\
            \rowcolor{cyan!8}
            DiT-L/2~\cite{peebles2023scalable} & 458M & 400K & 23.3 & - & - & - & - \\
            + RPDM (ours) & 458M & 100K & 39.9 & 7.7 & 33.9 & 0.50 & 0.64 \\
            + RPDM (ours) & 458M & 200K & 25.2 & 7.0 & 55.8 & 0.58 & 0.65 \\
            + RPDM (ours) & 458M & 300K & 20.0 & 6.8 & 69.6 & 0.61 & 0.66 \\
            + RPDM (ours) & 458M & 400K & 17.2 & 6.5 & 78.9 & 0.62 & 0.67 \\
            \midrule
            \rowcolor{orange!10}
            SiT-XL/2~\cite{ma2024sit} & 675M & 400K & 17.2 & - & - & - & - \\
            \rowcolor{cyan!8}
            DiT-XL/2~\cite{peebles2023scalable} & 675M & 400K & 19.5 & - & - & - & - \\
            + RPDM (ours) & 675M & 100K & 35.9 & 7.5 & 37.4 & 0.53 & 0.64 \\
            + RPDM (ours) & 675M & 200K & 21.7 & 6.6 & 62.6 & 0.60 & 0.65 \\
            + RPDM (ours) & 675M & 300K & 17.3 & 6.4 & 77.4 & 0.62 & 0.66 \\
            + RPDM (ours) & 675M & 400K & 15.0 & 6.4 & 86.4 & 0.64 & 0.66 \\
            \bottomrule
        \end{tabular}
    }
    % \vspace{-0.5em}
\end{table}

\begin{table}[!htbp]
    \vspace{-1.0em}
    \centering
    \caption{Detailed results across U-ViT model scales without classifier-free guidance.}
    \vspace{-0.5em}
    \label{tab:uvit_comparison}
    \resizebox{0.95\linewidth}{!}{
        \begin{tabular}{lccccccc}
            \toprule
            Model & \#Params & Iter. & FID$\downarrow$ & sFID$\downarrow$ & IS$\uparrow$ & Prec.$\uparrow$ & Rec.$\uparrow$ \\
            \midrule
            \rowcolor{cyan!8}
            U-ViT-M/2 & 131M & 400K & 27.7 & 7.0 & 50.9 & 0.55 & 0.66 \\
            + RPDM (ours) & 131M & 100K & 38.4 & 8.8 & 37.0 & 0.52 & 0.63 \\
            + RPDM (ours) & 131M & 200K & 26.7 & 8.1 & 54.3 & 0.57 & 0.65 \\
            + RPDM (ours) & 131M & 300K & 22.3 & 7.5 & 64.1 & 0.59 & 0.66 \\
            + RPDM (ours) & 131M & 400K & 19.9 & 7.0 & 70.8 & 0.60 & 0.66 \\
            \midrule
            \rowcolor{cyan!8}
            U-ViT-L/2 & 287M & 400K & 20.4 & 6.6 & 68.0 & 0.60 & 0.67 \\
            + RPDM (ours) & 287M & 100K & 31.9 & 7.4 & 44.3 & 0.56 & 0.63 \\
            + RPDM (ours) & 287M & 200K & 20.1 & 7.2 & 67.3 & 0.62 & 0.64 \\
            + RPDM (ours) & 287M & 300K & 16.6 & 6.8 & 79.2 & 0.63 & 0.66 \\
            + RPDM (ours) & 287M & 400K & 15.0 & 6.7 & 86.3 & 0.64 & 0.65 \\
            \midrule
            \rowcolor{green!10}
            Min-SNR & 501M & 400K & 11.7 & 6.6 & 99.7 & 0.66 & 0.65 \\
            \rowcolor{cyan!8}
            U-ViT-H/2 & 501M & 400K & 13.7 & 5.6 & 87.9 & 0.65 & 0.65 \\
            + RPDM (ours) & 501M & 100K & 21.6 & 6.3 & 60.6 & 0.63 & 0.60 \\
            + RPDM (ours) & 501M & 200K & 12.0 & 5.6 & 92.9 & 0.67 & 0.64 \\
            + RPDM (ours) & 501M & 300K & 9.8 & 5.5 & 106.8 & 0.68 & 0.64 \\
            + RPDM (ours) & 501M & 400K & 8.9 & 5.4 & 114.3 & 0.68 & 0.66 \\
            \bottomrule
        \end{tabular}
    }
    \vspace{-0.5em}
\end{table}

\paragraph{Classifier-Free Guidance and Guidance Intervals.}
\label{par:cfg_results}
Finally, we evaluate U-ViT-H/2 trained for 4M iterations with classifier-free guidance. \cref{tab:detailed_quantitative_cfg} reports results across guidance scales. Increasing the guidance scale from $w=1.30$ to $w=1.50$ consistently improves FID, sFID, and IS, achieving the best FID of 1.89 at $w=1.50$. We further apply guidance intervals following~\cite{kynkaanniemi2024applying}. As shown in \cref{tab:detailed_quantitative_cfg_interval}, restricting guidance to an appropriate portion of the sampling trajectory further improves generation quality. The optimal configuration combines the interval $[0,0.8]$ with $w=2.10$, achieving an FID of 1.45.

\begin{table}[!htbp]
% \vspace{-0.5em}
\centering
\caption{Results with different classifier-free guidance scales $w$.}
\vspace{-0.5em}
\begin{adjustbox}{width=0.95\textwidth}
\begin{tabular}{lcccccccc}
\toprule
Model & \#Params & Iter. & $w$ & FID$\downarrow$ & sFID$\downarrow$ & IS$\uparrow$ & Prec.$\uparrow$ & Rec.$\uparrow$ \\
\midrule
\rowcolor{gray!15}
DiT-XL/2~\cite{peebles2023scalable} & 675M & 7M & 1.500 & 2.27 & 4.60 & 278.24 & \textbf{0.83} & 0.57 \\
\rowcolor{gray!15}
SiT-XL/2~\cite{ma2024sit} & 675M & 7M & 1.500 & 2.06 & 4.50 & 270.3 & 0.82 & 0.59 \\
\rowcolor{gray!15}
U-ViT-H/2~\cite{bao2023all} & 501M & 2M & 1.500 & 2.29 & 5.68 & 263.88 & 0.82 & 0.57 \\
RPDM (ours) & 501M & 4M & 1.300 & 2.08 & 4.62 & 255.6 & 0.79 & \textbf{0.64} \\
RPDM (ours) & 501M & 4M & 1.350 & 2.00 & 4.57 & 266.9 & 0.80 & 0.63 \\
RPDM (ours) & 501M & 4M & 1.400 & 1.94 & 4.53 & 273.8 & 0.81 & 0.62 \\
RPDM (ours) & 501M & 4M & 1.450 & 1.91 & 4.50 & 279.3 & 0.82 & 0.61 \\
RPDM (ours) & 501M & 4M & 1.500 & \textbf{1.89} & \textbf{4.48} & \textbf{282.5} & 0.82 & 0.60 \\
\bottomrule
\end{tabular}
\end{adjustbox}
\label{tab:detailed_quantitative_cfg}

\par\vspace{0.8em}

\caption{Results with different classifier-free guidance intervals~\cite{kynkaanniemi2024applying}.}
\vspace{-0.5em}
\begin{adjustbox}{width=0.95\textwidth}
\begin{tabular}{lccccccccc}
\toprule
Model & \#Params & Iter. & Interval & $w$ & FID$\downarrow$ & sFID$\downarrow$ & IS$\uparrow$ & Prec.$\uparrow$ & Rec.$\uparrow$ \\
\midrule
RPDM (ours) & 501M & 4M & $[0,0.90]$ & 2.10 & 1.71 & 4.43 & \textbf{326.8} & \textbf{0.83} & 0.59 \\
RPDM (ours) & 501M & 4M & $[0,0.85]$ & 2.10 & 1.56 & 4.47 & 319.5 & 0.82 & 0.60 \\
RPDM (ours) & 501M & 4M & $[0,0.80]$ & 2.00 & 1.48 & 4.54 & 304.8 & 0.81 & 0.63 \\
RPDM (ours) & 501M & 4M & $[0,0.80]$ & 2.10 & \textbf{1.45} & \textbf{4.51} & 312.4 & 0.81 & 0.62 \\
RPDM (ours) & 501M & 4M & $[0,0.80]$ & 2.20 & 1.49 & 4.53 & 311.7 & 0.82 & 0.61 \\
RPDM (ours) & 501M & 4M & $[0,0.75]$ & 2.10 & 1.52 & 4.66 & 300.3 & 0.80 & 0.64 \\
RPDM (ours) & 501M & 4M & $[0,0.70]$ & 2.10 & 1.60 & 4.83 & 289.6 & 0.79 & \textbf{0.65} \\
\bottomrule
\end{tabular}
\end{adjustbox}
\label{tab:detailed_quantitative_cfg_interval}
\vspace{-0.5em}
\end{table}

\subsection{Qualitative Results}
\newsavebox{\showcaseimagebox}

\newlength{\showcasecolgap}
\newlength{\showcaserowgap}
\newlength{\showcasecaptiongap}
\newlength{\showcaseimageheight}

\setlength{\showcasecolgap}{16pt}
\setlength{\showcaserowgap}{12pt}
\setlength{\showcasecaptiongap}{3pt}
\setlength{\showcaseimageheight}{\dimexpr0.458\textheight-4pt\relax}

\newcommand{\showcasepanel}[3]{%
    \begin{minipage}[t]{\dimexpr(\linewidth-\showcasecolgap)/2\relax}
        \centering
        \vspace{0pt}%
        \sbox{\showcaseimagebox}{%
            \includegraphics[width=\linewidth,height=\showcaseimageheight,keepaspectratio]{Figures/showcase/vertical_samples_label_#1.png}%
        }%
        \begin{minipage}[t]{\wd\showcaseimagebox}
            \centering
            \usebox{\showcaseimagebox}\par
            \caption{\small We use classifier-free guidance with $w = 4.0$. Class label = ``#2'' (#1).}
            \label{#3}
        \end{minipage}%
    \end{minipage}%
}

\newcommand{\showcasepage}[4]{%
    \begin{figure}[p]
        \centering
        \setlength{\abovecaptionskip}{\showcasecaptiongap}
        \setlength{\belowcaptionskip}{0pt}
        \setlength{\parskip}{0pt}
        \vbox{%
            \hbox to \linewidth{#1\hfil#2}%
            \vskip\showcaserowgap
            \nointerlineskip
            \hbox to \linewidth{#3\hfil#4}%
        }%
    \end{figure}
    \clearpage
}

\clearpage

% Page 1
\showcasepage
    {\showcasepanel{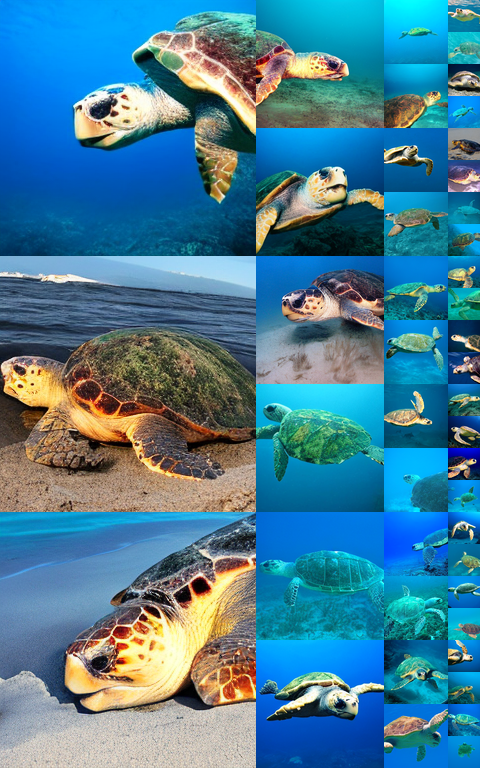}{loggerhead sea turtle}{fig:label33}}
    {\showcasepanel{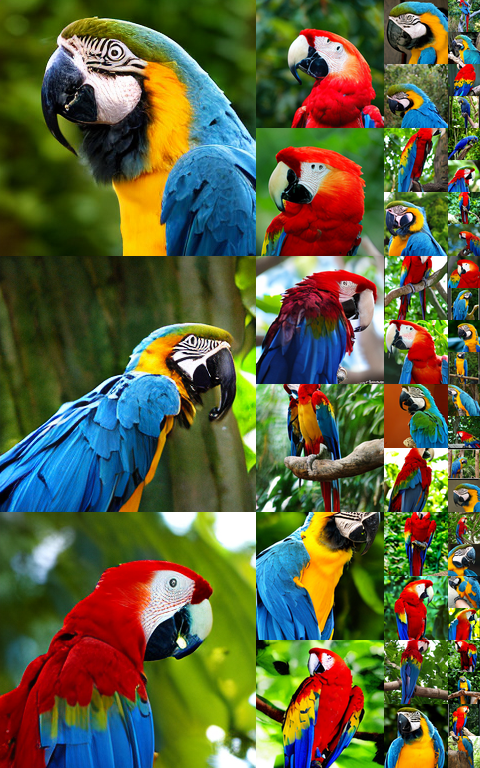}{macaw}{fig:label88}}
    {\showcasepanel{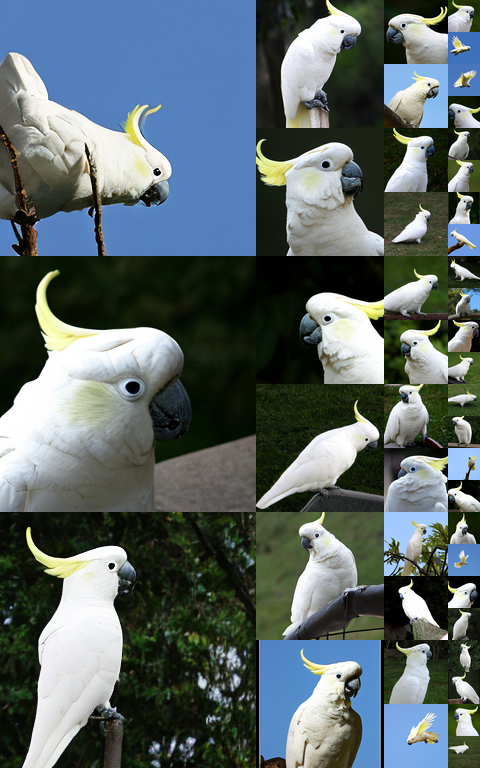}{\textit{Cacatua galerita}}{fig:label89}}
    {\showcasepanel{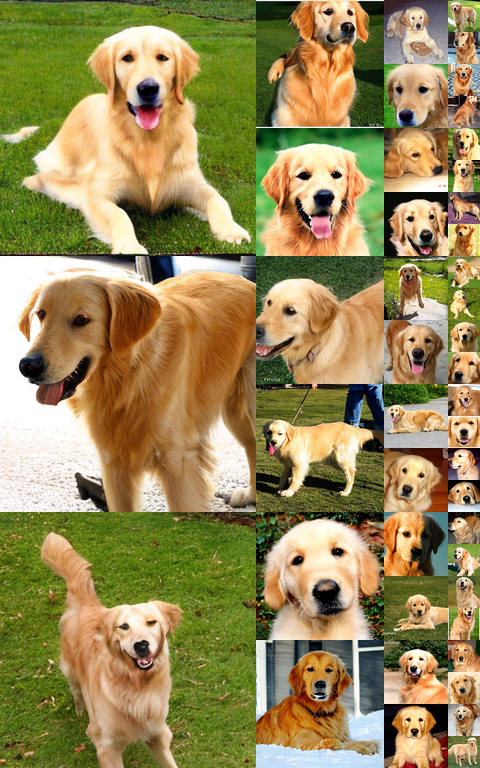}{golden retriever}{fig:label207}}

% Page 2
\showcasepage
    {\showcasepanel{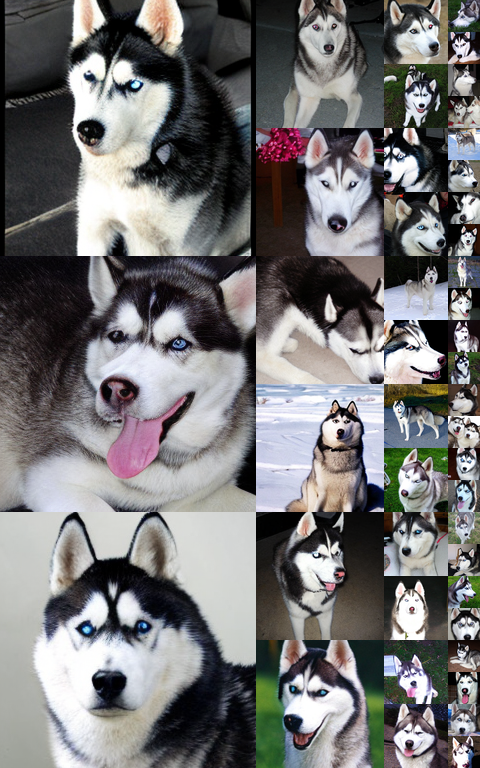}{husky}{fig:label250}}
    {\showcasepanel{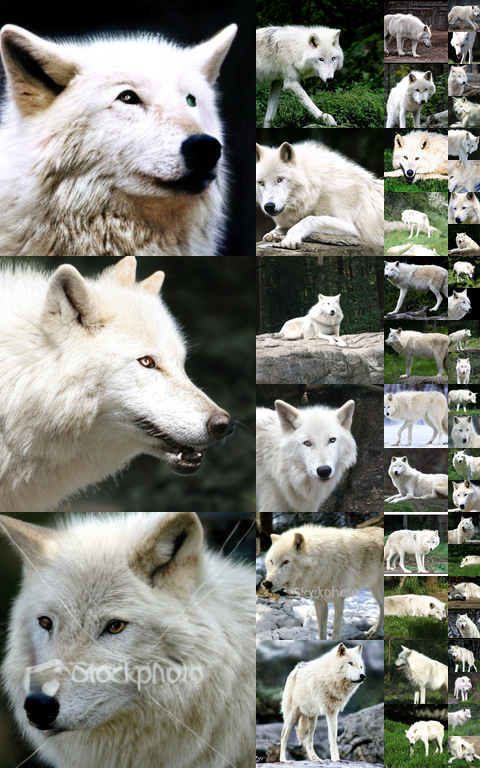}{arctic wolf}{fig:label270}}
    {\showcasepanel{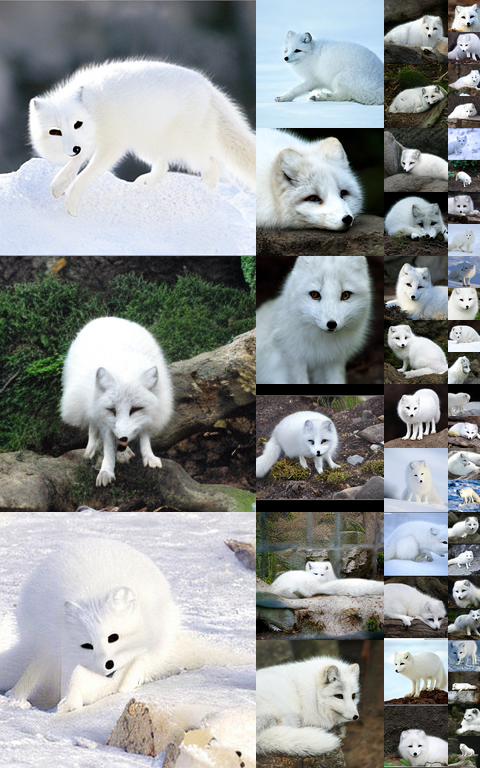}{arctic fox}{fig:label279}}
    {\showcasepanel{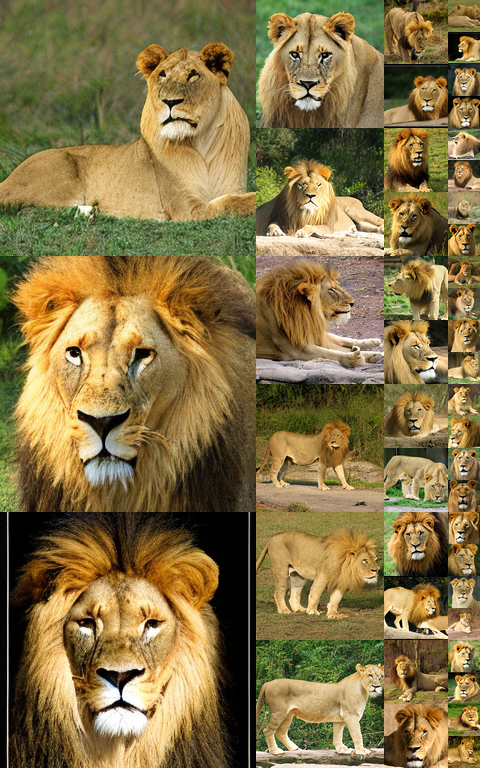}{lion}{fig:label291}}

% Page 3
\showcasepage
    {\showcasepanel{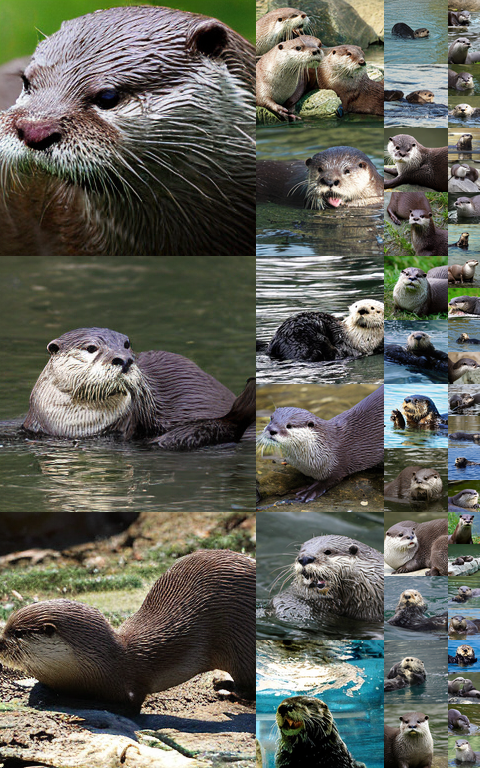}{otter}{fig:label360}}
    {\showcasepanel{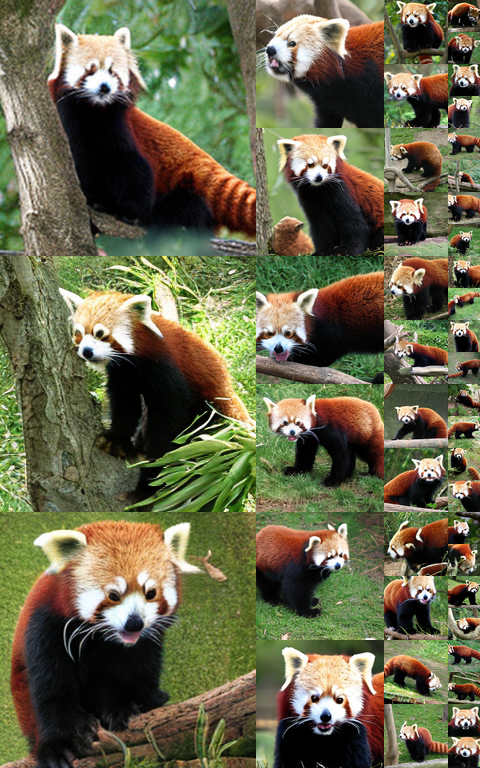}{red panda}{fig:label387}}
    {\showcasepanel{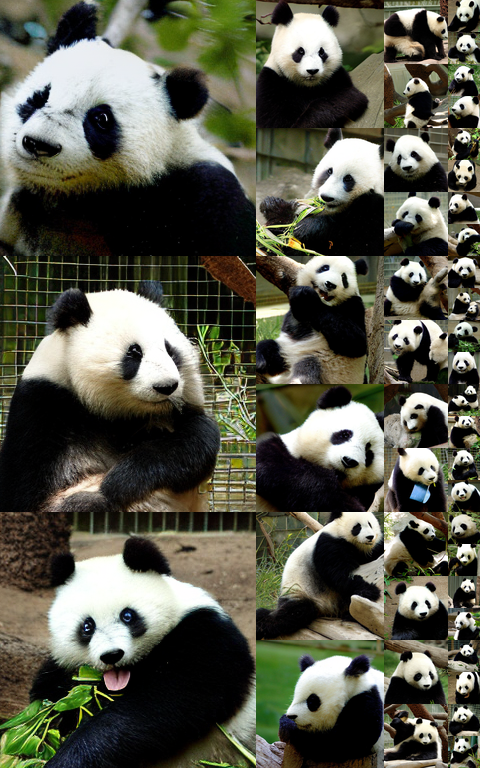}{panda}{fig:label388}}
    {\showcasepanel{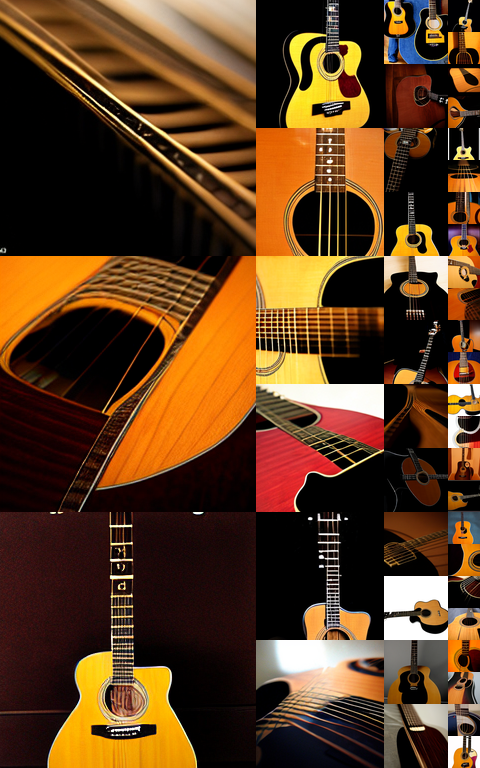}{acoustic guitar}{fig:label402}}

% Page 4
\showcasepage
    {\showcasepanel{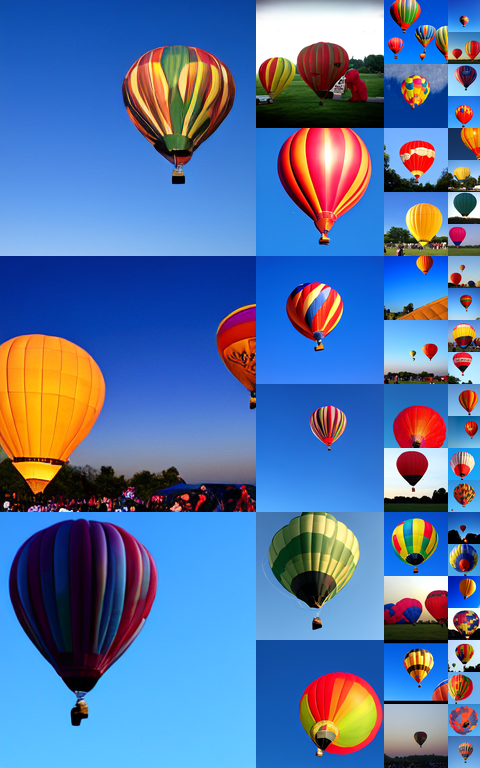}{balloon}{fig:label417}}
    {\showcasepanel{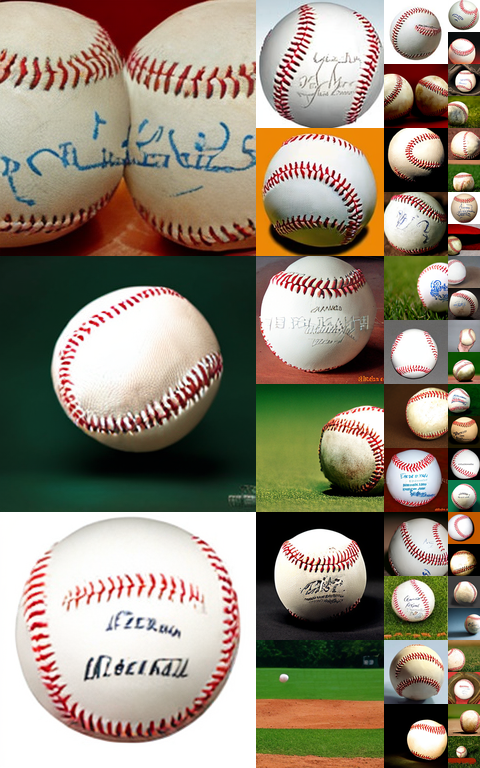}{baseball}{fig:label429}}
    {\showcasepanel{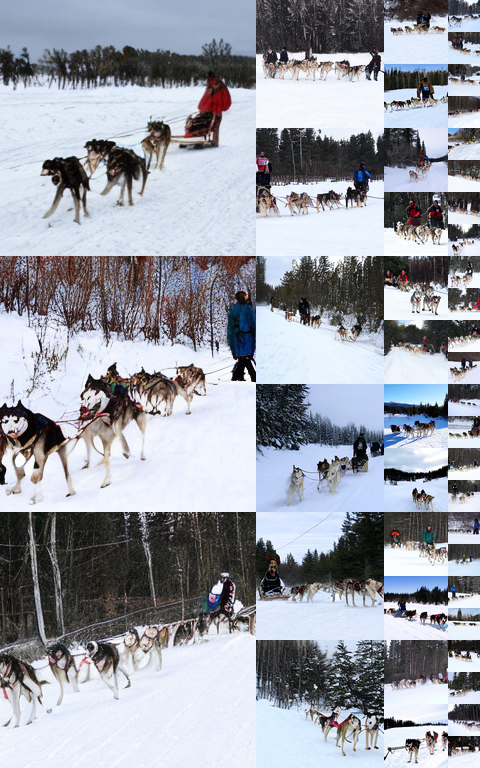}{dog sled}{fig:label537}}
    {\showcasepanel{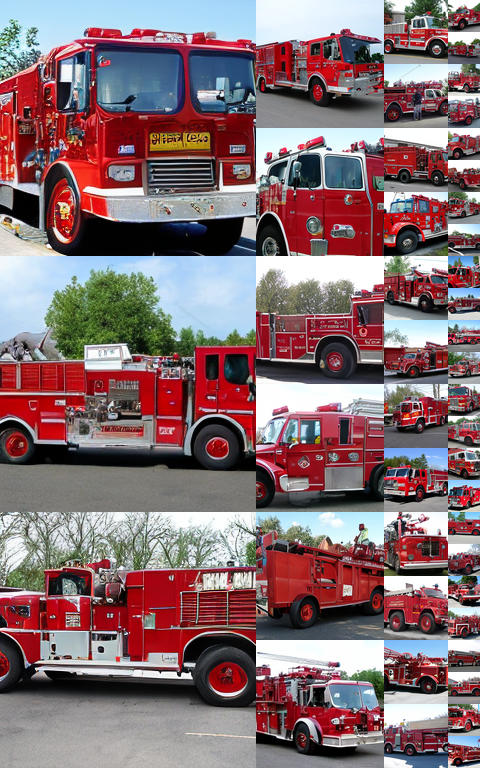}{fire truck}{fig:label555}}

% Page 5
\showcasepage
    {\showcasepanel{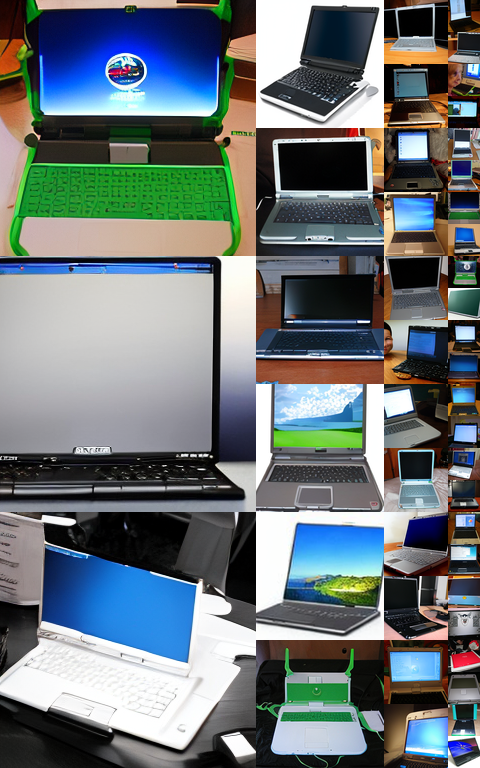}{laptop}{fig:label620}}
    {\showcasepanel{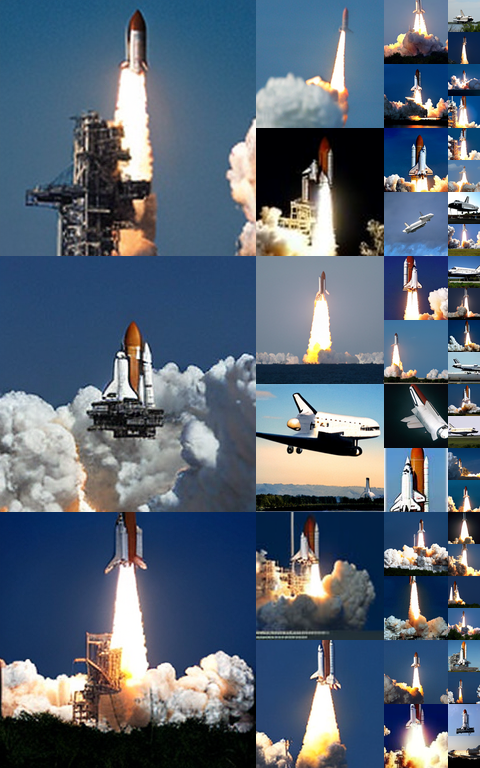}{space shuttle}{fig:label812}}
    {\showcasepanel{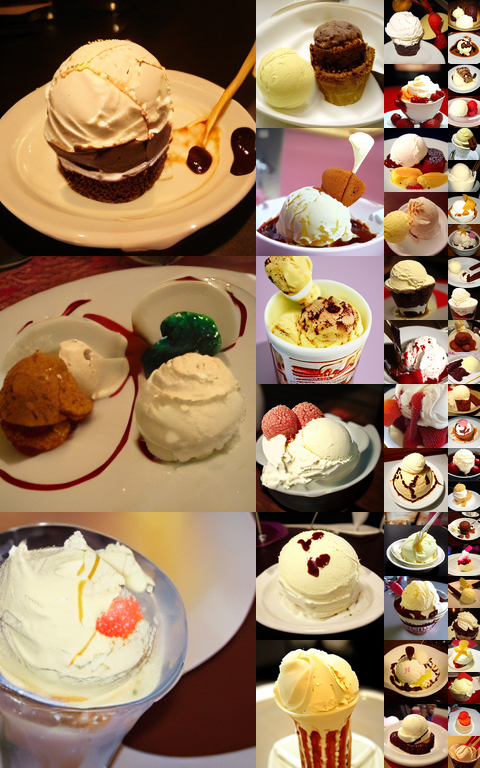}{ice cream}{fig:label928}}
    {\showcasepanel{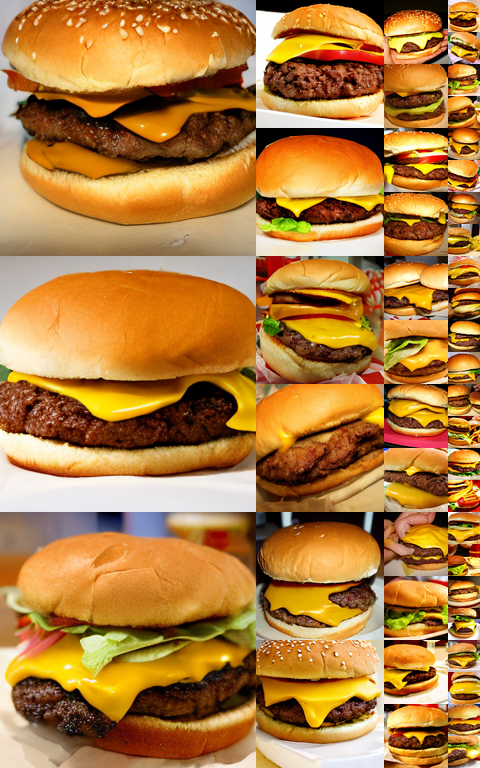}{cheeseburger}{fig:label933}}

% Page 6
\showcasepage
    {\showcasepanel{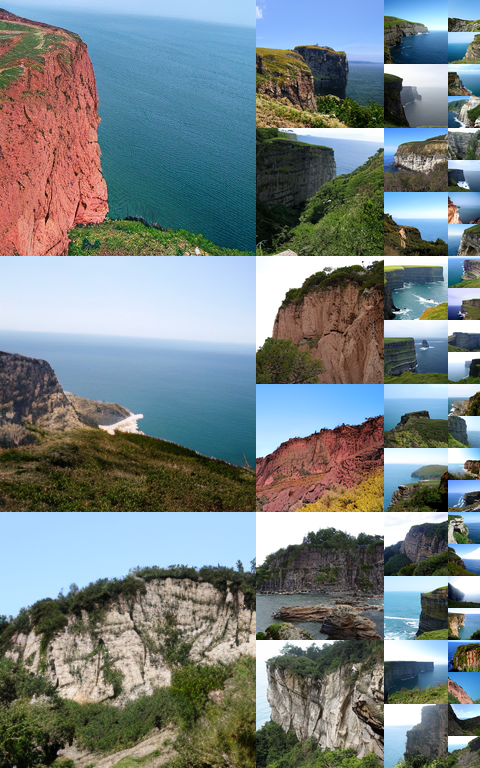}{cliff drop-off}{fig:label972}}
    {\showcasepanel{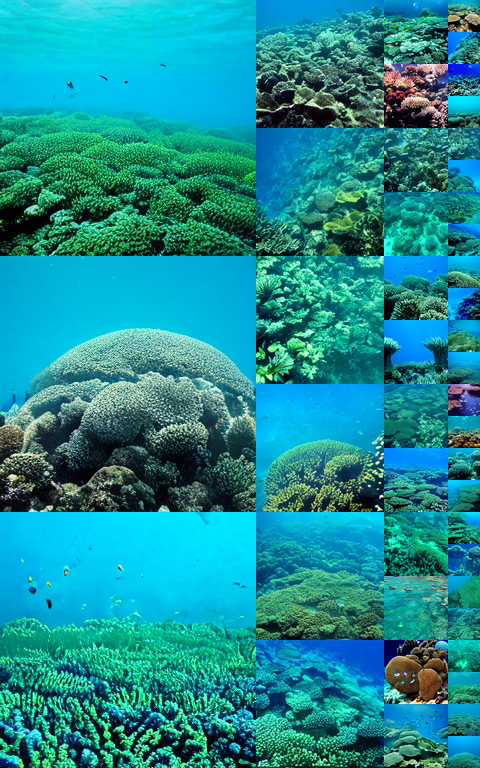}{coral reef}{fig:label973}}
    {\showcasepanel{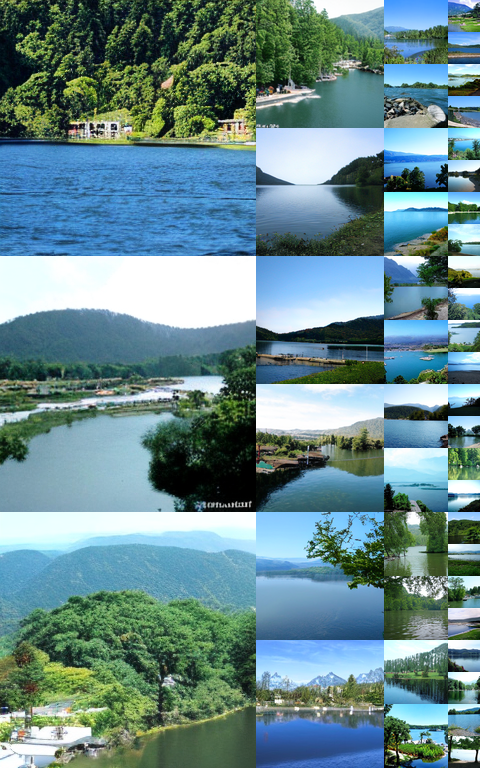}{lake shore}{fig:label975}}
    {\showcasepanel{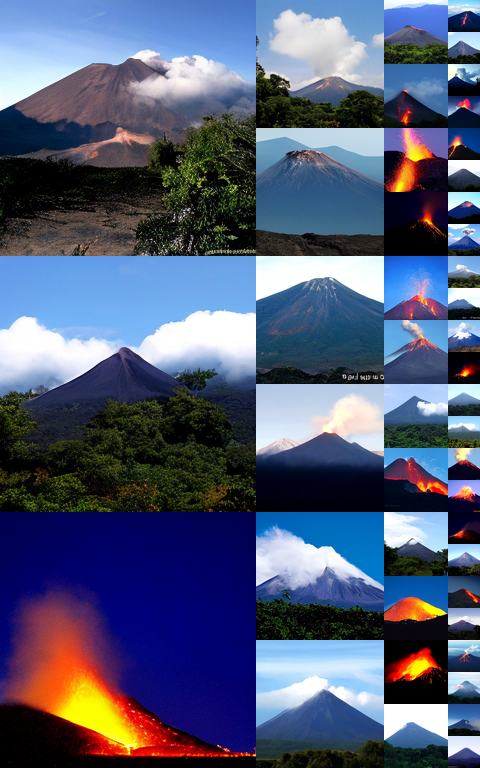}{volcano}{fig:label980}}

\end{document}